\documentclass[11pt,letterpaper]{article}

\usepackage{sty}

\author{
\normalsize Gabriel Arpino\thanks{University of Cambridge, \texttt{ga442@cam.ac.uk}.} \and
\normalsize Ramji Venkataramanan\thanks{University of Cambridge, \texttt{rv285@cam.ac.uk}.}}

\begin{document}

\title{Statistical-Computational Tradeoffs in \\ Mixed Sparse Linear Regression}

\date{\today}
\maketitle
\begin{abstract}
We consider the  problem of  mixed sparse linear regression with two components, where two $k$-sparse signals $\bbeta_1, \bbeta_2 \in \reals^p$ are to be recovered from $n$ unlabelled noisy linear measurements. The sparsity is allowed to be sublinear in the dimension ($k = o(p)$), and the additive noise is assumed to be independent Gaussian with variance $\sigma^2$. Prior work has shown that the problem suffers from a $\frac{k}{\snr^2}$-to-$\frac{k^2}{\snr^2}$ statistical-to-computational gap, resembling other computationally challenging high-dimensional inference problems such as Sparse PCA and Robust Sparse Mean Estimation \citep{brennan_reducibility_2020};  here $\snr := \|\bbeta_1\|_2/{\sigma^2} = \|\bbeta_2\|_2/{\sigma^2}$ is the signal-to-noise ratio. We establish the existence of a more extensive $\frac{k}{\snr^2}$-to-$\frac{k^2 (\snr + 1)^2}{\snr^2}$ computational barrier for this problem through the method of low-degree polynomials, but show that the problem is computationally hard \emph{only} in a very narrow symmetric parameter regime. We identify a smooth information-computation tradeoff between the sample complexity $n$ and runtime $\exp(\tilde{\Theta}(k^2 (\snr + 1)^2/(n \snr^2))$ for any randomized algorithm in this hard regime. Via a simple reduction, this provides novel rigorous evidence for the existence of a computational barrier to solving exact support recovery in sparse phase retrieval with sample complexity $n = \tilde{o}(k^2)$. Our second contribution is to analyze a simple thresholding algorithm which, outside of the narrow regime where the problem is hard, solves the associated mixed  regression detection problem in $O(np)$ time and matches the sample complexity required for (non-mixed) sparse linear regression of $\frac{k (\snr + 1)}{\snr} \log{p}$; this allows the recovery problem to be subsequently solved by state-of-the-art techniques from the dense case.  As a special case of our results, we show that this simple algorithm is order-optimal among a large family of algorithms in solving exact signed support recovery in sparse linear regression. To the best of our knowledge, this is the first thorough study of the interplay between mixture symmetry, signal sparsity, and their joint impact on the computational hardness of mixed sparse linear regression.
\end{abstract}
\thispagestyle{empty}
\clearpage

\tableofcontents
\addtocontents{toc}{\protect\thispagestyle{empty}}
\thispagestyle{empty}
\clearpage
\setcounter{page}{1}

\section{Introduction}
This work considers the problem of two-component mixed sparse linear regression ($\mslr$), where the goal is to estimate two $k$-sparse signals $\bbeta_1 , \bbeta_2 \in \reals^p$ from $n$ \emph{unlabelled} noisy linear measurements. The model  is defined as follows. 
%
%
\begin{definition}[$\mathtt{MSLR}$] \label{def:MSLR}
For $\X \in \mathbb{R}^{n \times p}$, $\w \in \mathbb{R}^{n}$, and $\z \in \mathbb{R}^n$, consider the model:
\[ \y = \X \bbeta_1 \odot \z + \X \bbeta_2 \odot (1-\z) + \w,  \]
where $\odot$ denotes element-wise product between vectors, $X_{i, j} \distas{\text{i.i.d.}} \mathcal{N}(0, 1)$, $w_i \distas{\text{i.i.d.}} \mathcal{N}(0, \sigma^2)$, $z_i \distas{\text{i.i.d.}} \text{Bernoulli}(\phi)$, and $\bbeta_1, \bbeta_2 \in \reals^{p}$ each $k$-sparse. Given $(\X, \y)$ the objective is to estimate $\bbeta_1$, $\bbeta_2$.
\end{definition}
This model was introduced by \citet{quandt_estimating_1978}  and has since been widely studied in the machine learning and statistics communities; see, e.g., \cite{stadler_l1-penalization_2010, chen_convex_2014, yi_alternating_2014, fan_curse_2018, javanmard_prediction_2022} and the references therein.

If the latent variables $(z_i)_{i \in [n]}$ are observed, the problem reduces to solving two separate linear regressions. However, in many applications, the latent variables may be unknown as  the data may come from different unlabelled sub-populations. The $\mslr$ model captures this effect and has been applied to a variety of settings including market segmentation \citep{wedel_market_2000}, music perception \citep{viele_modeling_2002}, health care \citep{deb_estimates_2000, luo_regression-based_2022, im_bayesian_2022}, and various others \citep{li_pursuing_2022, kazor_mixture_2019}. 
Variants of mixed regression models called hierarchical mixtures-of-experts have long been studied in the machine learning community \citep{jordan_hierarchical_1994}, where they have been used for ensemble learning, and in  Gated Recurrent Units and Attention Networks \citep{makkuva_breaking_2019}.

The maximum-likelihood estimator is a natural choice for estimating the signals $\bbeta_1, \bbeta_2$. However, the resulting optimization problem is non-convex and NP-hard \citep{yi_alternating_2014}. The problem is therefore challenging  both statistically and computationally, and a variety of efficient estimators have been proposed. These include spectral methods \citep{chaganty_spectral_2013, yi_alternating_2014, zhang_precise_2022}, expectation-maximization (EM) \citep{Kha07,Far10, stadler_l1-penalization_2010}, alternating minimization \citep{yi_alternating_2014, She19, ghosh_alternating_2020}, convex relaxation \citep{chen_convex_2014},  moment descent methods \citep{Li18,Che20}, and the use of tractable non-convex objectives \citep{Zho16,Bar22}.  

Despite recent works addressing the statistical and computational feasibility of mixed linear regression (including but not limited to \cite{ azizyan_minimax_2013, pal_learning_2022b, pal2021support}), little is understood about the problem in the high-dimensional sparse regime where both the sample size $n$ and the sparsity $k$ can be sublinear in the dimension $p$, over the range of all $\snr$ scalings, where  $\snr := \|\bbeta_1\|^2_2 / \sigma^2 = \|\bbeta_2\|^2_2 / \sigma^2$ is the signal-to-noise ratio. This regime is motivated by a variety of recent statistical applications, ranging from biology to communications (we refer to the monographs \cite{hastie_statistical_2015, giraud_introduction_2021} which contain multiple references). The assumptions in Definition \ref{def:MSLR} of i.i.d. Gaussian  data rows $\x_i$ and additive Gaussian noise $w_i$ have been often considered broadly in the high-dimensional statistics literature as an idealized assumption (e.g., \cite{wainwright_information-theoretic_2009, wainwright_sharp_2009,arias-castro_global_2011,janson_eigenprism_2017}).

One aspect of this  formulation that is starting to become clear is that in a symmetric parameter regime,  the $\mslr$ problem is hard, i.e.,  it cannot be solved by polynomial-time algorithms  at the information-theoretically optimal sample complexity $n_{\IT} = \tilde{\Theta}(k/\snr^2)$ \citep{fan_curse_2018}. Exhaustive search typically yields statistically near-optimal estimators for the signal support sets, but the running time is  exponential in $k$.
The recent works of \cite{brennan_reducibility_2020, fan_curse_2018} provided different ways of quantifying this phenomenon, evidencing a fundamental algorithmic barrier for algorithms performing at all sample complexities $n = \tilde{o}(k^2 / \snr^2)$ and sparsities $k = o(\sqrt{p})$ in a very narrow and symmetric parameter regime which we call \textit{Symmetric Balanced Mixture of Sparse Linear Regressions} ($\sbmslr$), defined as
\begin{align}
\sbmslr : \; \phi = 1/2 \text{ and } \bbeta_1 = -\bbeta_2 \label{eq:sbmslr}.
\end{align}
(Here we recall that $\phi$ is the mixture parameter in Definition \ref{def:MSLR}, so $\phi= \frac{1}{2}$ implies that each $y_i$ is equally likely to come from $\bbeta_1$ or $\bbeta_2$.)
This phenomenon has been termed a $\frac{k}{\snr^2}$-to-$\frac{k^2}{\snr^2}$ \textit{statistical-to-computational gap}, where the problem is solvable with order $k/\snr^2$ samples, but efficient algorithms require at least  order $k^2/\snr^2$ samples. (Throughout this paper, by efficient algorithms we mean those with running time
$O(p^\eta)$ for some constant $\eta >0$.) This computational threshold is similar in order to those derived  for a multitude of statistical estimation problems, from variants of Planted Clique, e.g., sparse PCA and robust mean estimation \citep{brennan_reducibility_2020}.
Notably,  $\sbmslr$ is close to a prominent formulation of  sparse phase retrieval  where $\y = |\X \bbeta| + \w$ \citep{brennan_reducibility_2020, fan_curse_2018}, which has been widely studied and is believed to possess a  $k$-to-$k^2$ statistical-computational gap \citep{liu_towards_2021, wu_hadamard_2021}.
%

The special case of sparse linear regression ($\slr$), where there is only one signal (i.e., $\bbeta_1= \bbeta_2$ in Definition \ref{def:MSLR})
has been extensively studied in the last few decades \citep{candes_decoding_2005, donoho_compressed_2006, wainwright_sharp_2009}. For $\slr$, the statistical-computational gap is much smaller, but still exists. Indeed, in the regime where $k = o(p)$,  the information-theoretically optimal sample complexity for $\slr$ is of order $\frac{k \log(p/k)}{\log(1 + \snr)}$ \citep{wang_information-theoretic_2010, reeves_all-or-nothing_2019}; in contrast, recent works such as \cite{bandeira_franz-parisi_2022, gamarnik_sparse_2022} have established lower bounds in the regime $\snr \to \infty$ via the study of the Overlap Gap Property and Low Degree polynomials, and shown that a  sample complexity of  order at least $k\log{p}$ is required for efficiently solving $\slr$.  Moreover,  upper bounds of the same order can be obtained using a number of algorithms \citep{wainwright_sharp_2009, bandeira_franz-parisi_2022, gamarnik_sparse_2022}.  

In this paper, for both $\mslr$ and the special case of $\slr$, we present new algorithmic lower bounds as well as upper bounds obtained by analyzing a simple thresholding algorithm. The thresholding algorithm, which we call $\corr$, was used by \citet{bandeira_franz-parisi_2022} to obtain upper bounds for approximate support recovery (up to $o(k)$ errors) in $\slr$, in the setting of binary signal and $\snr \to \infty$ with growing $k$. In all our results,  we make the  dependence on $\snr$ explicit, so that they hold for all $\snr$ regimes, including  $\snr = \Theta(1)$  and for $\snr = o(1)$. Before summarizing our results, we define the class of prior distributions we consider for the signals $\bbeta_1, \bbeta_2$.
\paragraph{Signal Priors} We consider joint priors for $\bbeta_1, \bbeta_2$ that are marginally uniform over $k$-sparse vectors $\bbeta_1, \bbeta_2 \in \reals^p$ with equal norm $\|\bbeta\|_2$. The case where the two signals have equal norm is more challenging as each entry of the observation $\y$  will have the same variance regardless of which signal it corresponds to. We denote such a prior  by $\mathcal{P}_{\|\bbeta\|_2}(\mathcal{D})$, where the non-zero entries of each vector  take values in $\mathcal{D} \subseteq \reals$. We assume that $\bmin := \min\{|\beta| \; | \; \beta \in \mathcal{D} \}  > 0$. 
\paragraph{Notation}
We use boldface font for vectors and matrices and plain font to denote scalars (e.g. $\boldsymbol{a}$ and $a$, respectively). For $\X \in \reals^{n \times p}$, $\x_i$ denotes the $i$-th row of this matrix and $\X_j$  the $j$-th column of this matrix. Throughout the work, we adopt the standard asymptotic notation $O(\cdot), \Omega(\cdot), o(\cdot), \omega(\cdot), \text{ and } \Theta(\cdot)$. We let $\tilde{O}(\cdot)$ and analogous variants denote these relations up to $\text{polylog}$ factors.  By $\lesssim$, $\gtrsim$, $\simeq$ we denote inequalities and equality up to constants, respectively. We let $[n] := \{1, 2\cdots n\}$. For the $\mslr$ setting in Definition \ref{def:MSLR} and the   parameter regime $\sbmslr$ in \eqref{eq:sbmslr}, we let $\mslr \setminus \sbmslr$ refer to the $\mslr$ problem with associated parameters lying outside the $\sbmslr$ parameter regime. 
\subsection{Our Contributions}
In what follows, our computational lower bound results hold in full generality for signals with bounded amplitude 
in the scaling regime $p \to \infty$, $n \to \infty$ and $k = o(\sqrt{p})$. Our algorithmic achievability results hold for general signals with high probability in the sublinear sparsity regime $p \to \infty$, $n \to \infty$, $k = o(p)$, and $n = \omega(k)$.

\paragraph{Computational Lower Bounds for $\mslr$ }
We provide novel rigorous evidence through the study of low-degree polynomials \citep{kunisky_notes_2022, schramm_computational_2022, hopkins_statistical_2018}  that there exists a fundamental algorithmic barrier to solving a  detection (hypothesis testing) variant of $\sbmslr$ at all sample complexities $n = o(\frac{k^2 (\snr + 1)^2}{\snr^2} \cdot \frac{1}{\log{p}})$ and sparsities $k = o(\sqrt{p})$.  Moreover, we show that this computational barrier implies a smooth tradeoff between sample  and  time complexities, preventing algorithms with  running time less than $\exp(\tilde{\Theta}(\frac{k^2}{n} \cdot (\snr + 1)^2/\snr^2))$ from succeeding. These results extend those of \citet{brennan_reducibility_2020, fan_curse_2018} by showing that $\sbmslr$ has a significant statistical-to-computational gap  in all \snr regimes (including the noiseless and $\snr= \omega(1)$ regimes), and by identifying a smooth tradeoff between sample size and running time in the hard regime. 

We then provide polynomial-time reductions between the detection and recovery variants of $\sbmslr$, for signals taking nonzero values in $\{1, -1\}$, translating our hardness results to evidence that exact support recovery is just as hard for growing \snr values. We also show that any $\mslr$ regime containing $\sbmslr$ as a subproblem must be hard, by reducing the $\sbmslr$ exact recovery problem to exact recovery in the more general \textit{Partially Symmetric Balanced} $\mslr$ regime, or $\psbmslr$ , where
\begin{align}
\psbmslr:  \phi = \frac{1}{2},  \text{ and } \bbeta_{1, j} = -\bbeta_{2, j} \text{ for } j \in J \subseteq \supp(\bbeta_1) \cap \supp(\bbeta_2), \text{ with } |J| = \Theta(k). \label{eq:psbmslr}
\end{align}
%
Our computational lower bounds for the  noiseless   version of $\sbmslr$  yield equivalent lower bounds for   exact support recovery in   sparse phase retrieval, where $\y = |\X \bbeta| + \w$.
This provides novel rigorous evidence of a computational barrier and a smooth information-computation tradeoff for solving exact support recovery in sparse phase retrieval with $n = \tilde{o}(k^2)$ samples,  addressing a prominent open question on the hardness of this problem \citep{liu_towards_2021, brennan_reducibility_2020, wu_hadamard_2021}.

\paragraph{Algorithms for $\mslr$ }
Perhaps surprisingly, however, we prove that the above algorithmic barrier vanishes outside of $\sbmslr$. We show that a simple thresholding algorithm called $\corr$  solves the detection variant of $\mslr$ outside of $\sbmslr$ with $O(np)$ running time and sample complexity $n$ of order $\frac{k (\snr + 1)}{\snr} \,  \log{p}$, matching that required for efficiently solving sparse linear regression. We note that $\sbmslr$ is a very narrow parameter regime. Indeed, for signal priors (on the non-zero  values) that are absolutely continuous with respect to the Lebesgue measure, the constraint \eqref{eq:sbmslr}  almost surely does not hold,  and therefore, $\corr$  succeeds on a set of measure one. 

In terms of the original recovery problem, $\corr$  is proven to exactly recover the joint support of both signals outside of a regime slightly broader than $\psbmslr$  (see Theorem \ref{thm:CORR-recovery} for a precise statement). Recovery of the joint support then reduces the problem to the dense or proportionally-sparse case ($k/p = n/p = \Theta(1)$) where existing algorithms can infer $\bbeta_1$ and $\bbeta_2$ exactly.
This extends the recent work of \cite{pal_learning_2022} which provides an exact joint support recovery algorithm for the case of binary signals (drawn from   $\{ 0,1 \}^p$) with sample complexity of order  $\frac{k (\snr + 1)}{\snr} \log^3{p}$. We highlight that the assumption of binary signals with all the non-zero entries equal to 1 is  restrictive as it does not encompass the important  regimes  $\sbmslr$, $\psbmslr$  where the problem is hard. 
%
We can summarize the algorithmically hard parameter regimes in set notation as:
\[ {\stackrel{\text{\footnotesize{``Low-degree hard detection''}}}{\sbmslr}} \subset \stackrel{\substack{\text{\footnotesize{``Exact support recovery}} \\ \text{\footnotesize{is hard by reduction''}}}}{\psbmslr} \subset \; \; \; \mslr. \]

\paragraph{Lower Bounds and Algorithms for $\slr$}
Our results also provide clarity into the computational barriers that arise in the  special case of sparse linear regression ($\slr$), where $\bbeta_1 = \bbeta_2$.
As mentioned in the introduction,  previous authors  have established that a  sample complexity of at least order $k\log{p}$ is required for efficient algorithms \citep{bandeira_franz-parisi_2022, gamarnik_sparse_2022}, with matching algorithmic upper bound results available for the case $\snr \to \infty$ \citep{wainwright_sharp_2009, bandeira_franz-parisi_2022, gamarnik_sparse_2022}. 
We extend these findings and provide rigorous low-degree evidence that polynomial-time algorithms require sample complexity of order at least  $n^{\slr}_{\alg}  := \frac{k (\snr + 1)}{\snr} \log{p}$ for the detection variant of $\slr$ (in the regime where $\| \bbeta\|^2_2$ is of order $k$). Our proof technique consists of a vanilla low-degree calculation for $\slr$; this is  different from the approach of \cite{bandeira_franz-parisi_2022}, who established a connection between the low-degree method and the Franz-Parisi criterion to obtain computational lower bounds for $\slr$. Our direct proof technique allows us to explicitly quantify the role of $\snr$ in the problem. 

Furthermore, we prove that $\corr$  solves both detection and signed support recovery in $\slr$ with $9n^{\slr}_{\alg}$ samples, for all \snr scalings and general sparse signal priors. Moreover, it runs in $O(np)$  time which can be significantly more efficient than alternative solutions such as the Lasso depending on the convergence criterion used \citep{wainwright_sharp_2009}. This in turn certifies the order optimality of $\corr$  for exact signed support recovery in $\slr$ with respect to the class of algorithms that are analytic polynomials of the input of degree at most $O(\log{p})$ (including spectral methods running in $O(\log{p})$ iterations). We note that the statistical-computational gap in $\slr$ between $\frac{k \log(p/k)}{\log(1 + \snr)}$ and $n^{\slr}_{\alg}$ is only up to multiplicative constants unless $\snr = \omega(1)$.
 
Our contributions are summarized along with existing results in Table \ref{table} below, for signals taking values in $\{-1,0,1 \}$. In Table \ref{table}, $n_{\IT}$ denotes the information-theoretic threshold for detection (and by reduction, recovery) and $n_{\alg}$ denotes the sample threshold for efficient algorithms. Importantly, we show that $\mslr$ behaves like $\slr$ outside of the narrow $\sbmslr$ regime, and reconcile existing results in the literature proving achievable sample complexity of order $k$ in the binary case \citep{pal_learning_2022} but of order $k^2$ in the general case \citep{stadler_l1-penalization_2010}. These results lead us to believe that the $\frac{k}{\snr^2}$-to-$\frac{k^2 (\snr + 1)^2}{\snr^2}$ gap arises from brittle symmetries in the signals, and that $\sbmslr$ and $\slr$ are computationally very different problems, the former only inefficiently solvable in high-dimensional settings.
\begin{table}[ht] 
\centering
\begin{tabular}{ |p{3.5cm}||p{3.4cm}|p{3.4cm}|p{4.8cm}|  }
 \hline
  & Information-theoretic lower bound $n_{\IT} $ & Algorithmic lower bound $n_{\alg} $ & Algorithms\\
 \hline
 $\mslr$ (Previous)  &  $\tilde{\Theta}(k/\snr^2 )$ \citep{fan_curse_2018} & $\tilde{\Theta}(k^2/\snr^2 )$  \citep{fan_curse_2018, brennan_reducibility_2020} &  $\ell_1$-\text{penalization} \big($n = \Omega(k^2)$, $\snr \to \infty$); Polynomial Identities \big(for 0-1 valued signals, $n = \Omega(\frac{k(\snr + 1)}{\snr} \log^3{p})$\big) \citep{stadler_l1-penalization_2010, pal_learning_2022}\\
 \hline
 $\mslr$ (This Work) &    &
  &  \\
  \; $\sbmslr$, $\psbmslr$  &    & $\Theta\left(\frac{k^2 (\snr + 1)^2}{\snr^2}  \frac{1}{\log{p}} \right)$
  & \\
  \; $\mslr\setminus\sbmslr$ &    &
  & $\corr$ \big($n = \Omega\big(\frac{k(\snr + 1)}{\snr} \log{p}\big)$\big) \\
 \hline
 $\slr$ (Previous) & $\Theta\left(\frac{2k\log{(p/k)}}{\log_2{(1 + \snr)}}\right)$ \citep{wang_information-theoretic_2010, gamarnik_sparse_2022, reeves_all-or-nothing_2019} & $\Theta(k\log{p})$  \citep{wainwright_sharp_2009, gamarnik_sparse_2022, bandeira_franz-parisi_2022, arpino_computational_2021} & Lasso, $\corr$ , Search, OMP ($n = \Omega(k \log{p})$, $\snr \to \infty$) \citep{wainwright_sharp_2009, bandeira_franz-parisi_2022, gamarnik_sparse_2022, wainwright_information-theoretic_2009, cai_orthogonal_2011}\\ 
 \hline
 $\slr$ (This Work)    &   & $\Theta{\left(\frac{k(\snr + 1)}{\snr} \log{p}\right)}$ & $\corr$  ($n \geq \frac{8k(\snr + 1)}{\bmin^2 \snr} \log{2p}$)\\
 \hline
\end{tabular}
\caption{\label{table}Summary of contributions for signals taking values in $\{-1,0,1\}$.}
\end{table}
\vspace{-2pt}

\subsection{Connections to Previous Work}
Among the first works rigorously evidencing statistical-to-computational gaps was that of \cite{barak_nearly_2016} who proved a tight computational lower bound for the Planted Clique ($\PC$) problem using the sum-of-squares (SOS) hierarchy.  Based on the SOS method, \cite{hopkins_statistical_2018} then formulated a conjecture  (a version of Conjecture \ref{conj:low-degree-conjecture} described in the next subsection) on the optimality of low-degree polynomials for hypothesis testing. This approach has yielded evidence for computational barriers in high-dimensional inference problems such as sparse PCA \citep{hopkins_efficient_2017, bandeira_computational_2020}. Other approaches to evidencing computational barriers  include the failure of classes of algorithms such as statistical query \citep{diakonikolas_efficient_2019}, local \citep{linial_locality_1992, gamarnik_limits_2017} and message passing algorithms \citep{zdeborova_statistical_2016,krzakala_gibbs_2007}, and the reduction from variants of canonical ``hard'' problems such as Planted Clique \citep{berthet_complexity_2013, brennan_reducibility_2020}.

Notably, the problem of high-dimensional $\mslr$ has attracted attention as the special case of $\sbmslr$ has been shown to exhibit a $k$-to-$k^2$ statistical-to-computational gap, which we more precisely define as a $\frac{k}{\snr^2}$-to-$\frac{k^2}{\snr^2}$ gap. This was identified through the study of average-case reductions from Planted Clique \citep{brennan_reducibility_2020} and the statistical query model \citep{fan_curse_2018}. After noticing that no polynomial-time algorithms for $\sbmslr$ were known to succeed below sample complexity $\tilde{\Theta}(k^2/\snr^2)$, \citet{fan_curse_2018} derived lower bounds on the information-theoretic and computational limits of an associated detection problem. Specifically, they proved that the information-theoretic minimal sample complexity is $n = \tilde{\Theta}(k/\snr^2)$, while statistical query algorithms (and conjecturally polynomial-time algorithms) are proven to fail for all sample complexities below the larger threshold of $n = \tilde{o}\left(k^2 / \snr^2 \right)$. This matches in order the failure threshold of many existing algorithms in the literature, although it has not been rigorously shown that the computational lower bound is tight. 

Similarly,  \citet{brennan_reducibility_2020} proved that the associated detection problem we consider in this work ($\sbmslrd$) reduces to a variant of the $\PC$ detection problem termed ``Secret Leakage $\PC$'' in a regime contained within sample complexity $n = o\left(k^2 / \snr^2 \right)$. The detection version of Planted Clique can be formulated as that of identifying whether a clique of size $k$ has been artifically ``planted'' in an Erdös-Rényi graph of size $n$. The problem can be solved by exhaustive search for $k = \Omega(\log n)$. The Planted Clique conjecture is that there is no polynomial time algorithm solving $\PC$ if $k = o(\sqrt{n})$. There are a variety of sources of evidence for the $\PC$ conjecture, see \cite{feldman_statistical_2013, barak_nearly_2016, brennan_reducibility_2020} and the references therein. 

The results above provide evidence for a $\frac{k}{\snr^2}$-to-$\frac{k^2}{\snr^2}$ statistical-to-computational gap between the information-theoretic and the computational limits of $\sbmslr$.  
More broadly, the work in \cite{brennan_average-case_2020} makes a step towards understanding the pervasiveness of $k$-to-$k^2$ gaps in high-dimensional statistics by showing that efficient algorithms for learning mixtures with $k$-sparse means require at least $\tilde{\Omega}(k^2)$ sample complexity.  In Theorem \ref{thm:sbmslrd-hardness-cor-snr}, we sharpen the existing computational lower bounds for $\sbmslr$, evidencing a 
more extensive $\frac{k}{\snr^2}$-to-$\frac{k^2 (\snr + 1)^2}{\snr^2}$ gap,  which unlike earlier lower bounds, indicates a significant computational barrier even in the noiseless regime ($\snr = \infty$).

\subsection{The Low-Degree Method} \label{subsec:LD_method}
The low-degree method is a framework for obtaining lower bounds on the complexity of hypothesis testing problems, that emerged from the study of the sum-of-squares hierarchy \citep{barak_nearly_2016, hopkins_power_2017, hopkins_efficient_2017, hopkins_statistical_2018}. The low-degree method boils
down to rigorously ruling out the possibility of low-degree polynomial functions of the input for
solving a given hypothesis testing problem. 
Consider the setting of simple binary hypothesis testing, where one seeks to to distinguish between two  distributions $\P_N$ and $\Q_N$ over $\reals^N$, where $N$ is the (potentially growing) problem size. Given a sample $\x$ drawn from either $\P_N$ or $\Q_N$, the goal is to identify whether $\x$ originated from the former or the latter through a hypothesis test. In our setting of $\mslr$, we can view $N= np + n$ as the total dimension of our data  $(\X, \y)$, and notice that $\log{N} = O(\log{p})$. We consider two notions of success in testing:
\begin{itemize}
	\item \textbf{Strong Detection/Distinguishing}: the test succeeds with probability $1 - o(1)$ as $p \to \infty$.
	\item \textbf{Weak Detection/Distinguishing}: the test succeeds with probability $\frac{1}{2} + \epsilon$ for some constant $\epsilon > 0$.
\end{itemize}

A \textit{degree-D polynomial algorithm} denotes a sequence of (possibly random) multivariate polynomials $g_N: \reals^N \to \reals$ of degree $D$, and $f_{\leq D}$ we denotes the orthogonal projection of a function $f$ onto the space of degree-$D$ polynomials. Over the last decade, it has been established that for a large array of high-dimensional testing problems (including sparse PCA, planted clique, community detection, and many others), the class of degree-$O(\log{p})$ polynomial algorithms is strictly as powerful as the best known polynomial-time algorithms \citep{bandeira_computational_2020, ding_subexponential-time_2023, hopkins_statistical_2018, hopkins_efficient_2017, hopkins_power_2017, kunisky_notes_2022}. This is formalized in the following conjecture. 
%
\begin{conjecture} [The Low Degree Conjecture \cite{coja-oghlan_statistical_2022, hopkins_statistical_2018}] \label{conj:low-degree-conjecture}
Define the chi-square divergence between $\P_N$ and  $\Q_N$ as $\chi^2(\P_N \| \Q_N) := \E_{\x \distas{} \Q_N} \frac{d\P_N(\x)}{d\Q_N(\x)}^2 - 1$, and let
$\chi^2_{\leq D}(\P_N \| \Q_N)$ be its projection onto the space of degree-$D$ polynomials.
\begin{itemize}
	\item If $\chi^2_{\leq D}(\P_N \| \Q_N) = O(1)$ for some $D = \omega(\log{N})$, strong detection has no polynomial-time algorithm and furthermore requires runtime $\exp(\tilde\Omega(D))$.
	\item If $\chi^2_{\leq D}(\P_N \| \Q_N) = o(1)$ for some $D = \omega(\log{N})$, weak detection has no polynomial-time algorithm and furthermore requires runtime $\exp(\tilde\Omega(D))$.
\end{itemize}
\end{conjecture}

A variety of state-of-the-art algorithms can be approximated by low-degree polynomials and therefore rigorously ruled out by low-degree lower bounds of the above form, including the important class of spectral methods (see Theorem 4.4 of \cite{kunisky_notes_2022}), and all statistical query algorithms \citep{brennan_statistical_2021}. Recent works have also proven the equivalence between low-degree polynomial algorithms and well-established algorithmic solutions derived from statistical physics in certain classes of problems \citep{bandeira_franz-parisi_2022, montanari_equivalence_2022}. Although degree $O(\log{p})$ polynomials are not proven to encompass all polynomial-time algorithms, the success of such a polynomial in hypothesis testing tends to indicate the success of general polynomial-time algorithms. In this light, we aim to provide concrete evidence for computational hardness in $\mslr$ and $\slr$ by proving a low-degree lower bound of the form $\chi^2_{\leq D}(\P_N \| \Q_N) = O(1)$ for an associated detection problem, which can then be reduced to recovery. For more background on the low-degree method, see Appendix \ref{appendix:background-low-degree}.

\section{Main Results}
\subsection{Lower bounds for $\mslr$}
We begin by defining a detection variant of $\mslr$, where given $(\X, \y)$ the goal is to distinguish between two hypotheses: one in which  the data correspond to the $\mslr$ model, and another in which $\X$ and $\y$ are independent.
\begin{definition}[Detection Variant $\mslrd$] \label{def:mslrd}
For $\X \in \mathbb{R}^{n \times p}$, $\sigma > 0$, and $\w \in \mathbb{R}^{n}$, consider the following hypothesis testing problem:
\begin{align*}
   &  {\P(\X) \otimes \P(\y)}: \begin{bmatrix}
        \X \\
        \y
    \end{bmatrix} = \begin{bmatrix}
           \X \\
           \sqrt{\frac{\|\bbeta\|^2_2}{\sigma^2} + 1} \cdot \w
         \end{bmatrix} \\
   &  {\P(\X, \y)}: \begin{bmatrix}
        \X \\
        \y
    \end{bmatrix} = \begin{bmatrix}
           \X \\
           \frac{1}{\sigma} \X\bbeta_1 \odot \z + \frac{1}{\sigma} \X \bbeta_2 \odot (1 - \z) + \w
         \end{bmatrix}
\end{align*}
where $(\bbeta_1, \bbeta_2) \sim \mathcal{P}_{\|\bbeta\|_2}(\mathcal{D})$, and $X_{i, j} \distas{\text{i.i.d.}} \mathcal{N}(0, 1)$, $w_i \distas{\text{i.i.d.}} \mathcal{N}(0, 1)$, $z_i \distas{\text{i.i.d.}} \text{Bernoulli}(\phi)$. The task is to construct a function $f$ which strongly distinguishes $\P(\X) \otimes \P(\y)$ from $\P(\X, \y)$.
\end{definition}
Notice that the marginal distributions of $\P(\X) \otimes \P(\y)$ and $\P(\X, \y)$ are equal, so as to rule out solutions that simply threshold the moments of $\y$ and ignore $\X$. The corresponding detection variant of $\sbmslr$, denoted by $\sbmslrd$, is defined similarly to $\mslrd$ in the parameter regime of $\sbmslr$ given in \eqref{eq:sbmslr}.  From this formulation we obtain the following hardness result for $\sbmslrd$. The proof is given in Appendix \ref{subsec:sbmslrd-hardness-proofs}.
\begin{theorem} [Low-degree lower bound for $\sbmslrd$] \label{thm:sbmslrd-hardness-cor-snr}
Consider the setting of $\sbmslrd$ with $\bbeta_1, \bbeta_2 \sim \mathcal{P}_{\|\bbeta\|_2}(\mathcal{D})$, and bounded amplitude signals ($\bmin = \Theta(\|\bbeta\|_{\infty})$). For sample sizes $n$ where $n = \omega(\max\{k, \log{p}\})$ and $n = o\left(\frac{k^2 (\snr + 1)^2}{\snr^2} \cdot \frac{1}{\log{p}}\right)$, Conjecture \ref{conj:low-degree-conjecture} implies that any randomized algorithm requires running time $\exp\left({\tilde{\Omega}\left(\min\left\{\frac{k^2 (\snr + 1)^2}{n \snr^2 }, n\right\} \right)}\right)$ to solve $\sbmslrd$ in the regime $k = o(\sqrt{p})$.
\end{theorem}
Theorem \ref{thm:sbmslrd-hardness-cor-snr} is our main low-degree hardness result. There are three regimes of interest, which we describe in terms of $n^\sbmslr_{\alg} := \frac{k^2 (\snr + 1)^2}{\snr^2}$. First, if  $n = \Omega\left(n^\sbmslr_{\alg} / {\log{p}}\right)$, the lower bound on the running time  in Theorem \ref{thm:sbmslrd-hardness-cor-snr} equals $e^{\tilde{O}(\log{p})}$, and hence does not rule out polynomial-time solutions. Otherwise, Theorem \ref{thm:sbmslrd-hardness-cor-snr} (via Conjecture \ref{conj:low-degree-conjecture}) implies a smooth tradeoff between sample size $n$ and super-polynomial (but sub-exponential) running time $\exp\left(\tilde{\Omega}\left(n^\sbmslr_{\alg} / n\right)\right)$, for $n = \omega(({n^\sbmslr_{\alg}})^\frac{1}{2})$; this is reminiscent of a similar  tradeoff in Sparse PCA \citep{ding_subexponential-time_2023}. In the third case, where $n = o(({n^\sbmslr_{\alg}})^\frac{1}{2})$, Theorem \ref{thm:sbmslrd-hardness-cor-snr} implies that $e^{\tilde{\Omega}(n)}$ running time is required. 
Thus there are  three distinct computational regimes depending on the sample complexity $n$:  the first permitting polynomial-time solutions, the second enforcing a smooth inversely related information-computation tradeoff, and the last implying an exponential increase in running time as the sample size increases. This extends the results of \citep{brennan_reducibility_2020, fan_curse_2018} which indicated that the $n = \tilde{o}(k^2/\snr^2)$ sample regime presents statistical-query and planted-clique related algorithmic barriers for $\sbmslrd$ with signals in $\{-1, 0, 1\}^p$; note that a lower bound of  order  $k^2/\snr^2$ is vacuous in the noiseless setting, as well as in the natural setting where $\snr = \frac{\| \bbeta \|_2^2}{\sigma^2} = \Theta(k)$. 

The work in \cite{fan_curse_2018} proved that the information-theoretic minimal sample complexity of $\sbmslrd$ is $n = \tilde{\Theta}(k / \snr^2)$, which is vacuous for $\snr = \omega(\sqrt{k})$. The information-theoretic minimal sample complexity of the related sparse phase retrieval ($\spr$) detection problem, however, is known to be of order $k \log{p}$ for a broad class of signal-to-noise ratios (see, for example, Theorem 3.2 in \cite{cai_optimal_2016} and Section 6.1 in \cite{lecue_minimax_2015}). By straightforward reductions from $\slr$ to $\sbmslr$ to $\spr$, one can show that the information-theoretic sample complexity of detection in $\sbmslr$ lies between $\frac{k \log{(p/k)}}{\log{(1 + \snr)}}$ and $k \log{p}$. In this light, Theorem \ref{thm:sbmslrd-hardness-cor-snr} certifies a statistical-computational gap in $\sbmslrd$ of order at least $k$ for broad $\snr$ regimes. 

We highlight that Theorem \ref{thm:sbmslrd-hardness-cor-snr} rigorously rules out the success of analytic polynomials of the input of degree at most $O(\log{p})$, including spectral methods. The $k = o(\sqrt{p})$ assumption is often standard for detection lower bounds where the signal is $k$-sparse (see \citep{brennan_reducibility_2020, fan_curse_2018, ding_subexponential-time_2023} and references therein), and can at times be lifted by conditioning away a certain bad event \citep{bandeira_franz-parisi_2022}. 
\begin{remark}
We have included the bounded amplitude assumption in Theorem \ref{thm:sbmslrd-hardness-cor-snr} for  interpretability. The dependence on $\| \bbeta \|_{\infty}$ can be made explicit by replacing  $k^2$ in Theorem \ref{thm:sbmslrd-hardness-cor-snr} with 
$\| \bbeta \|_2^4/ \| \bbeta \|_{\infty}^4$. We believe the dependence on $\| \bbeta \|_{\infty}$ is an artifact of the proof technique; see Appendix \ref{subsec:sbmslrd-hardness-proofs}.
\end{remark}
Through Theorem \ref{thm:psbmslrd-hardness} in Appendix \ref{sec:Proofs-red}, we provide a polynomial-time reduction from $\sbmslrd$ to exact support recovery in $\psbmslr$ , for signals in $\{-1, 0, 1\}^p$ and $\snr = \omega(1)$, transferring hardness from Theorem \ref{thm:sbmslrd-hardness-cor-snr} to this case. In Appendix \ref{sec:spr-hardness}, we provide a polynomial-time reduction from $\sbmslrd$ to both exact support recovery and detection in sparse phase retrieval ($\spr$) for signals with non-zero entries in $\{-1, 0, 1\}^p$, translating the hardness results of Theorem \ref{thm:sbmslrd-hardness-cor-snr} to $\spr$. This provides novel rigorous evidence for the conjecture that $\spr$ is computationally infeasible for sample sizes $n = \tilde{o}(k^2)$ \citep{wu_hadamard_2021, li_pursuing_2022, brennan_reducibility_2020}.
\subsection{Algorithms for $\mslr$}
We denote the support sets of $\bbeta_1, \bbeta_2$ by $\mc{S}_1, \mc{S}_2$, respectively. Note that $|\mc{S}_1| = |\mc{S}_2| =k$.
Let us define the following quantities:
\begin{align*}
& \langle \bbeta  \rangle^2_{\texttt{min}} := \min_{j \in \mathcal{S}_1 \cup \mathcal{S}_2} (\phi \bbeta_{1, j} + (1-\phi) \bbeta_{2, j})^2, \\
& \langle \bbeta \rangle^2_{>0} := \min\limits_{\substack{j \in \mathcal{S}_1 \cup \mathcal{S}_2 \\ (\phi \bbeta_{1, j} + (1-\phi) \bbeta_{2, j}) > 0}} (\phi \bbeta_{1, j} + (1-\phi) \bbeta_{2, j})^2.
\end{align*}
Note that $\langle \bbeta \rangle^2_{>0} > 0$ for $(\bbeta_1, \bbeta_2) \distas{} \mathcal{P}_{\|\bbeta\|_2}(\mathcal{D})$ outside of the $\sbmslr$ regime. 
Also recall that $\bmin = \min\{|\beta| \; | \; \beta \in \mathcal{D} \}  > 0$.

\begin{definition} [$\corr$ ]
Let $\mathtt{CORR}$ be the algorithm that outputs an estimate of the joint support set $\mathcal{S}_1 \cup \mathcal{S}_2$ of $\bbeta_1$, $\bbeta_2$ according to $\widehat{\mathcal{S}_1 \cup \mathcal{S}_2} = \left\{j \in [p] : \, \left|\frac{\langle \X_j, \y \rangle}{\|\y\|_2}\right| \geq \tau \right\}$, where $\tau = \sqrt{2(1 + \frac{\epsilon}{2}) \log{2p}}$ for some $\epsilon \in (0,1)$.
\end{definition}
\begin{theorem} [Success of $\mathtt{CORR}$ on $\mslrd$ outside $\sbmslr$] \label{thm:CORR-mslrd-general}
Consider the general setting of $\mslrd \setminus \sbmslr$ with $(\bbeta_1, \bbeta_2) \distas{} \mathcal{P}_{\|\bbeta\|_2}(\mathcal{D})$. Let $\epsilon \in (0,1)$ be the parameter used in  $\corr$. Then provided 
\[ n \geq \frac{32(1+\epsilon)}{\min\{\phi^2 \bbeta^2_{\texttt{min}}, \, (1-\phi)^2 \bbeta^2_{\texttt{min}}, \, \langle \bbeta \rangle^2_{>0} \}} \frac{\|\bbeta\|^2_2 \,  (\snr + 1)}{\snr} \log{2p}, \] the $\mathtt{CORR}$ algorithm solves strong detection in $\mslrd \setminus \sbmslr$. 
\end{theorem}

The proof of Theorem \ref{thm:CORR-mslrd-general} is given in Appendix \ref{subsec:corr-for-mslrd}. In the natural setting where $\|\bbeta\|_2^2$ is of order $k$, the theorem implies that
 $\corr$  solves $\mslrd$ outside of the $\sbmslr$ regime with \emph{square-root} the number of samples implied by the low-degree lower bound in Theorem \ref{thm:sbmslrd-hardness-cor-snr}, up to $\log$ factors. Indeed, the sample complexity in Theorem \ref{thm:CORR-mslrd-general} matches the optimal sample complexity for the simpler $\slrd$ problem; see  Theorem \ref{thm:slrd-hardness-conj-snr} below. This theorem effectively quantifies the extent to which one can solve $\mslr$ with the sample complexity of $\slr$. The proof of Theorem \ref{thm:CORR-mslrd-general} also holds in the more general case where $\langle \bbeta \rangle^2_{>0} > 0$ and the signal norms $\|\bbeta_1\|_2, \|\bbeta_2\|_2$ are not constrained to be equal. For signal priors on the nonzero entries that are absolutely continuous with respect to the Lebesgue measure, the event $\{\langle \bbeta \rangle^2_{>0} > 0\}$ has measure one, as $\phi \bbeta_1 + (1-\phi) \bbeta_2 \neq 0$ is almost surely satisfied. 
\begin{theorem} [Sucess of $\corr$  for recovery in $\mslr$ for $\bavg > 0$] \label{thm:CORR-recovery}
Consider the general setting of $\mslr$ with either $\sigma = 0, \phi \neq 1/2$ (noiseless), or $\phi = 1/2, \snr = \Omega(k)$ (balanced). Let $(\bbeta_1, \bbeta_2) \distas{} \mathcal{P}_{\|\bbeta\|_2}(\mathcal{D})$.  Let $\epsilon \in (0,1)$ be the parameter used in  $\corr$ , and
\[ n \geq \frac{32 (1 + \epsilon)}{\min\{\phi^2 \bbeta^2_{\texttt{min}}, (1-\phi)^2 \bbeta^2_{\texttt{min}}, \langle \bbeta \rangle^2_{\texttt{min}} \}}  \frac{\|\bbeta\|^2_2 (\snr + 1)}{\snr} \log{2p}. \] 
 Then there exists an algorithm which, in combination with $\corr$ , exactly recovers $\bbeta_1$ and $\bbeta_2$ (up to relabeling) with probability at least $1 - c_1(\frac{k}{p} + ke^{-c_2 n} + \frac{k}{n} + \frac{1}{p^{c_2}})$ for constants $c_1, c_2 > 0$.
\end{theorem}

The proof of Theorem \ref{thm:CORR-recovery},  given in Appendix \ref{sec:recovery-algorithms-mslr}, first uses $\corr$  for support recovery, followed by  existing recovery algorithms for the noiseless and balanced cases of dense ($k/p = \Theta(1)$)  mixed linear regression \citep{yi_alternating_2014, chen_convex_2014}. Under the condition $\langle \bbeta \rangle^2_{\texttt{min}} > 0$, which is slightly more restrictive than $\sbmslr$, Theorem \ref{thm:CORR-recovery}  yields a sample complexity of the same order as that for $\slr$. 
We note that for signal priors on the nonzero entries that are absolutely continuous with respect to the Lebesgue measure, the event $\{\bavg > 0\}$  has measure one. We highlight that the noiseless case can be formulated as a mixed variant of compressed sensing with independent Gaussian design \citep{yu_statistical_2011}. 
\begin{remark}
The restriction to the noiseless and balanced cases in Theorem \ref{thm:CORR-recovery} is due to the guarantees provided by existing algorithms in the dense case, for which experiments indicate success far beyond these regimes \citep{yi_alternating_2014, chen_convex_2014}.
\end{remark}

\subsection{Lower bounds for Sparse Linear Regression ($\slr$)}
We define the detection variant of $\slr$, called $\slrd$,  as per  Definition \ref{def:mslrd} with the  constraint $\bbeta_1 = \bbeta_2$. The following lower bound for $\slrd$ is proved in Appendix \ref{subsec:slrd-hardness-proofs}. 
\begin{theorem} [Low-degree lower bound for $\slrd$] \label{thm:slrd-hardness-conj-snr}
Consider the setting of $\slrd$ (Definition \ref{def:mslrd} under $\bbeta_1 = \bbeta_2$) with $\bbeta \distas{} \mathcal{P}_{\|\bbeta\|_2}\left(\mathcal{D} \right)$. 
For  $n = \omega(\log{p})$  and $n \le  (1-\epsilon) (1 - 2\theta) \frac{\| \bbeta \|_2^2}{\| \bbeta \|^2_\infty} \frac{ (\snr + 1)}{\snr} \log{p}$ for any $\epsilon \in (0,1)$,
Conjecture \ref{conj:low-degree-conjecture} implies that any randomized algorithm requires running time $e^{\tilde{\Omega}\left(n\right)}$ 
to solve $\slrd$ in the regime $k = O(p^{\theta}) \leq \sqrt{p}$ with $\theta \in (0, 1/2]$.
\end{theorem}
In the natural setting where $\| \bbeta \|^2_2$ is of  order $k$ and the entries have bounded amplitude, the low-degree lower bound on $n$ is of order $ \frac{k (\snr + 1)}{\snr}  \log p$. This matches the order of existing lower bounds for $\slr$ in \citep{bandeira_franz-parisi_2022,gamarnik_sparse_2022}, but has the advantage of being valid for all $\snr$ regimes and generic priors on the sparse signal $\bbeta$. We believe that the dependence of the bound on $\| \bbeta\|_{\infty}$ is an artifact of the proof technique; see Appendix \ref{subsec:slrd-hardness-proofs}. 

A reduction from $\slrd$ to $\slr$ follows similarly to the reduction from $\sbmslrd$ to $\sbmslr$, which is given in Appendix \ref{sec:Proofs-red}. 

\subsection{Algorithms for $\slr$}
\begin{theorem} \label{thm:corr-slrd}
Consider the setting of $\slr$ with $\bbeta \distas{} \mathcal{P}_{\|\bbeta\|_2}(\mathcal{D})$. Let $\epsilon\in (0,1)$ be the parameter used in $\corr$. Then for $n \geq \frac{8 (1 + \epsilon)}{\bbeta^2_{\min}} \|\bbeta\|^2_2 \frac{(\snr + 1)}{\snr} \log{2p}$, we have that $\mathtt{CORR}$ solves strong detection in $\slrd$.
\end{theorem}
We next consider a slight variant of $\corr$  
that recovers the \emph{signed} support of $\bbeta$. It produces $\hat{\bbeta}$ with entries given by
\begin{align}
    \hat{\bbeta}_j = \mathbbm{1}\left\{\left|\frac{\langle \X_j, \y \rangle}{\|\y\|_2}\right| \geq \sqrt{2 (1 + \epsilon/2) \log{2p}}\right\} \text{sign}\left(\frac{\langle \X_j, \y \rangle}{\|\y\|_2}\right),  \  \text{ for } j \in [p],
    \label{eq:signed_corr}
\end{align}
where $\text{sign}(x)$ equals $1$ for $x >0$, equals $-1$ for $x<0$, and $0$ for $x=0$.
\begin{theorem} \label{thm:corr-slr}
Consider the setting of $\slr$ with $\bbeta \distas{} \mathcal{P}_{\|\bbeta\|_2}(\mathcal{D})$. Let $\epsilon \in (0,1)$ be the parameter used in the above variant of $\corr$. Then for $n \geq \frac{8 (1 + \epsilon)}{\bbeta^2_{\min}} \|\bbeta\|^2_2 \frac{(\snr + 1)}{\snr} \log{2p}$,  the vector $\hat{\bbeta}$  in \eqref{eq:signed_corr} equals the signed support of $\bbeta$  with probability at least $1 - (\frac{k}{p} + 2ke^{-c_2 n} + \frac{1}{p^{c_2}})$ for some constant $c_2 > 0$.
\end{theorem}
The proofs of Theorem \ref{thm:corr-slrd} and Theorem \ref{thm:corr-slr} are given in Appendix \ref{subsec:corr-slr}. The sample complexity required for the success of $\corr$  matches the low-degree lower bound in Theorem \ref{thm:slrd-hardness-conj-snr} up to constants, which rigorously certifies the order optimality of $\corr$  among low-degree polynomial algorithms, including spectral methods running in  $O(\log{p})$ iterations, in all \snr regimes. These achievable sample complexities also match those of previous work \citep{wainwright_sharp_2009, bandeira_franz-parisi_2022, gamarnik_sparse_2022, cai_orthogonal_2011, donoho_counting_2010}, with the important extension that they hold for all \snr scalings and general sparse signal priors. 
\section{Proof Ideas}
\paragraph{Low-Degree Lower Bounds}
Theorem \ref{thm:sbmslrd-hardness-cor-snr} amounts to proving that $\chi^2_{\leq D}( \P(\X, \y) \| \P(\X) \otimes \P(\y) ) = O(1)$ in $\mslrd$ (Definition \ref{def:mslrd}) with $\bbeta_1 = -\bbeta_2$ and $\phi = 1/2$, for $n$ in the regime specified in the theorem. We rewrite the expression for $\chi^2_{\leq D}( \P(\X, \y) \| \P(\X) \otimes \P(\y) )$ in  Conjecture \ref{conj:low-degree-conjecture} in terms of multivariate Hermite polynomials in the data $(\X, \y)$ of degree up to $D$. For $\bm{\alpha} = [\alpha_1, \alpha_2, \hdots, \alpha_{np + n} ]$, with $\alpha_i \in \mathbb{N}$, the normalized Hermite polynomial of order $\bm{\alpha}$ is denoted by
$\frac{\tilde{H}_{\bm{\alpha}}(\X, \y)}{\sqrt{\bm{\alpha}}!}$. The precise definition of the polynomial is given in Appendix \ref{sec:Proofs-ld}, but the key fact we will use is that $\Big\{ \frac{\tilde{H}_{\bm{\alpha}}}{\sqrt{\bm{\alpha}}!} \Big\}$ form an orthonormal system with respect to the null distribution in Definition \ref{def:mslrd} (see Proposition \ref{prop:ortho} in Appendix \ref{sec:Proofs-ld}).  Using this, we have
%
%
%
%
\begin{align}
\chi^2_{\leq D}(\P(\X, \y) \| \P(\X) \otimes \P(\y)) + 1 \nonumber &=\E_{(\X, \y) \distas{} \P(\X) \otimes \P(\y)} \left(\frac{d\P(\X, \y)}{d(\P(\X) \otimes \P(\y))}\right)_{\leq D}^2 \nonumber \\
&= \sum_{0 \leq |\bm{\alpha}| \leq D} \frac{1}{\bm{\alpha}!} \, \mathop{\mathbb{E}}_{\P(\X) \otimes \P(\y)} \left[ \frac{d\P(\X, \y)}{d(\P(\X) \otimes \P(\y))} \tilde{H}_{\bm{\alpha}}(\X, \y) \right]^2 \nonumber \\
&= \sum_{0 \leq |\bm{\alpha}| \leq D} \frac{1}{\bm{\alpha}!} \, \mathop{\mathbb{E}}_{\P(\X, \y)} \left[ \tilde{H}_{\bm{\alpha}}(\X, \y) \right]^2, \label{eq:I_simp}
\end{align}
where $\bm{\alpha}! = \prod_{i} \alpha_i !$.
%
The key element of the proof involves subsequently upper bounding (\ref{eq:I_simp}) through Hermite polynomial identities and multinomial-theorem manipulations, yielding a weighted sum over $D$ moments of the overlap $\langle \bbeta_1^{(1)}, \bbeta_1^{(2)} \rangle$, where $\bbeta_1^{(1)}, \bbeta_1^{(2)}$ are two i.i.d copies of the signal $\bbeta_1$  (see Lemma \ref{lemma:mslr-hard-norm}). Each of these $D$ moments can be bounded for $k \leq \sqrt{p}$, allowing the entire sum over $D \simeq \min\left\{\frac{k^2 (\snr + 1)^2}{n \snr^2 }, n\right\}$ terms to converge, and yielding the result. The case of $\slr$ in Theorem \ref{thm:slrd-hardness-conj-snr} is similar, but with the simplification $\bbeta_1 = \bbeta_2$,  we can afford to set $D \simeq n$ and still have this sum converge, yielding the key difference in lower bounds between $\slr$ and $\sbmslr$.  
\paragraph{Reductions from detection to recovery}
%
We follow the procedure for average-case reductions outlined by \citet{brennan_reducibility_2020}.  We transfer computational hardness from $\sbmslrd$ to  recovery in $\sbmslr$ by forming an average-case reduction for $k$-sparse signals in $\{-1, 0, 1\}^p$. Denote the parameter regime of Theorem \ref{thm:sbmslrd-hardness-cor-snr} as the ``critical'' parameter regime. Given any sequence of parameters $\mathcal{P}$ in the critical  regime, we construct another sequence of parameters $\mathcal{P}'$ in the critical regime with the following property: if there exists a randomized polynomial-time algorithm $\mathcal{A}'$ solving exact recovery in $\psbmslr$  with parameter scaling $\mathcal{P}'$, then we can construct a randomized polynomial-time algorithm solving $\sbmslrd$ with parameter scaling $\mathcal{P}$. This would in turn contradict Theorem \ref{thm:sbmslrd-hardness-cor-snr},  implying computational hardness of exact recovery in $\psbmslr$ in the critical regime. We first provide an average case reduction from $\sbmslrd$ to exact recovery in $\sbmslr$ in Lemma \ref{lemma:mslrd-to-mslr}, and then reduce exact recovery in $\sbmslr$ to exact recovery in $\psbmslr$  in Theorem \ref{thm:psbmslrd-hardness}.
\paragraph{The $\corr$  algorithm}
For $\mslr$, the proofs of Theorems \ref{thm:CORR-mslrd-general} and \ref{thm:CORR-recovery} crucially rely on Theorem \ref{thm:CORR-mslr-r-general}, which shows that $\corr$  recovers the joint support of the signals ($\mathcal{S}_1 \cup \mathcal{S}_2$) if $n$ satisfies the condition in Theorem \ref{thm:CORR-recovery}. To prove Theorem \ref{thm:CORR-mslr-r-general}, we analyze the quantity  $u_j := \frac{\langle \X_j, \y \rangle}{\|\y\|_2}$ in three  cases. When $j \in (\mathcal{S}_1 \cup \mathcal{S}_2)^\complement$, we have that $u_j \distas{\text{i.i.d.}} \mathcal{N}(0, 1)$ for  $j \in [p]$ by the independence of $\X_j$ and $\y$. The typical value of $\max_{j \in [p]} u_j$ in this case is $\sqrt{2 \log{p}}$, and we can bound the probability of a false positive by standard concentration bounds, detailed in Lemma \ref{lem:high-prob-events}. 
When $j \in \mathcal{S}_1 \cap \mathcal{S}_2$, we show that conditioned on $\y, \z, \bbeta_1, \bbeta_2$,  $u_j$ is normally distributed with mean
\begin{align}
\E\left[ u_j \mid  \y, \z, \bbeta_1, \bbeta_2 \right] = \frac{\| \y_{\{\z = 1\}} \|_2^2 \bbeta_{1, j} + \|\y_{\{\z = 0\}}\|^2_2 \bbeta_{2, j}}{\|\y\|_2 (\|\bbeta\|^2_2 + \sigma^2)}, \label{eq:u_mean}
\end{align}
and variance less than $1$ (Lemma \ref{lemma:cond-mslr-general}).  Here, $\| \y_{\{\z = 1\}} \|_2$ denotes the norm of the vector with entries $( y_i 1_{\{z_i=1 \}})_{i \in [n]}$. For large $n,p$ and $j \in [p]$, the typical value of the conditional mean above is $\sqrt{\frac{n}{\|\bbeta\|^2_2 + \sigma^2}} \, (\phi \bbeta_{1, j} + (1 - \phi) \bbeta_{2, j})$, which is greater than $\sqrt{2 (1 + \epsilon/2) \log{2p}}$ for 
$$n \geq \frac{2 (1 + \epsilon)}{\langle \bbeta  \rangle^2_{\texttt{min}}} (\|\bbeta\|^2_2 + \sigma^2) \log{2p} \simeq (1 + \epsilon) \frac{k(\snr + 1)}{\langle \bbeta  \rangle^2_{\texttt{min}} \snr} \log{2p}.$$ The remaining case $j \in \mathcal{S}_1 \Delta \mathcal{S}_2$ is similar, with the conditional mean obtained by setting $\bbeta_{2, j} = 0$ in \eqref{eq:u_mean}. The results for $\slr$ in Theorems \ref{thm:corr-slrd}, \ref{thm:corr-slr} follow a similar reasoning, with $\E\left[ u_j \mid \y, \z, \bbeta_1, \bbeta_2 \right] = \frac{\bbeta_j \| \y \|_2}{\|\bbeta\|^2_2 + \sigma^2}$. 

\section{Discussion}
In this work we rigorously characterize the computational hardness of Mixed Sparse Linear Regression ($\mslr$) through the method of low-degree polynomials. We evidence that in the highly symmetric $\sbmslr$ regime, randomized polynomial-time algorithms cannot solve an associated detection problem with sample complexity $n = \tilde{o}\left( \frac{k^2 (\snr + 1)^2}{\snr^2}\right)$, revealing a statistical-computational gap of order at least $k$. Outside of the $\sbmslr$ regime, however, a simple polynomial-time algorithm $\corr$ succeeds in solving detection with minimal sample complexity.

We note that our low-degree statistical-computational gap for $\sbmslr$ persists even in the noiseless ($\snr = \infty$) regime. Recent discoveries have highlighted that evidence for statistical-to-computational gaps do not always hold in the noiseless setting. Examples include ``brittle'' algorithms such as Gaussian elimination ``breaking'' the statistical-to-computational gap in learning parities \citep{zadik_lattice-based_2022}. It was also recently found in \cite{zadik_lattice-based_2022} that the LLL family of algorithms, originating from cryptography, can  break the statistical-to-computational gaps predicted in certain noiseless clustering problems. Further, in a recent talk by \citet{zadik_talk_LLL}, a proof sketch was presented for a lattice-based algorithm that can recover $\bbeta \in \reals^p$ in \emph{dense} noiseless phase retrieval with $p + 1$ measurements --- this breaks a conjectured statistical-to-computational gap for  dense phase retrieval, but the algorithm does not capture the sparse problem structure present in $\mslr$. To the best of our knowledge, noiseless inference in $\sbmslr$ and sparse phase retrieval still cannot be achieved with fewer than order $k^2 \log{p}$ samples for the case of $k$-sparse $\bbeta \in \{-1, 0, 1\}^p$. Such an achievement, if possible, would constitute an interesting and novel contribution. 


\section*{Acknowledgements}
GA was supported by a Cambridge Trust scholarship and funding from Invenia Labs. GA thanks Alex Wein, Ilias Zadik, Afonso Bandeira, Nicol\`o Grometto, Daniil Dmitriev for helpful conversations regarding statistical-computational gaps in sparse linear regression and phase retrieval.

\newpage 

{\small{
\bibliographystyle{plainnat}
\bibliography{referencesarxiv}
}}

\newpage

\appendix

\section{Additional background on the Low Degree Method} \label{appendix:background-low-degree}

In this section, we give additional background on the low-degree method, the chi-squared divergence and its orthogonal projection onto the space of low-degree polynomials. Consider the setting in Section \ref{subsec:LD_method}, where the task is to  distinguish between 
two probability distributions $\P_N$ and $\Q_N$ over $\reals^N$ where $N$ is the (potentially growing) problem size. Given a sample $\x$ drawn from $\P_N$ or $\Q_N$, one seeks to identify whether $\x$ originated from the former or the latter through a hypothesis test.  Recall  the notions of strong and weak detection from Section \ref{subsec:LD_method}.

One powerful method of identifying whether strong or weak detection is possible is through the study of the \textit{chi-squared divergence} $\chi^2(\P_N \| \Q_N)$. Indeed, assume that $\P_N$ is absolutely continuous with respect to $\Q_N$, and let $\mathsf{L} = \frac{d\P_N}{d\Q_N}$ be the likelihood ratio. We have:
\begin{align*}
\chi^2(\P_N \|  \Q_N) &:= \E_{\x \distas{} \Q_N} \, \mathsf{L}(\x)^2 - 1\\
&= \sup_{f: \reals^p \to \reals} \frac{(\E_{\x \distas{} \P_N} f(\x))^2}{\E_{\x \distas{} \Q_N} f(\x)^2} - 1 \\
&= \sup_{\substack{f:\reals^p \to \reals \\ \E_{\x \distas{} \Q_N} f(\x) = 0}} \frac{\left(\E_{\x \distas{} \P_N} f(\x) \right)^2}{\E_{\x \distas{} \Q_N} f(\x)^2},
\end{align*}
where the equivalences follow from standard arguments (see \cite{kunisky_notes_2022}). Interpreting the above result, the chi-squared divergence represents optimality in the $L^2$ sense. It relates to the squared maximum expectation any function can have under $\P_N$, while still being bounded in the space $L^2(\Q_N)$. In fact, the chi-square divergence between two distributions can rigorously characterize their behaviour under testing:
\begin{lemma}[Adapted from Lemma 2 of \cite{montanari_limitation_2015} and Lemma 7.1 of \cite{coja-oghlan_statistical_2022}] \label{lemma:chi-square-test}
\begin{itemize}
	\item If $\chi^2(\P_N \| \Q_N) = O(1)$ as $N \to \infty$, then strong detection is impossible.
	\item If $\chi^2(\P_N \| \Q_N) = o(1)$ as $N \to \infty$, then weak detection is impossible.
\end{itemize}
\end{lemma}
This result is powerful, as it identifies the chi-square divergence as a sufficient quantity for finding identifying results in testing. Note however, that this quantity reveals nothing with regards to \textit{computation}.

The computational analogue of the chi-square divergence is the \textit{degree-$D$ chi-square divergence} $\chi^2_{\leq D}(\P_N \| \Q_N)$. This quantity measures whether $\P_N$ and $\Q_N$ can be distinguished by a degree-$D$ polynomial of the input $\x$. Consider the Hilbert Space $L^2(\Q_N)$, where for functions $f, g: \reals^p \to \reals$ we have the inner product $\<f, g\> := \E_{\x \distas{} Q_N}[f(\x) g(\x)]$ and the corresponding norm $\|f\|_{\Q_N} = \sqrt{\<f, f\>_{\Q_N}}$. Additionally, denote $\reals[\x]_{\leq D}$ as the space of multivariate polynomials from $\reals^p$ to $\reals$ of degree at most $D$, and let $f^{\leq D}$ denote the orthogonal projection of $f$ onto $R[\x]_{\leq D}$ in $L^2(\Q_N)$. We can then define $\chi^2_{\leq D}(\P_N \| \Q_N)$ as follows:
\begin{align}
\chi^2_{\leq D}(\P_N \| \Q_N) &:= \E_{\x \distas{} \Q_N} \mathsf{L}^{\leq D}(\x)^2 - 1 \label{def:chi-square}\\
&= \|\mathsf{L}^{\leq D}\|^2_{Q_N} - 1 \nonumber \\
&= \sup_{f \in \reals[\x]_{\leq D}} \frac{(\E_{\x \distas{} \P_N} f(\x))^2}{\E_{\x \distas{} \Q_N} f(\x)^2} - 1 \nonumber \\
&= \sup_{\substack{f \in \reals[\x]_{\leq D} \\ \E_{\x \distas{} \Q_N} f(\x) = 0}} \frac{\left(\E_{\x \distas{} \P_N} f(\x) \right)^2}{\E_{\x \distas{} \Q_N} f(\x)^2}. \nonumber
\end{align}
The proof of this result can be found in \cite{hopkins_statistical_2018, kunisky_notes_2022}. The low-degree chi-square divergence can therefore interpreted analogously to chi-square divergence: it quantifies the maximum expectation any low-degree function can have under $\P_N$ while still being in the degree-$D$ polynomial subspace of $L^2(\Q_N)$. We then have the analogue of Lemma \ref{lemma:chi-square-test} for low-degree polynomial functions of the input and, conjecturally, general polynomial-time algorithms, given by Conjecture \ref{conj:low-degree-conjecture}.

In this work, we consider testing between distributions that do not simply consist of \textit{signal plus noise}, but instead of linearly transformed signals plus noise. Along with the recent work in \cite{bandeira_franz-parisi_2022, arpino_computational_2021}, this is, to the best of our knowledge, among the first applications of the low-degree method to such problems, which were previously believed to be out of reach from current methods \citep{schramm_computational_2022}.

\section{Proofs of Low-Degree Lower Bounds} \label{sec:Proofs-ld}
\paragraph{Preliminaries and Notation.}
All results concerning the low-degree hardness of the associated problems are asymptotic in $p$, as we take $p \to \infty$ first.  We use the conventions from \cite{schramm_computational_2022}. Let $\mathbb{N} = \{0, 1, 2, \cdots\}$ and $\left[n\right] = \{1, 2\cdots n\}$. We define $0^0 := 1$. We denote by boldface a multiset or vector, so for $\bm{\alpha} \in \mathbb{N}^n$ we mean $\bm{\alpha} = [\alpha_1, \alpha_2, \hdots, \alpha_n ]$ for $\alpha_i \in \mathbb{N}, \forall i \in [n]$. For $\bm{\alpha} \in \mathbb{N}^n$, define $\vert \bm{\alpha} \vert = \sum_{i=1}^n \alpha_i$, $\bm{\alpha}! = \prod_{i=1}^n \alpha_i!$ and (for $\bm{X} \in \mathbb{R}^n$) $\bm{X}^{\bm{\alpha}} = \prod_{i=1}^n X_i^{\alpha_i}$. Let $\texttt{abs}(\bm{\alpha})$ denote the entry-wise absolute value operation on the vector $\bm{\alpha}$. We use $\bm{\alpha} \geq \bm{\bbeta}$ to mean $\alpha_i \geq \bbeta_i$ for all $i$. The operations $\bm{\alpha} + \bm{\bbeta}$ and $\bm{\alpha} - \bm{\bbeta}$ are performed entrywise. For $\bm{\alpha}, \bm{\bbeta} \in \mathbb{N}^n$ with $\bm{\alpha} \geq \bm{\bbeta}$, define ${\bm{\alpha} \choose \bm{\bbeta}} = \prod_{i=1}^n {\alpha_i \choose \bbeta_i}$. We use subindices to denote subsets of a vector or multiset, so for $\bm{\alpha} \in \mathbb{N}^{n \times (p+1)}$, we let $\bm{\alpha}_{p+1} := [\alpha_{1, p+1}, \hdots, \alpha_{n, p+1} ]$ denote the $p+1$th column of the matrix $\bm{\alpha}$. We let $\bm{\alpha}_{\cdot, :p}$ denote the entire $n \times p$ submatrix obtained by selecting only up to the $p$th column, and  $\bm{\alpha}_{i, :p}$ the vector consisting of elements from the $i$th row up to the $p$th column. We denote by $[\A \;\; \y]$ the matrix formed through the horizontal concatenation of $\y \in \mathbb{R}^{n}$ onto $\A \in \mathbb{R}^{n \times p}$, forming an $n \times (p+1)$ real matrix. Unless otherwise indicated, we let $\| \cdot \| := \| \cdot \|_{\mathbb{Q}_p}$ and $\langle \cdot, \cdot \rangle := \langle \cdot, \cdot \rangle_{\mathbb{Q}_p}$. We use $\mathbbm{1}$ to denote the indicator function.

The univariate Hermite polynomials $H_k(x)$ for $k \ge 0$ are defined by the recursion $H_0(x)=1$, and 
$H_{k+1}(x) = x H_k(x) - H_k'(x)$.
For $\bm{\alpha} \in  \mathbb{N}^N$, let $H_{\bm{\alpha}}$ denote the \emph{multivariate} Hermite polynomial of  order $\bm{\alpha}$, defined as  $H_{\bm{\alpha}}(\u) = \prod_{i=1}^N H_{\alpha_i}(u_i)$, for $\u \in \mathbb{R}^N$.
 For $N \in \mathbb{N}$ , the normalized $N$-variate Hermite polynomials $\frac{1}{\sqrt{\bm{\alpha}}!} {H}_{\bm{\alpha}}$ form a complete orthonormal system of (multivariate) polynomials for $L^2(\mathcal{N}(\bm{0}, \bm{I}_N))$ (see \cite{kunisky_notes_2022}). 

In what follows, we give further basic facts regarding Hermite polynomials (see \cite{kunisky_notes_2022} for more detailed descriptions), along with two auxiliary combinatorial lemmas that will be of use for the main proofs.

\begin{proposition}[\textup{Gaussian Integration by Parts}, Prop. 2.10 in \cite{kunisky_notes_2022}] \label{prop:GIP}
If $f: \mathbb{R} \to \mathbb{R}$ is $k$-times continuously differentiable and $f(y)$ and its first $k$ derivatives are bounded by $\mathcal{O}\left(\exp(|y|^{\alpha}) \right)$ for some $\alpha \in (0, 2)$, then
\[ \mathbb{E}_{y \distas{} \mathcal{N}(0, 1)} \left[ H_k(y) f(y) \right] = \mathbb{E}_{y \distas{} \mathcal{N}(0, 1)} \left[ \frac{d^k f}{dy^k}(y) \right].\]
\end{proposition}
\begin{proposition}[\textup{Hermite derivative} \citep{jakimovski_hermite_2006}] \label{prop:herm-deriv}
For $n \in \mathbb{N}$, $m \in \mathbb{N}$:
\[
  H_n^{(m)}(x) = \frac{n!}{(n-m)!} H_{n-m}(x).
\]
\end{proposition}
\begin{proposition}[\textup{Hermite sum formula}, Prop 3.1 in \cite{schramm_computational_2022}] \label{prop:hsum}
For any $k \in \mathbb{N}$ and $z, \mu \in \mathbb{R}$,
\begin{align*}
    H_k(z + \mu) = \sum_{l=0}^k {k \choose l} \mu^{k-l} H_l(z).
\end{align*}
\end{proposition}
\begin{proposition}[\textup{Hermite multiplication formula} \cite{oldham_atlas_2009}] \label{prop:hprod}
For $\gamma \in \mathbb{R}$,

\[ H_n(\gamma x) = \sum_{i=0}^{\floor{\frac{n}{2}}} \gamma^{n - 2i} \left(\gamma^2 - 1 \right)^i \binom{n}{2i} \frac{(2i)!}{i!} 2^{-i} H_{n - 2i}(x). \]
\end{proposition}
\begin{proposition} \label{prop:ortho}
Consider the null distribution $\P(\X) \otimes \P(\y)$ whose law given by $$\mathcal{N}(0, 1)^{\otimes n \times p} \otimes \mathcal{N}\left(0, \frac{\|\bbeta\|^2_2}{\sigma^2} + 1\right)^{\otimes (p+1)}.$$ Let $\u = [\X \ \y] \in \mathbb{R}^{N \times (p+1)}$. Then, an orthonormal system with respect to this null distribution, indexed by $\bm{\alpha} \in \mathbb{N}^{n \times (p+1)}$, is given  by $$\frac{1}{\sqrt{\bm{\alpha}}!} \tilde{H}_{\bm{\alpha}}(\u) := \frac{1}{\sqrt{\bm{\alpha}}!} \prod_{i = 1}^n \prod_{j = 1}^{p} H_{\alpha_{i, j}}(u_{i, j}) H_{\alpha_{i, p+1}}\left(\frac{u_{i, p+1}}{\sqrt{\frac{\|\bbeta\|^2_2}{\sigma^2} + 1}}\right).$$
\end{proposition}

\begin{proof} 
Let $\bm{\alpha}^{(1)}, \bm{\alpha}^{(2)} \in \mathbb{N}^{n \times (p+1)}$. Then,
\begin{align*}
&\E_{\P(\X) \otimes \P(\y)} \frac{1}{\sqrt{\bm{\alpha}^{(1)}}!} \tilde{H}_{\bm{\alpha}^{(1)}}(\u) \frac{1}{\sqrt{\bm{\alpha}^{(2)}}!} \tilde{H}_{\bm{\alpha}^{(2)}}(\u) \\
&= \frac{1}{\sqrt{\bm{\alpha}^{(1)}! \aalpha^{(2)}!}} \E_{\P(\X) \otimes \P(\y)}  \prod_{i = 1}^n \prod_{j = 1}^{p} H_{\alpha^{(1)}_{i, j}}(u_{i, j}) H_{\alpha^{(1)}_{i, p+1}}\left(\frac{u_{i, p+1}}{\sqrt{\frac{\|\bbeta\|^2_2}{\sigma^2} + 1}}\right) H_{\alpha^{(2)}_{i, j}}(u_{i, j}) H_{\alpha^{(2)}_{i, p+1}}\left(\frac{u_{i, p+1}}{\sqrt{\frac{\|\bbeta\|^2_2}{\sigma^2} + 1}}\right) \\
&= \frac{1}{\sqrt{\bm{\alpha}^{(1)}! \aalpha^{(2)}!}} \E  \prod_{i = 1}^n \prod_{j = 1}^{p} H_{\alpha^{(1)}_{i, j}}(u_{i, j}) H_{\alpha^{(2)}_{i, j}}\left(u_{i, j}\right) \E  \prod_{i = 1}^n H_{\alpha^{(1)}_{i, p+1}}\left(\frac{u_{i, p+1}}{\sqrt{\frac{\|\bbeta\|^2_2}{\sigma^2} + 1}}\right)  H_{\alpha^{(2)}_{i, p+1}}\left(\frac{u_{i, p+1}}{\sqrt{\frac{\|\bbeta\|^2_2}{\sigma^2} + 1}}\right) \\
&= \frac{1}{\sqrt{\bm{\alpha}^{(1)}! \aalpha^{(2)}!}} \E  \prod_{i = 1}^n \prod_{j = 1}^{p} H_{\alpha^{(1)}_{i, j}}(u_{i, j}) H_{\alpha^{(2)}_{i, j}}\left(u_{i, j}\right) \E  \prod_{i = 1}^n H_{\alpha^{(1)}_{i, p+1}}\left(w_i\right)  H_{\alpha^{(2)}_{i, p+1}}(w_i) \\
&= \frac{1}{\sqrt{\bm{\alpha}^{(1)}! \aalpha^{(2)}!}} \sqrt{\aalpha^{(1)}! \aalpha^{(2)}!} \prod_{i = 1}^n \prod_{j = 1}^{p} \mathbbm{1}_{\{\alpha^{(1)}_{i, j} = \alpha^{(2)}_{i, j}\}} \prod_{i = 1}^n \mathbbm{1}_{\{\alpha^{(1)}_{i, p+1} = \alpha^{(2)}_{i, p+1}\}} \\
&= \mathbbm{1}_{\aalpha^{(1)} = \aalpha^{(2)}},
\end{align*}
where $w_i \distas{\text{i.i.d.}} \mathcal{N}(0, 1)$ are independent of all other variables for $i \in [n]$.
\end{proof}

\begin{lemma} \label{bin_sum_scaling}
For $\bbeta \in \mathbb{N}$ even:
\[
\sum_{\xi = 0}^{\frac{\bbeta}{2}} \binom{\bbeta}{2\xi} \frac{(2\xi)!}{\xi!} \left( -\frac{1}{2}\right)^{\xi} (\bbeta - 2\xi - 1)!! = \mathbbm{1}_{\bbeta = 0}. \]
\end{lemma}

\begin{proof}
We have:
\begin{align*}
\sum_{\xi = 0}^{\frac{\bbeta}{2}} \binom{\bbeta}{2\xi} \frac{(2\xi)!}{\xi!} \left( -\frac{1}{2}\right)^{\xi} (\bbeta - 2\xi - 1)!! 
&= \sum_{\xi = 0}^{{\frac{\bbeta}{2}}} \frac{\bbeta!}{\xi! (\bbeta - 2\xi)!!} \left( \frac{-1}{2}\right)^{\xi}\\
&= \sum_{\xi = 0}^{{\frac{\bbeta}{2}}} \frac{\bbeta!}{\xi! \cdot (\frac{\bbeta}{2} - \xi)! \cdot 2^{\frac{\bbeta}{2} - \xi}} \left( \frac{-1}{2}\right)^{\xi}\\
&= \frac{\bbeta!}{(\frac{\bbeta}{2})! \cdot 2^{\frac{\bbeta}{2}}} \sum_{\xi = 0}^{{\frac{\bbeta}{2}}} \binom{\frac{\bbeta}{2}}{\xi} \left( -1 \right)^{\xi} \\
&= \frac{\bbeta!}{(\frac{\bbeta}{2})! \cdot 2^{\frac{\bbeta}{2}}} (1 + (-1))^{\frac{\bbeta}{2}} \\
&= \mathbbm{1}_{\bbeta = 0}.
\end{align*}
\end{proof}

\begin{lemma} \label{lemma:comb-bound}
For $p \geq 4$ and $k \leq \sqrt{p}$,  it holds that $\frac{p^k}{4 k!} \leq \binom{p}{k}$.
\end{lemma}
\begin{proof}
Note that $\binom{p}{k} \geq \frac{p^k}{4k!}$ if and only if:
\[
\prod_{j=1}^{k-1} \left(1 - \frac{j}{p}\right) \geq \frac{1}{4}
\]
Then applying the $k \leq \sqrt{p}$ assumption:
\begin{align*}
\prod_{j=1}^{k-1} \left(1 - \frac{j}{p}\right) &\geq \prod_{i=1}^{\floor{\sqrt{p}}} \left(1 - \frac{j}{p} \right)
\geq \left(1 - \frac{1}{\sqrt{p}} \right)^{\sqrt{p}}
\end{align*}
Now notice that for $\sqrt{p} \geq 2$, we have that $(1 - \frac{1}{\sqrt{p}})^{\sqrt{p}} \geq \frac{1}{4}$, leading to the desired result.
\end{proof}

\subsection{Low-degree analysis for $\mslr$: general mixtures} \label{appendix:low-deg-proofs}
In subsection, we prove two technical lemmas. The first (Lemma \ref{lemma:mslr-inner-prod}) derives an expression for the projection of the likelihood ratio onto the multivariate Hermite polynomial $\tilde{H}_{\aalpha}$  defined in Proposition \ref{prop:ortho}. The second lemma (Lemma \ref{lemma:mslrd-norm}) derives an explicit expression for the low-degree chi-squared divergence. 
\begin{lemma} \label{lemma:mslr-inner-prod}
Let $\bbeta_1, \bbeta_2 \distas{} \mathcal{P}_{\|\bbeta\|_2}(\mathcal{D})$. Let  $ \mathsf{L} = \frac{d\P(\X, \y)}{d\P(\X) \otimes \P(\y)}$ be the likelihood ratio, 
and $\tilde{H}_{\aalpha}$ the Hermite polynomial  defined in Proposition \ref{prop:ortho}. Then, for  $\aalpha \in \mathbb{N}^{n \times (p+1)}$, we have
\begin{align*}
\langle \mathsf{L}, \tilde{H}_{\aalpha} \rangle &  = \left( \frac{\frac{1}{\sigma}}{\sqrt{\frac{\|\bbeta\|^2_2}{\sigma^2} + 1}} \right)^{|\bm{\alpha}_{p+1}|}  \|\bbeta\|_2^{|\bm{\alpha}_{p+1}| - |\bm{\alpha}_{\cdot, :p}|} \bm{\alpha}_{p+1}! \\ 
& \quad  \cdot  \prod_{i=1}^n \mathbbm{1}_{ \{\alpha_{i, p+1} - |\bm{\alpha}_{\cdot, :p}| = 0 \}} \mathop{\mathbb{E}}_{\bbeta_1, \bbeta_2} \left[ \prod_{i=1}^n \prod_{j=1}^p \left(\bbeta_{1, j} z_i + \bbeta_{2, j} (1-z_i) \right)^{\alpha_{i, j}} \right].
\end{align*}
\end{lemma}
\begin{proof}
We begin by expanding the inner product:
\begin{align*}
&\langle \mathsf{L}, \tilde{H}_{\bm{\alpha}} \rangle \\
&= \mathop{\mathbb{E}}_{\P(\X) \otimes \P(\y)} \left[ \frac{d\P(\X, \y)}{d\P(\X) \otimes \P(\y)} \tilde{H}_{\bm{\alpha}}(\X, \y) \right] \\
&= \mathop{\mathbb{E}}_{\P(\X, \y)} \left[ \prod_{i=1}^{n} \left( \prod_{j = 1}^{p} H_{\alpha_{i, j}}(X_{i, j}) \right) H_{\alpha_{i, p+1}}\left( \frac{\left(\frac{1}{\sigma} \X\bbeta_1 \odot \z + \frac{1}{\sigma} \X\bbeta_2 \odot (1 - \z) + \w \right)_{i} }{\sqrt{\frac{\|\bbeta\|^2_2}{\sigma^2} + 1}} \right) \right],
\end{align*}
and applying Gaussian Integration by Parts (Proposition \ref{prop:GIP}) we obtain
\begin{align*}
&\langle \mathsf{L}, \tilde{H}_{\bm{\alpha}} \rangle \\
&= \mathop{\mathbb{E}}_{\X, \bbeta_1, \bbeta_2, \z, \w} \prod_{i=1}^{n} \left( \frac{\frac{1}{\sigma}}{\sqrt{\frac{\|\bbeta\|^2_2}{\sigma^2} + 1}} \right)^{|\bm{\alpha}_{i, :p}|} \frac{\alpha_{i, p+1}!}{(\alpha_{i, p+1} - |\bm{\alpha}_{i, :p}|)!} \\
& \hspace{3em} \cdot \prod_{j=1}^p \left(\bbeta_{1, j} z_i + \bbeta_{2, j} (1-z_i) \right)^{\alpha_{i, j}} H_{\alpha_{i, p+1} - |\bm{\alpha_{i, :p}}|}\left( \frac{\left(\frac{1}{\sigma} \X\bbeta_1 \odot \z + \frac{1}{\sigma} \X\bbeta_2 \odot (1 - \z) + \w\right)_{i} }{\sqrt{\frac{\|\bbeta\|^2_2}{\sigma^2} + 1}}  \right) \\
&= \left( \frac{\frac{1}{\sigma}}{\sqrt{\frac{\|\bbeta\|^2_2}{\sigma^2} + 1}} \right)^{|\bm{\alpha}_{\cdot, :p}|} \mathop{\mathbb{E}}_{\X, \bbeta_1, \bbeta_2, \z, \w} \prod_{i=1}^{n} \frac{\alpha_{i, p+1}!}{(\alpha_{i, p+1} - |\bm{\alpha}_{i, :p}|)!} \\
& \hspace{3em} \cdot \prod_{j=1}^p \left(\bbeta_{1, j} z_i + \bbeta_{2, j} (1-z_i) \right)^{\alpha_{i, j}}  H_{\alpha_{i, p+1} - |\bm{\alpha_{i, :p}}|}\left( \frac{\left(\frac{1}{\sigma} \X\bbeta_1 \odot \z + \frac{1}{\sigma} \X\bbeta_2 \odot (1 - \z) + \w\right)_{i} }{\sqrt{\frac{\|\bbeta\|^2_2}{\sigma^2} + 1}} \right).
\end{align*}
We then apply the Hermite multiplication and addition formulas outlined in Propositions \ref{prop:hsum} and \ref{prop:hprod}:
\begin{align*}
&\langle \mathsf{L}, \tilde{H}_{\bm{\alpha}} \rangle \\
&= \left( \frac{\frac{1}{\sigma}}{\sqrt{\frac{\|\bbeta\|^2_2}{\sigma^2} + 1}} \right)^{|\bm{\alpha}_{\cdot, :p}|}  \mathop{\mathbb{E}}_{\X, \bbeta_1, \bbeta_2, \z, \w} \prod_{i=1}^n \frac{\alpha_{i, p+1}!}{(\alpha_{i, p+1} - |\alpha_{i, :p}|)!} \prod_{j=1}^p \left(\bbeta_{1, j} z_i + \bbeta_{2, j} (1-z_i) \right)^{\alpha_{i, j}} \\
& \cdot \sum_{\xi = 0}^{\floor{\frac{\alpha_{i, p+1} - |\bm{\alpha}_{i, :p}|}{2}}} \left(\frac{1}{\sqrt{\frac{\|\bbeta\|^2_2}{\sigma^2} + 1}} \right)^{\alpha_{i, p+1} - |\bm{\alpha}_{i, :p}| - 2\xi} \left(\frac{1}{\frac{\|\bbeta\|^2_2}{\sigma^2} + 1} - 1 \right)^\xi \binom{\alpha_{i, p+1}- |\bm{\alpha}_{i, :p}| }{2 \xi} \frac{(2 \xi)!}{\xi!} 2^{-\xi}  \\
& \cdot H_{\alpha_{i, p+1} - |\bm{\alpha}_{i, :p}| - 2\xi} \left( \left(\frac{1}{\sigma} \X\bbeta_1 \odot \z + \frac{1}{\sigma} \X\bbeta_2 \odot (1 - \z) + \w\right)_{i} \right) \\
&= \left( \frac{\frac{1}{\sigma}}{\sqrt{\frac{\|\bbeta\|^2_2}{\sigma^2} + 1}} \right)^{|\bm{\alpha}_{\cdot, :p}|} \mathop{\mathbb{E}}_{\X, \bbeta_1, \bbeta_2, \z, \w} \prod_{i=1}^n\frac{\alpha_{i, p+1}!}{(\alpha_{i, p+1} - |\alpha_{i, :p}|)!} \prod_{j=1}^p \left(\bbeta_{1, j} z_i + \bbeta_{2, j} (1-z_i) \right)^{\alpha_{i, j}}\\
& \cdot  \sum_{\xi = 0}^{\floor{\frac{\alpha_{i, p+1}- |\bm{\alpha}_{i, :p}| }{2}}} \left(\frac{1}{\sqrt{\frac{\|\bbeta\|^2_2}{\sigma^2} + 1}} \right)^{\alpha_{i, p+1} - |\bm{\alpha}_{i, :p}| - 2\xi} \left(\frac{1}{\frac{\|\bbeta\|^2_2}{\sigma^2} + 1} - 1 \right)^\xi \binom{\alpha_{i, p+1}- |\bm{\alpha}_{i, :p}| }{2 \xi} \frac{(2 \xi)!}{\xi!} 2^{-\xi} \\
& \cdot \hspace{-2em} \sum_{\eta = 0}^{\alpha_{i, p+1} - |\bm{\alpha}_{i, :p}| - 2\xi} \binom{\alpha_{i, p+1} - |\bm{\alpha}_{i, :p}| - 2\xi}{\eta} \left( \left(\frac{1}{\sigma} \X\bbeta_1 \odot \z + \frac{1}{\sigma} \X\bbeta_2 \odot (1 - \z)\right)_{i} \right)^{\alpha_{i, p+1} - |\bm{\alpha}_{i, :p}| - 2\xi - \eta} \hspace{-2em} \underbrace{H_{\eta} \left( w_i \right)}_{\text{$\neq 0$ only if $\eta = 0$}},
\end{align*}
which we simplify by noting that $H_{\eta} \left( w_i \right) \neq 0$ only if $\eta = 0$ to obtain:
\begin{align*}
&\langle \mathsf{L}, \tilde{H}_{\bm{\alpha}} \rangle \\
&= \left( \frac{\frac{1}{\sigma}}{\sqrt{\frac{\|\bbeta\|^2_2}{\sigma^2} + 1}} \right)^{|\bm{\alpha}_{\cdot, :p}|}  \mathop{\mathbb{E}}_{\X, \bbeta_1, \bbeta_2, \z} \prod_{i=1}^n\frac{\alpha_{i, p+1}!}{(\alpha_{i, p+1} - |\alpha_{i, :p}|)!} \prod_{j=1}^p \left(\bbeta_{1, j} z_i + \bbeta_{2, j} (1-z_i) \right)^{\alpha_{i, j}} \\
& \cdot \sum_{\xi = 0}^{\floor{\frac{\alpha_{i, p+1}- |\bm{\alpha}_{i, :p}| }{2}}} \left(\frac{1}{\sqrt{\frac{\|\bbeta\|^2_2}{\sigma^2} + 1}} \right)^{\alpha_{i, p+1} - |\bm{\alpha}_{i, :p}| - 2\xi} \left(\frac{1}{\frac{\|\bbeta\|^2_2}{\sigma^2} + 1} - 1 \right)^\xi \binom{\alpha_{i, p+1}- |\bm{\alpha}_{i, :p}| }{2 \xi} \frac{(2 \xi)!}{\xi!} 2^{-\xi} \\
& \cdot \left(\left(\frac{1}{\sigma} \X\bbeta_1 \odot \z + \frac{1}{\sigma} \X\bbeta_2 \odot (1 - \z)\right)_{i} \right)^{\alpha_{i, p+1} - |\bm{\alpha}_{i, :p}| - 2\xi}.
\end{align*}
Now switching the sum with the product and grouping terms we obtain:
\begin{align*}
&\langle \mathsf{L}, \tilde{H}_{\bm{\alpha}} \rangle \\
&= \left( \frac{\frac{1}{\sigma}}{\sqrt{\frac{\|\bbeta\|^2_2}{\sigma^2} + 1}} \right)^{|\bm{\alpha}_{\cdot, :p}|} \prod_{i=1}^n \frac{\alpha_{i, p+1}!}{(\alpha_{i, p+1} - |\alpha_{i, :p}|)!} \mathop{\mathbb{E}}_{[\X, \bbeta_1, \bbeta_2, \z] \distas{} \P(\X, y)} \prod_{j=1}^p \left(\bbeta_{1, j} z_i + \bbeta_{2, j} (1-z_i) \right)^{\alpha_{i, j}}  \\
& \cdot \sum_{\xi = 0}^{\floor{\frac{\alpha_{i, p+1}- |\bm{\alpha}_{i, :p}| }{2}}} \binom{\alpha_{i, p+1}- |\bm{\alpha}_{i, :p}| }{2 \xi} \frac{(2 \xi)!}{\xi!} \left(\frac{\left(\frac{1}{\sigma} \X\bbeta_1 \odot \z + \frac{1}{\sigma} \X\bbeta_2 \odot (1 - \z) \right)_i }{\sqrt{\frac{\|\bbeta\|^2_2}{\sigma^2} + 1}} \right)^{\alpha_{i, p+1} - |\bm{\alpha}_{i, :p}| - 2\xi} \left(\frac{-\frac{\|\bbeta\|^2_2}{\sigma^2}}{2 \big(\frac{\|\bbeta\|^2_2}{\sigma^2} + 1 \big)} \right)^{\xi} \\
&= \left( \frac{\frac{1}{\sigma}}{\sqrt{\frac{\|\bbeta\|^2_2}{\sigma^2} + 1}} \right)^{|\bm{\alpha}_{\cdot, :p}|} \frac{\bm{\alpha}_{p+1}!}{(\bm{\alpha}_{p+1} - |\bm{\alpha}_{\cdot, :p}|)!} \mathop{\mathbb{E}}_{[\X, \bbeta_1, \bbeta_2, \z] \distas{} \P(\X, y)} \sum_{0 \leq \bm{\xi} \leq \bm{\floor{\frac{\alpha_{i, p+1}- |\bm{\alpha}_{i, :p}|}{2}}}} \prod_{i=1}^n \\
& \cdot  \binom{\alpha_{i, p+1}- |\bm{\alpha}_{i, :p}| }{2 \xi_i} \frac{(2 \xi_i)!}{\xi_i!} \left({\frac{-\frac{\frac{\|\bbeta\|^2_2}{\sigma^2}}{\frac{\|\bbeta\|^2_2}{\sigma^2} + 1}}{2}} \right)^{\xi_i} \left(\frac{\frac{1}{\sigma}}{\sqrt{\frac{\|\bbeta\|^2_2}{\sigma^2} + 1}} \right)^{\alpha_{i, p+1}- |\bm{\alpha}_{i, :p}|  - 2\xi_i} \\
& \cdot \left[ \prod_{j=1}^p \left(\bbeta_{1, j} z_i + \bbeta_{2, j} (1-z_i) \right)^{\alpha_{i, j}} \left( \X\bbeta_1 \odot \z + \X\bbeta_2 \odot (1 - \z) \right)_i^{\alpha_{i, p+1}- |\bm{\alpha}_{i, :p}|  - 2\xi_i}  \right],
\end{align*}
which by simplification and expansion then leads us to
\begin{align*}
&\langle \mathsf{L}, \tilde{H}_{\bm{\alpha}} \rangle \\
&= \left( \frac{\frac{1}{\sigma}}{\sqrt{\frac{\|\bbeta\|^2_2}{\sigma^2} + 1}} \right)^{|\bm{\alpha}_{\cdot, p+1}|} \frac{\bm{\alpha}_{p+1}!}{(\bm{\alpha}_{p+1} - |\bm{\alpha}_{\cdot, :p}|)!} \sum_{0 \leq \bm{\xi} \leq \bm{\floor{\frac{\alpha_{i, p+1}- |\bm{\alpha}_{i, :p}|}{2}}}}\\
& \cdot \left( \prod_{i=1}^n \binom{\alpha_{i, p+1}- |\bm{\alpha}_{i, :p}| }{2 \xi_i} \frac{(2 \xi_i)!}{\xi_i!} \left(\frac{-\|\bbeta\|^2_2}{2} \right)^{\xi_i} \right) \\
& \cdot \mathop{\mathbb{E}}_{[\X, \bbeta_1, \bbeta_2, \z] \distas{} \P(\X, y)} \prod_{i=1}^n \left[\prod_{j=1}^p \left(\bbeta_{1, j} z_i + \bbeta_{2, j} (1-z_i) \right)^{\alpha_{i, j}} \left( \sum_{j=1}^p X_{i, j} \bbeta_{1, j} z_i + \sum_{j=1}^p X_{i, j} \bbeta_{2, j} (1-z_i)\right)^{\alpha_{i, p+1}- |\bm{\alpha}_{i, :p}|  - 2\xi_i}  \right].
\end{align*}
Bringing out the expectation with respect to $\bbeta_1, \bbeta_2$ we then obtain:
\begin{align*}
&\langle \mathsf{L}, \tilde{H}_{\bm{\alpha}} \rangle \\
&= \left( \frac{\frac{1}{\sigma}}{\sqrt{\frac{\|\bbeta\|^2_2}{\sigma^2} + 1}} \right)^{|\bm{\alpha}_{\cdot, p+1}|} \frac{\bm{\alpha}_{p+1}!}{(\bm{\alpha}_{p+1} - |\bm{\alpha}_{\cdot, :p}|)!} \sum_{0 \leq \bm{\xi} \leq \bm{\floor{\frac{\alpha_{i, p+1}- |\bm{\alpha}_{i, :p}|}{2}}}} \mathop{\mathbb{E}}_{\bbeta_1, \bbeta_2}\left( \prod_{i=1}^n \prod_{j=1}^p \left(\bbeta_{1, j} z_i + \bbeta_{2, j} (1-z_i) \right)^{\alpha_{i, j}} \right) \\
& \cdot \prod_{i=1}^n  \binom{\alpha_{i, p+1}- |\bm{\alpha}_{i, :p}| }{2 \xi_i} \frac{(2 \xi_i)!}{\xi_i!} \left(\frac{-\|\bbeta\|^2_2}{2} \right)^{\xi_i} \mathop{\mathbb{E}}_{\X, \z} \left( \sum_{j=1}^p X_{i, j} \bbeta_{1, j} z_i + \sum_{j=1}^p X_{i, j} \bbeta_{2, j} (1-z_i) \right)^{\alpha_{i, p+1}- |\bm{\alpha}_{i, :p}|  - 2\xi_i} \\
&= \left( \frac{\frac{1}{\sigma}}{\sqrt{\frac{\|\bbeta\|^2_2}{\sigma^2} + 1}} \right)^{|\bm{\alpha}_{\cdot, p+1}|} \frac{\bm{\alpha}_{p+1}!}{(\bm{\alpha}_{p+1} - |\bm{\alpha}_{\cdot, :p}|)!} \sum_{0 \leq \bm{\xi} \leq \bm{\floor{\frac{\alpha_{i, p+1}- |\bm{\alpha}_{i, :p}|}{2}}}} \mathop{\mathbb{E}}_{\bbeta_1, \bbeta_2}\left( \prod_{i=1}^n \prod_{j=1}^p \left(\bbeta_{1, j} z_i + \bbeta_{2, j} (1-z_i) \right)^{\alpha_{i, j}} \right) \\
& \cdot \prod_{i=1}^n  \binom{\alpha_{i, p+1}- |\bm{\alpha}_{i, :p}| }{2 \xi_i} \frac{(2 \xi_i)!}{\xi_i!} \left(\frac{-\|\bbeta\|^2_2}{2} \right)^{\xi_i} \mathop{\mathbb{E}}_{w \distas{} \mathcal{N}(0, \|\bbeta\|^2_2)} w^{\alpha_{i, p+1}- |\bm{\alpha}_{i, :p}|  - 2\xi_i},
\end{align*}
where $\sum_{j=1}^p X_{i, j} \bbeta_{1, j} z_i + \sum_{j=1}^p X_{i, j} \bbeta_{2, j} (1-z_i) \distas{} \mathcal{N}(0, \|\bbeta\|^2_2)$, both marginally and conditionally on $\bbeta_1, \bbeta_2, \z$, and hence is independent of $\prod_{i=1}^n \prod_{j=1}^p \left(\bbeta_{1, j} z_i + \bbeta_{2, j} (1-z_i) \right)^{\alpha_{i, j}}$ (since $\bbeta_1, \bbeta_2$ are constrained to have norm $\|\bbeta\|_2$ according to our prior). After switching the sum with the product, combining the known equation for Gaussian moments $\mathop{\mathbb{E}}_{w \distas{} \mathcal{N}(0, \|\bbeta\|^2_2)} w^{b} = (b - 1)!! \|\bbeta\|^{b}_2 \mathbbm{1}_{\{b \text{ even}\}}$ with additional factorial simplifications, and applying Lemma \ref{bin_sum_scaling}, we obtain
\begin{align*}
&\langle \mathsf{L}, \tilde{H}_{\bm{\alpha}} \rangle \\
&= \left( \frac{\frac{1}{\sigma}}{\sqrt{\frac{\|\bbeta\|^2_2}{\sigma^2} + 1}} \right)^{|\bm{\alpha}_{p+1}|} \frac{\bm{\alpha}_{p+1}!}{(\bm{\alpha}_{p+1} - |\bm{\alpha}_{\cdot, :p}|)!} \sum_{0 \leq \bm{\xi} \leq \bm{\floor{\frac{\alpha_{i, p+1}- |\bm{\alpha}_{i, :p}|}{2}}}} \\
& \cdot  \left( \prod_{i=1}^n \binom{\alpha_{i, p+1}- |\bm{\alpha}_{i, :p}| }{2 \xi_i} \frac{(2 \xi_i)!}{\xi_i!} \left(\frac{-\|\bbeta\|^2_2}{2} \right)^{\xi_i} \right) \mathop{\mathbb{E}}_{\bbeta_1, \bbeta_2}\left( \prod_{i=1}^n \prod_{j=1}^p \left(\bbeta_{1, j} z_i + \bbeta_{2, j} (1-z_i) \right)^{\alpha_{i, j}} \right) \\
& \cdot \left( \prod_{i=1}^n (\alpha_{i, p+1}- |\bm{\alpha}_{i, :p}| - 2\xi_i - 1)!! \cdot \|\bbeta\|^{\alpha_{i, p+1}- |\bm{\alpha}_{i, :p}| - 2\xi_i}_2 \mathbbm{1}_{ \{ \alpha_{i, p+1}- |\bm{\alpha}_{i, :p}| - 2\xi_i \text{ even} \}} \right) \\
&= \left( \frac{\frac{1}{\sigma}}{\sqrt{\frac{\|\bbeta\|^2_2}{\sigma^2} + 1}} \right)^{|\bm{\alpha}_{p+1}|} \frac{\bm{\alpha}_{p+1}!}{(\bm{\alpha}_{p+1} - |\bm{\alpha}_{\cdot, :p}|)!} \|\bbeta\|_2^{|\bm{\alpha}_{p+1}| - |\bm{\alpha}_{\cdot, :p}|} \\
& \cdot \prod_{i=1}^n \sum_{\xi}^{\floor{\frac{\alpha_{i, p+1}- |\bm{\alpha}_{i, :p}|}{2}}} \binom{\alpha_{i, p+1}- |\bm{\alpha}_{i, :p}| }{2 \xi} \frac{(2 \xi)!}{\xi!} \left(\frac{-1}{2} \right)^{\xi} \\
& \cdot (\alpha_{i, p+1}- |\bm{\alpha}_{i, :p}| - 2\xi - 1)!! \cdot \mathbbm{1}_{ \{ \alpha_{i, p+1}- |\bm{\alpha}_{i, :p}| - 2\xi \text{ even} \}} \mathop{\mathbb{E}}_{\bbeta_1, \bbeta_2}\left( \prod_{i=1}^n \prod_{j=1}^p \left(\bbeta_{1, j} z_i + \bbeta_{2, j} (1-z_i) \right)^{\alpha_{i, j}} \right)  
\end{align*}
\begin{align*}
&= \left( \frac{\frac{1}{\sigma}}{\sqrt{\frac{\|\bbeta\|^2_2}{\sigma^2} + 1}} \right)^{|\bm{\alpha}_{p+1}|} \frac{\bm{\alpha}_{p+1}!}{(\bm{\alpha}_{p+1} - |\bm{\alpha}_{\cdot, :p}|)!} \|\bbeta\|_2^{|\bm{\alpha}_{p+1}| - |\bm{\alpha}_{\cdot, :p}|} \\
& \cdot \prod_{i=1}^n \mathbbm{1}_{\{\alpha_{i, p+1} - |\bm{\alpha}_{i, :p}| = 0 \}} \mathop{\mathbb{E}}_{\bbeta_1, \bbeta_2}
\left[ \prod_{i=1}^n \prod_{j=1}^p \left(\bbeta_{1, j} z_i + \bbeta_{2, j} (1-z_i) \right)^{\alpha_{i, j}} \right] \\
&= \left( \frac{\frac{1}{\sigma}}{\sqrt{\frac{\|\bbeta\|^2_2}{\sigma^2} + 1}} \right)^{|\bm{\alpha}_{p+1}|} \|\bbeta\|_2^{|\bm{\alpha}_{p+1}| - |\bm{\alpha}_{\cdot, :p}|} \bm{\alpha}_{p+1}! \prod_{i=1}^n \mathbbm{1}_{ \{\alpha_{i, p+1} - |\bm{\alpha}_{\cdot, :p}| = 0 \}} \mathop{\mathbb{E}}_{\bbeta_1, \bbeta_2}
\left[ \prod_{i=1}^n \prod_{j=1}^p \left(\bbeta_{1, j} z_i + \bbeta_{2, j} (1-z_i) \right)^{\alpha_{i, j}} \right],
\end{align*}
which leads to the desired result.
\end{proof}

\begin{lemma} \label{lemma:mslrd-norm}
Let $(\bbeta^{(1)}_1, \bbeta^{(1)}_2)$ and $(\bbeta^{(2)}_1, \bbeta^{(2)}_2)$ be two independent copies of signals sampled from $\mathcal{P}_{\|\bbeta\|_2}(\mathcal{D})$, and likewise for $\z^{(1)}$ and $\z^{(2)}$ sampled entrywise from $\text{Bernoulli}(\phi)$. We then have
\begin{align*}
&\chi^2_{\leq D}( \P(\X, \y) \| \P(\X) \otimes \P(\y) ) + 1 
= \mathop{\mathbb{E}}_{\substack{(\bbeta_1^{(1)}, \bbeta_2^{(1)}), (\bbeta_1^{(2)}, \bbeta_2^{(2)}) \distas{\text{i.i.d.}} \mathcal{P} \\ z^{(1)}, z^{(2)} \distas{\text{i.i.d.}} \text{Ber}(\phi)}} \, \sum_{\frac{d}{2} = 0}^{\floor{\frac{D}{2}}} \left( \frac{1}{\|\bbeta\|^2_2 + \sigma^2} \right)^{\frac{d}{2}} \\  
& \hspace{3.5em}\cdot\sum_{|\bm{\alpha}_{p+1}| = \frac{d}{2}} \prod_{i=1}^n \langle \bbeta^{(1)}_{1} z^{(1)}_i + \bbeta^{(1)}_{2} (1-z^{(1)}_i), \bbeta^{(2)}_{1} z^{(2)}_i + \bbeta^{(2)}_{2} (1-z^{(2)}_i) \rangle^{\alpha_{i, p+1}}.
\end{align*}
\end{lemma}

\begin{proof}
We begin the proof by applying Lemma \ref{lemma:mslr-inner-prod} to obtain:
\begin{align*}
&\chi^2_{\leq D}( \P(\X, \y) \| \P(\X) \otimes \P(\y) ) + 1  \\
&= \sum_{0 \leq |\bm{\alpha}| \leq D} \frac{1}{\bm{\alpha}!} {\langle \mathsf{L}, \tilde{H}_{\bm{\alpha}} \rangle}^2 \\
&= \sum_{0 \leq |\bm{\alpha}| \leq D} \frac{1}{\bm{\alpha}!}  \left( \frac{\frac{1}{\sigma}}{\sqrt{\frac{\|\bbeta\|^2_2}{\sigma^2} + 1}} \right)^{2 |\bm{\alpha}_{p+1}|} \|\bbeta\|^{2|\bm{\alpha}_{p+1}|- 2|\bm{\alpha}_{\cdot, :p}|}_2 (\bm{\alpha}_{p+1}!)^2 \\
& \cdot \prod_{i=1}^n \mathbbm{1}_{ \{\alpha_{i, p+1} - |\bm{\alpha}_{\cdot, :p}| = 0 \}} \left(  \mathop{\mathbb{E}} \prod_{i=1}^n \prod_{j=1}^p \left(\bbeta_{1, j} z_i + \bbeta_{2, j} (1-z_i) \right)^{\alpha_{i, j}}  \right)^2 \\
&= \sum_{d = 0}^D \sum_{h = 0}^d \sum_{|\bm{\alpha}_{p+1}| = h} \sum_{|\bm{\alpha}_{\cdot, :p}| = d - h} \frac{1}{\bm{\alpha}!}  \left( \frac{\frac{1}{\sigma}}{\sqrt{\frac{\|\bbeta\|^2_2}{\sigma^2} + 1}} \right)^{2 |\bm{\alpha}_{p+1}|} \|\bbeta\|^{2|\bm{\alpha}_{p+1}|- 2|\bm{\alpha}_{\cdot, :p}|}_2 (\bm{\alpha}_{p+1}!)^2 \\
& \cdot \prod_{i=1}^n \mathbbm{1}_{ \{\alpha_{i, p+1} - |\bm{\alpha}_{\cdot, :p}| = 0 \}} \left(  \mathop{\mathbb{E}} \prod_{i=1}^n \prod_{j=1}^p \left(\bbeta_{1, j} z_i + \bbeta_{2, j} (1-z_i) \right)^{\alpha_{i, j}}  \right)^2.
\end{align*}

We next split a squared expectation into the expectation of the multiplication of two independent random variables: $\left( \mathbb{E}_w \left[w \right] \right)^2 = \mathbb{E}_{w^{(1)}} \left[{w^{(1)}} \right] \mathbb{E}_{w^{(2)}} \left[{w^{(2)}} \right] = \mathbb{E}_{w^{(1)}, w^{(2)}} \left[ w^{(1)} w^{(2)} \right]$, where we have chosen $w^{(1)}$ and $w^{(2)}$ to be two independent and identically distributed random variables. Continuing in this way, we obtain:

\begin{align*}
&\chi^2_{\leq D}( \P(\X, \y) \| \P(\X) \otimes \P(\y) )  + 1 \\
&= \sum_{d = 0}^D \sum_{h = 0}^d \sum_{|\bm{\alpha}_{p+1}| = h} \sum_{|\bm{\alpha}_{\cdot, :p}| = d - h} \frac{1}{\bm{\alpha}!}  \left( \frac{\frac{1}{\sigma}}{\sqrt{\frac{\|\bbeta\|^2_2}{\sigma^2} + 1}} \right)^{2 |\bm{\alpha}_{p+1}|} \|\bbeta\|^{2|\bm{\alpha}_{p+1}|- 2|\bm{\alpha}_{\cdot, :p}|}_2 (\bm{\alpha}_{p+1}!)^2 \\
& \hspace{0em} \cdot \underbrace{\prod_{i=1}^n \mathbbm{1}_{ \{\alpha_{i, p+1} - |\bm{\alpha}_{i, :p}| = 0 \}}}_{\implies \sum_{i=1}^n \alpha_{i, p+1} - |\bm{\alpha_{i, :p}}| = 0 \implies 2h - d = 0 \implies d \text{ even}} \\
& \cdot \mathop{\mathbb{E}}_{\substack{(\bbeta_1^{(1)}, \bbeta_2^{(1)}), (\bbeta_1^{(2)}, \bbeta_2^{(2)}) \distas{\text{i.i.d.}} \mathcal{P} \\ z^{(1)}, z^{(2)} \distas{\text{i.i.d.}} \text{Ber}(\phi)}} \prod_{i=1}^n \prod_{j=1}^p \left(\bbeta^{(1)}_{1, j} z^{(1)}_i + \bbeta^{(1)}_{2, j} (1-z^{(1)}_i) \right)^{\alpha_{i, j}} \left(\bbeta^{(2)}_{1, j} z^{(2)}_i + \bbeta^{(2)}_{2, j} (1-z^{(2)}_i) \right)^{\alpha_{i, j}},
\end{align*}
which we simplify after noticing $\sum_{i=1}^n \alpha_{i, p+1} - |\bm{\alpha_{i, :p}}| = 0$ implies $d$ must be even, 
\begin{align*}
&\chi^2_{\leq D}( \P(\X, \y) \| \P(\X) \otimes \P(\y) )  + 1 \\
&= \mathop{\mathbb{E}} \sum_{\frac{d}{2} = 0}^{\floor{\frac{D}{2}}} \sum_{|\bm{\alpha}_{p+1}| = \frac{d}{2}} \sum_{|\bm{\alpha}_{\cdot, :p}| = \frac{d}{2}} \frac{1}{\bm{\alpha}_{p+1}! \cdot \bm{\alpha}_{\cdot, :p}!} \left( \frac{1}{\|\bbeta\|^2_2 + \sigma^2} \right)^{\frac{d}{2}} \|\bbeta\|^{0}_2 (\bm{\alpha}_{p+1}!)^2 \\
& \hspace{2em} \cdot \prod_{i=1}^n \mathbbm{1}_{ \{\alpha_{i, p+1} - |\bm{\alpha}_{i, :p}| = 0 \}} \\
& \hspace{2em} \cdot \prod_{i=1}^n \prod_{j=1}^p \left(\bbeta^{(1)}_{1, j} z^{(1)}_i + \bbeta^{(1)}_{2, j} (1-z^{(1)}_i) \right)^{\alpha_{i, j}} \left(\bbeta^{(2)}_{1, j} z^{(2)}_i + \bbeta^{(2)}_{2, j} (1-z^{(2)}_i) \right)^{\alpha_{i, j}}   \\
&= \mathop{\mathbb{E}} \sum_{\frac{d}{2} = 0}^{\floor{\frac{D}{2}}} \left( \frac{1}{\|\bbeta\|^2_2 + \sigma^2} \right)^{\frac{d}{2}} \\
& \hspace{2em} \cdot \sum_{|\bm{\alpha}_{p+1}| = \frac{d}{2}} \frac{(\bm{\alpha}_{p+1}!)^2}{\bm{\alpha}_{p+1}!}\sum_{|\bm{\alpha}_{\cdot, :p}| = \frac{d}{2}} \frac{1}{\bm{\alpha}_{\cdot, :p}!} \prod_{i=1}^n \mathbbm{1}_{ \{\alpha_{i, p+1} - |\bm{\alpha}_{i, :p}| = 0 \}} \\
& \hspace{2em} \cdot \prod_{i=1}^n \prod_{j=1}^p \left(\bbeta^{(1)}_{1, j} z^{(1)}_i + \bbeta^{(1)}_{2, j} (1-z^{(1)}_i) \right)^{\alpha_{i, j}} \left(\bbeta^{(2)}_{1, j} z^{(2)}_i + \bbeta^{(2)}_{2, j} (1-z^{(2)}_i) \right)^{\alpha_{i, j}},
\end{align*}
and can be re-ordered in order to more clearly apply the multinomial theorem: 
\begin{align*}
&\chi^2_{\leq D}( \P(\X, \y) \| \P(\X) \otimes \P(\y) )  + 1 \\
&= \mathop{\mathbb{E}} \sum_{\frac{d}{2} = 0}^{\floor{\frac{D}{2}}} \left( \frac{1}{\|\bbeta\|^2_2 + \sigma^2} \right)^{\frac{d}{2}} \sum_{|\bm{\alpha}_{p+1}| = \frac{d}{2}} \bm{\alpha}_{p+1}! \\
& \hspace{2em} \cdot \sum_{|\bm{\alpha}_{1, :p}| = \alpha_{1, p+1}} \frac{1}{\bm{\alpha}_{1, :p}!} \hdots \sum_{|\bm{\alpha}_{n, :p}| = \alpha_{n, p+1}} \frac{1}{\bm{\alpha}_{n, :p}!} \\
& \hspace{2em} \prod_{i=1}^n \prod_{j=1}^p \left(\bbeta^{(1)}_{1, j} z^{(1)}_i + \bbeta^{(1)}_{2, j} (1-z^{(1)}_i) \right)^{\alpha_{i, j}} \left(\bbeta^{(2)}_{1, j} z^{(2)}_i + \bbeta^{(2)}_{2, j} (1-z^{(2)}_i) \right)^{\alpha_{i, j}} \\
&= \mathop{\mathbb{E}} \sum_{\frac{d}{2} = 0}^{\floor{\frac{D}{2}}} \left( \frac{1}{\|\bbeta\|^2_2 + \sigma^2} \right)^{\frac{d}{2}} \sum_{|\bm{\alpha}_{p+1}| = \frac{d}{2}} \bm{\alpha}_{p+1}! \\
& \hspace{2em} \cdot \sum_{|\bm{\alpha}_{1, :p}| = \alpha_{1, p+1}} \frac{\prod_{j=1}^p \left(\bbeta^{(1)}_{1, j} z^{(1)}_i + \bbeta^{(1)}_{2, j} (1-z^{(1)}_i) \right)^{\alpha_{1, j}} \left(\bbeta^{(2)}_{1, j} z^{(2)}_i + \bbeta^{(2)}_{2, j} (1-z^{(2)}_i) \right)^{\alpha_{1, j}}}{\bm{\alpha}_{1, :p}!} \\
& \hspace{2em} \hdots \sum_{|\bm{\alpha}_{n, :p}| = \alpha_{n, p+1}} \frac{\prod_{j=1}^p \left(\bbeta^{(1)}_{1, j} z^{(1)}_i + \bbeta^{(1)}_{2, j} (1-z^{(1)}_i) \right)^{\alpha_{n, j}} \left(\bbeta^{(2)}_{1, j} z^{(2)}_i + \bbeta^{(2)}_{2, j} (1-z^{(2)}_i) \right)^{\alpha_{n, j}}}{\bm{\alpha}_{n, :p}!}.
\end{align*}
We then apply the multinomial theorem to obtain the result:
\begin{align*}
&\chi^2_{\leq D}( \P(\X, \y) \| \P(\X) \otimes \P(\y) ) + 1 \\
&= \mathop{\mathbb{E}} \sum_{\frac{d}{2} = 0}^{\floor{\frac{D}{2}}} \left( \frac{1}{\|\bbeta\|^2_2 + \sigma^2} \right)^{\frac{d}{2}} \sum_{|\bm{\alpha}_{p+1}| = \frac{d}{2}} \prod_{i=1}^n \langle \bbeta^{(1)}_{1} z^{(1)}_i + \bbeta^{(1)}_{2} (1-z^{(1)}_i), \bbeta^{(2)}_{1} z^{(2)}_i + \bbeta^{(2)}_{2} (1-z^{(2)}_i) \rangle^{\alpha_{i, p+1}}.
\end{align*}
\end{proof}
In the next two subsections, we prove the computational lower bounds for $\slrd$ (Theorem \ref{thm:slrd-hardness-conj-snr} and $\sbmslrd$ (Theorem \ref{thm:sbmslrd-hardness-cor-snr})  by specializing Lemma \ref{lemma:mslrd-norm}.
\subsection{Special Case: $\slrd$} \label{subsec:slrd-hardness-proofs}
\begin{proof} [Proof of Theorem \ref{thm:slrd-hardness-conj-snr}]
In Theorem \ref{thm:slrd} below, recalling that $\snr = \|\bbeta\|^2_2/\sigma^2$, we let $n = (1-\epsilon)(1-2\theta) \frac{\| \bbeta \|_2^2}{ \bbeta \|_{\infty}^2} \frac{\snr +1}{ \snr} \log p$. Then Theorem \ref{thm:slrd} implies that for all $D \le \frac{2\epsilon}{1 - \epsilon} n$, we have 
$\chi^2_{\leq D}( \P(\X, \y) \| \P(\X) \otimes \P(\y) )= O(1)$. Applying Conjecture \ref{conj:low-degree-conjecture} with $D=\frac{2\epsilon}{1 - \epsilon} n$ and recalling $n = \omega(\log p)$ (by the assumptions of the theorem), we have that running time  $\exp\left(\tilde{\Omega}(n) \right)$ is required.
\end{proof}
\begin{theorem}[General $\slrd$ lower bound] \label{thm:slrd}
Consider the setting of $\slrd$ (Definition \ref{def:mslrd} with $\bbeta_1 = \bbeta_2$). Let $\bbeta \distas{} \mathcal{P}_{\|\bbeta\|_2}(\mathcal{D})$. If $k = O(p^{\theta}) \leq \sqrt{p}$  for some $\theta \in (0, 1/2]$, then for any $\epsilon \in (0, 1)$, $n \leq (1 - \epsilon) (1 - 2\theta) \left(\frac{\|\bbeta\|^2_2 + \sigma^2}{\|\bbeta\|_{\infty}^2}\right) \log{p}$ and $D \leq  \frac{2\epsilon}{1 - \epsilon} n$, we have $\chi^2_{\leq D}( \P(\X, \y) \| \P(\X) \otimes \P(\y) )= O(1)$.
\end{theorem}
\begin{proof} 
Let $\mathcal{S}^{(1)}$ and $\mathcal{S}^{(2)}$ denote the support sets of $\bbeta^{(1)}$ and $\bbeta^{(2)}$, respectively. We apply Lemma \ref{lemma:mslrd-norm} to obtain:
\begin{align*}
\chi^2_{\leq D}( \P(\X, \y) \| \P(\X) \otimes \P(\y) ) + 1 &= \mathop{\mathbb{E}} \sum_{\frac{d}{2} = 0}^{\floor{\frac{D}{2}}} \left( \frac{1}{\|\bbeta\|^2_2 + \sigma^2} \right)^{\frac{d}{2}} \sum_{|\bm{\alpha}_{p+1}| = \frac{d}{2}} \prod_{i=1}^n \langle \bbeta^{(1)}, \bbeta^{(2)}\rangle^{\alpha_{i, p+1}} \\
&= \mathop{\mathbb{E}} \sum_{\frac{d}{2} = 0}^{\floor{\frac{D}{2}}} \left( \frac{1}{\|\bbeta\|^2_2 + \sigma^2} \right)^{\frac{d}{2}} \sum_{|\bm{\alpha}_{p+1}| = \frac{d}{2}} \langle \bbeta^{(1)}, \bbeta^{(2)}\rangle^\frac{d}{2} \\
&\leq \mathop{\mathbb{E}} \sum_{\frac{d}{2} = 0}^{\floor{\frac{D}{2}}} \left( \frac{1}{\|\bbeta\|^2_2 + \sigma^2} \right)^{\frac{d}{2}} {{\frac{D}{2} + n - 1} \choose {n - 1}} \langle \bbeta^{(1)}, \bbeta^{(2)}\rangle^\frac{d}{2} \\
&\leq \mathop{\mathbb{E}} \sum_{\frac{d}{2} = 0}^{\floor{\frac{D}{2}}} \left( \frac{1}{\|\bbeta\|^2_2 + \sigma^2} \right)^{\frac{d}{2}} \frac{\left( \frac{D}{2} + n \right)^{\frac{d}{2}}}{\frac{d}{2}!} \langle \bbeta^{(1)}, \bbeta^{(2)}\rangle^\frac{d}{2} \\
&\leq \mathop{\mathbb{E}}_{\bbeta^{(1)}, \bbeta^{(2)} \distas{\text{i.i.d.}} \mathcal{P}} \exp\left( \frac{\frac{1}{\sigma^2}}{\frac{\|\bbeta\|^2_2}{\sigma^2} + 1} \left( \frac{D}{2} + n \right) \langle \bbeta^{(1)}, \bbeta^{(2)} \rangle \right) \\
&\leq \mathop{\mathbb{E}}_{\bbeta^{(1)}, \bbeta^{(2)} \distas{\text{i.i.d.}} \mathcal{P}} \exp\left( \frac{\langle \bbeta^{(1)}, \bbeta^{(2)} \rangle}{\|\bbeta\|_{\infty}^2} (1 - 2\theta) \log{p} \right).
\end{align*}
We then apply Lemma \ref{lemma:comb-bound} and notice that $\< \bbeta^{(1)}, \bbeta^{(2)} \> \leq \|\bbeta\|_{\infty}^2 |\mathcal{S}^{(1)} \cap \mathcal{S}^{(2)}|$ to obtain, for $p > 4$ and $k \leq p$:
\begin{align*}
\chi^2_{\leq D}( \P(\X, \y) \| \P(\X) \otimes \P(\y) ) + 1&= \mathop{\mathbb{E}}_{\bbeta^{(1)}, \bbeta^{(2)} \distas{\text{i.i.d.}} \mathcal{P}} \exp\left( \frac{\langle \bbeta^{(1)}, \bbeta^{(2)} \rangle}{\|\bbeta\|_{\infty}^2} (1 - 2\theta) \log{p} \right) \\
&\leq \mathop{\mathbb{E}}_{\langle \bbeta^{(1)}, \bbeta^{(2)} \rangle} \exp\left( |\mathcal{S}^{(1)} \cap \mathcal{S}^{(2)}| (1 - 2\theta) \log{p} \right) \\
&\leq \sum_{l = 0}^k \frac{{k \choose l} {{p-k} \choose {k-l}}}{{p \choose k}} \exp\left( l (1 - 2\theta) \log{p} \right) \\
&\leq \sum_{l = 0}^k \frac{4k!}{p^k} \frac{k^l}{l!} \frac{(p-k)^{k-l}}{(k-l)!} \exp\left( l (1 - 2\theta) \log{p} \right) \\
&\leq 4\sum_{l = 0}^k \left(\frac{k^2}{p} \right)^l \exp\left( l (1 - 2\theta) \log{p} \right) \\
&= 4\sum_{l = 0}^k \left(\frac{k^2}{p^{2\theta}} \right)^l = O(1).
\end{align*}

\end{proof}

\subsection{Special Case: $\sbmslrd$} \label{subsec:sbmslrd-hardness-proofs}
\begin{proof} [Proof of Theorem \ref{thm:sbmslrd-hardness-cor-snr}]
This follows from Theorem \ref{thm:sbmslrd-hardness} below. Choosing any sample size $n$ such that $n \geq k$, $n = \omega(\log{p})$, and $n =  o((\|\bbeta\|^2_2 + \sigma^2)^2/(\|\bbeta\|_{\infty}^4 \log{p}))$, we have that $\chi^2_{\leq D}( \P(\X, \y) \| \P(\X) \otimes \P(\y) ) = O(1)$ for  $D = (\sqrt{2}-1) \min\left\{\frac{(\|\bbeta\|^2_2+\sigma^2)^2}{n \|\bbeta\|_{\infty}^4}, n\right\} $. 
We then invoke Conjecture \ref{conj:low-degree-conjecture} and use $ \frac{\| \bbeta\|^2_2}{\sigma^2} = \snr$, and  notice that for signals with bounded amplitude, we have $k\|\bbeta\|_{\infty}/\sigma^2 \gtrsim \snr \gtrsim k\|\bbeta\|_{\infty}/\sigma^2$.
\end{proof}
\begin{lemma} \label{lemma:mslr-hard-norm}
For $\sbmslrd$, we have that:
\[\chi^2_{\leq D}( \P(\X, \y) \| \P(\X) \otimes \P(\y) ) + 1 \leq \mathop{\mathbb{E}}_{\bbeta^{(1)}, \bbeta^{(2)}} \sum_{\frac{d}{4} = 0}^{\floor{\frac{D}{4}}} \left( \frac{1}{\|\bbeta\|^2_2 + \sigma^2} \right)^{\frac{d}{2}} \frac{\left( \frac{D}{4} + n\right)^{\frac{d}{4}}}{\frac{d}{4}!} \langle \bbeta^{(1)}, \bbeta^{(2)} \rangle^{\frac{d}{2}},\]
where $\bbeta^{(1)}, \bbeta^{(2)}$ are two independent copies of the random variable $\bbeta \stackrel{d}{=} \bbeta_1$.
\end{lemma}
\begin{proof}
We begin by applying the assumptions into Lemma \ref{lemma:mslrd-norm} and applying independence of the $z_i$'s. Notice that in the context of $\sbmslr$, we have in particular that $\phi \bbeta_1 + (1-\phi) \bbeta_2 = 0$, and hence we plug in $\bbeta := \bbeta_1 = -\frac{\phi}{1-\phi} \bbeta_2$.
\begin{align*}
&\chi^2_{\leq D}( \P(\X, \y) \| \P(\X) \otimes \P(\y) ) + 1\\
&\leq \mathop{\mathbb{E}} \sum_{\frac{d}{2} = 0}^{\floor{\frac{D}{2}}} \left( \frac{1}{\|\bbeta\|^2_2 + \sigma^2} \right)^{\frac{d}{2}} \sum_{|\bm{\alpha}_{p+1}| = \frac{d}{2}} \prod_{i=1}^n \langle \bbeta^{(1)}_{1} z^{(1)}_i + \bbeta^{(1)}_{2} (1-z^{(1)}_i), \bbeta^{(2)}_{1} z^{(2)}_i + \bbeta^{(2)}_{2} (1-z^{(2)}_i) \rangle^{\alpha_{i, p+1}} \\
&= \mathop{\mathbb{E}}_{\bbeta} \sum_{\frac{d}{2} = 0}^{\floor{\frac{D}{2}}} \left( \frac{1}{\|\bbeta\|^2_2 + \sigma^2} \right)^{\frac{d}{2}} \sum_{|\bm{\alpha}_{p+1}| = \frac{d}{2}}  
 \cdot \prod_{i=1}^n \left( \phi^2 \langle \bbeta_1^{(1)}, \bbeta_1^{(2)} \rangle^{\alpha_{i, p+1}} +  \, \phi (1-\phi) \langle \bbeta_1^{(1)}, \bbeta_2^{(2)} \rangle^{\alpha_{i, p+1}} \right. \\
& \hspace{4em} \left.  + \  \phi (1-\phi) \langle \bbeta_2^{(1)}, \bbeta_1^{(2)} \rangle^{\alpha_{i, p+1}} + (1 - \phi)^2 \langle \bbeta_2^{(1)}, \bbeta_1^{(2)} \rangle^{\alpha_{i, p+1}} \right) \\
&= \mathop{\mathbb{E}}_{\bbeta} \sum_{\frac{d}{2} = 0}^{\floor{\frac{D}{2}}} \left( \frac{1}{\|\bbeta\|^2_2 + \sigma^2} \right)^{\frac{d}{2}} \sum_{|\bm{\alpha}_{p+1}| = \frac{d}{2}}  \\
&\hspace{2em} \cdot \prod_{i=1}^n \left( \phi^2  + 2 \phi (1-\phi) \left(- \frac{\phi}{1 - \phi}\right)^{\alpha_{i, p+1}} + (1 - \phi)^2 \left(\frac{\phi}{1 - \phi}\right)^{2\alpha_{i, p+1}} \right)\langle \bbeta^{(1)}, \bbeta^{(2)} \rangle^{\alpha_{i, p+1}} \\
&= \mathop{\mathbb{E}}_{\bbeta} \sum_{\frac{d}{2} = 0}^{\floor{\frac{D}{2}}} \left( \frac{1}{\|\bbeta\|^2_2 + \sigma^2} \right)^{\frac{d}{2}} \sum_{|\bm{\alpha}_{p+1}| = \frac{d}{2}} \prod_{i=1}^n \left(\phi + (1-\phi) \left(-\frac{\phi}{1-\phi} \right)^{\alpha_{i, p+1}}\right)^2 \langle \bbeta^{(1)}, \bbeta^{(2)} \rangle^{\alpha_{i, p+1}}.
\end{align*}
We now notice that the term inside of the product equals zero for all $\alpha_{i, p+1}$ odd if and only if $\phi = 1/2$, which is the case for $\sbmslr$. So we sum only over even terms to obtain:
\begin{align*}
&\chi^2_{\leq D}( \P(\X, \y) \| \P(\X) \otimes \P(\y) ) + 1\\
&= \mathop{\mathbb{E}}_{\bbeta} \sum_{\substack{\frac{d}{2} = 0 \\ \text{even}}}^{\floor{\frac{D}{2}}} \left( \frac{1}{\|\bbeta\|^2_2 + \sigma^2} \right)^{\frac{d}{2}} \sum_{\substack{|\bm{\alpha}_{p+1}| = \frac{d}{2} \\ \text{even}}} \prod_{i=1}^n \langle \bbeta^{(1)}, \bbeta^{(2)} \rangle^{\alpha_{i, p+1}} \\
&= \mathop{\mathbb{E}}_{\bbeta} \sum_{\frac{d}{4} = 0}^{\floor{\frac{D}{4}}} \left( \frac{1}{\|\bbeta\|^2_2 + \sigma^2} \right)^{\frac{d}{2}} \sum_{\substack{|\bm{\alpha}_{p+1}| = \frac{d}{2} \\ \text{even}}} \langle \bbeta^{(1)}, \bbeta^{(2)} \rangle^{|\bm{\alpha}_{p+1}|} \\
&= \mathop{\mathbb{E}}_{\bbeta} \sum_{\frac{d}{4} = 0}^{\floor{\frac{D}{4}}} \left( \frac{1}{\|\bbeta\|^2_2 + \sigma^2} \right)^{\frac{d}{2}} {{\frac{d}{4} + n - 1} \choose {n-1}} \langle \bbeta^{(1)}, \bbeta^{(2)} \rangle^{\frac{d}{2}} \\
&\leq \mathop{\mathbb{E}}_{\bbeta} \sum_{\frac{d}{4} = 0}^{\floor{\frac{D}{4}}} \left( \frac{1}{\|\bbeta\|^2_2 + \sigma^2} \right)^{\frac{d}{2}} \frac{\left( \frac{D}{4} + n\right)^{\frac{d}{4}}}{\frac{d}{4}!} \langle \bbeta^{(1)}, \bbeta^{(2)} \rangle^{\frac{d}{2}}.
\end{align*}
\end{proof}
\begin{theorem} [General $\sbmslrd$ lower bound] \label{thm:sbmslrd-hardness}
Consider the setting of $\sbmslrd$ with joint prior $\mathcal{P}_{\|\bbeta\|_2}(\mathcal{D})$. If $k \leq  \sqrt{\frac{p}{e}}$, and  $k \leq D \leq 2(\sqrt{2} - 1) \min\left\{\frac{(\|\bbeta\|^2_2+\sigma^2)^2}{n \|\bbeta\|_{\infty}^4}, n\right\}$, we have that $\chi^2_{\leq D}( \P(\X, \y) \| \P(\X) \otimes \P(\y) ) = O(1)$.
\end{theorem}
\begin{proof} 
We first apply the result of Lemma \ref{lemma:mslrd-norm} to obtain:
\begin{align*}
& \chi^2_{\leq D}( \P(\X, \y) \| \P(\X) \otimes \P(\y) ) + 1 \leq \sum_{\frac{d}{4} = 0}^{\frac{D}{4}} \left( \frac{1}{(\|\bbeta\|^2_2 + \sigma^2)^2} \right)^{\frac{d}{4}} \frac{\left( \frac{D}{4} + n \right)^{\frac{d}{4}} }{\frac{d}{4}!} \mathop{\mathbb{E}}_{\bbeta^{(1)}, \bbeta^{(2)} \distas{\text{i.i.d.}} \mathcal{P}} \langle \bbeta^{(1)}, \bbeta^{(2)} \rangle^\frac{d}{2} \\
& \hspace{3em}\leq \sum_{\frac{d}{4} = 0}^{\frac{D}{4}}  \left( \frac{1}{(\|\bbeta\|^2_2 + \sigma^2)^2} \right)^{\frac{d}{4}} \frac{\left( \frac{D}{4} + n \right)^{\frac{d}{4}} }{\frac{d}{4}!} \mathop{\mathbb{E}}_{\bbeta^{(1)}, \bbeta^{(2)} \distas{\text{i.i.d.}} \mathcal{P}} \langle \texttt{abs}(\bbeta^{(1)}), \texttt{abs}(\bbeta^{(2)}) \rangle^\frac{d}{2},
\end{align*}
where $\langle \bbeta^{(1)}, \bbeta^{(2)} \rangle \leq \langle \texttt{abs}(\bbeta^{(1)}), \texttt{abs}(\bbeta^{(2)}) \rangle$, and we recall $\texttt{abs}(\bbeta)$ denotes the entry-wise absolute value operation on the vector $\bbeta$. Notice that by Lemma \ref{lemma:comb-bound} we have, for $p > 4, k \leq \sqrt{p}$,
\begin{align*}
\mathop{\mathbb{E}}_{\bbeta^{(1)}, \bbeta^{(2)} \distas{\text{i.i.d.}} \mathcal{P}} \langle \texttt{abs}(\bbeta^{(1)}), \texttt{abs}(\bbeta^{(2)}) \rangle^\frac{d}{2} &\leq \sum_{l=0}^k \frac{{k \choose l}{{p-k} \choose {k-l}}}{{p \choose k}} (l \|\bbeta\|_{\infty}^2)^{\frac{d}{2}} \\
&\leq 4 \sum_{l=0}^k \left(\frac{k^2}{p}\right)^l (l \|\bbeta\|_{\infty}^2)^{\frac{d}{2}}.
\end{align*}
We then obtain:
\begin{align*}
\chi^2_{\leq D}( \P(\X, \y) \| \P(\X) \otimes \P(\y) ) + 1 \leq 4 \sum_{\frac{d}{4} = 0}^{\frac{D}{4}} \left( \frac{1}{(\|\bbeta\|^2_2 + \sigma^2)^2} \right)^{\frac{d}{4}} \frac{\left( \frac{D}{4} + n \right)^{\frac{d}{4}} }{\frac{d}{4}!} \sum_{l=0}^k \left(\frac{k^2}{p}\right)^l (l \|\bbeta\|^2_{\infty})^{\frac{d}{2}}.
\end{align*}
With the aim of bounding the right hand side, we enforce condition $i)$: $(D/4 + n) D \leq (\|\bbeta\|^2_2 + \sigma^2)^2/\|\bbeta\|_{\infty}^4$. After switching sums, this yields
\begin{align*}
\chi^2_{\leq D}( \P(\X, \y) \| \P(\X) \otimes \P(\y) ) + 1 &\leq 4 \sum_{l=0}^k \left(\frac{k^2}{p}\right)^l \sum_{\frac{d}{4} = 0}^{\frac{D}{4}} \frac{\left( \frac{(\|\bbeta\|^2_2 + \sigma^2)^2 \|\bbeta\|_{\infty}^4}{D (\|\bbeta\|^2_2 + \sigma^2)^2 \|\bbeta\|_{\infty}^4} \right)^{\frac{d}{4}}}{\frac{d}{4}!} l^{\frac{d}{2}} \\
&\leq 4 \sum_{l=0}^k \left(\frac{k^2}{p}\right)^l \exp\left(\frac{l^2}{D}\right).
\end{align*}
We now enforce condition $ii)$: $k \leq D$ to obtain the result,
\begin{align*}
\chi^2_{\leq D}( \P(\X, \y) \| \P(\X) \otimes \P(\y) ) + 1 &\leq 4 \sum_{l=0}^k \left(\frac{k^2}{p}\right)^l \exp\left( l \right)\\
&= 4 \sum_{l=0}^k \left(\frac{k^2 e}{p}\right)^l = O(1).
\end{align*}
Note that conditions $i)$ and $ii)$ are satisfied for any $n > 0$ and $k \leq D \leq 2(\sqrt{2} - 1) \min\left\{\frac{(\|\bbeta\|^2_2+\sigma^2)^2}{n \|\bbeta\|_{\infty}^4}, n\right\}$.
\end{proof}

\section{Proofs of Polynomial-Time Reductions} \label{sec:Proofs-red}

Consider signed support recovery in the $\mslr$ problem, where we seek to recover the support of $\bbeta_1$ and $\bbeta_2$, along with the  along with the signs of their entries. Take $(\bbeta_1, \bbeta_2) \distas{} \mathcal{P}_{\|\beta\|_2}(\{-1, 1\})$, and let $\mathcal{S}_1 := \supp(\bbeta_1) = \{j \in [p]: \bbeta_{1, j} \neq 0\}$,  and $\mathcal{S}_2$ defined similarly for $\bbeta_2$. We study the computational hardness of the problem as we vary two parameters of our joint signal distribution, the \textit{overlap} $\xi$ and the \textit{signed overlap} $\tau$ respectively:
\[ \xi = \frac{|\mathcal{S}_1 \cap \mathcal{S}_2|}{k}, \qquad  \tau = \frac{\langle \bbeta_1, \bbeta_2 \rangle}{|\mathcal{S}_1 \cap \mathcal{S}_2|}, \]
that are of constant order, i.e., do not scale with respect to $n, p,k$. Previous work \citep{gamarnik_sparse_2022} studies exact support recovery in sparse linear regression and the computational hardness that arises from the overlap distribution of two identical copies of the signal. We extend the analysis by considering exact \textit{signed} support recovery by varying the parameter $\tau$, which measures the relative frequency of $+1$ and $-1$ entries with the same index. 
Note that for $\xi = 1, \tau = 1$ we have the usual $\slr$ problem, for $\xi = 1, \tau = -1$ we have the $\sbmslr$ regime, and importantly for $\phi = 1/2, \tau \in (-1, 1), \xi > 0$ we have the $\psbmslr$  regime. We denote $\mslr_{\xi, \tau}$ and $\mslrd_{\xi, \tau}$ as the $\mslr$ and $\mslrd$ problems with the joint signal prior $\mathcal{P}_{\|\beta\|_2}(\{-1, 1\})$ constrained to signals $(\bbeta_1, \bbeta_2)$ with overlap and signed overlap $\xi$ and $\tau$ respectively.

Using these definitions, we first form in Lemma \ref{lemma:mslrd-to-mslr} a polynomial-time reduction from $\sbmslrd$ to exact signed support recovery in $\sbmslr$ within the scaling regime of Theorem \ref{thm:sbmslrd-hardness-cor-snr}. Notice that $\mslrd_{1, -1}$ is equivalent to the $\sbmslrd$ problem. Next, we prove a polynomial-time reduction from exact signed support recovery in $\sbmslr$ to exact signed support recovery in $\mslr_{\xi, \tau}$ for  $\tau \in (-1, 1), \xi > 0$ ($\psbmslr$ ) within the scaling regime of Theorem \ref{thm:sbmslrd-hardness-cor-snr}, proving that if exact signed support recovery can be achieved in $\psbmslr$ , then it can also be achieved in $\sbmslr$.

Combining the  two arguments above, we have that solving exact signed support recovery in $\psbmslr$  implies solving strong detection in $\sbmslrd$, which would contradict the implication in Theorem \ref{thm:sbmslrd-hardness-cor-snr} that $\sbmslrd$ cannot be solved in polynomial time, resulting in Theorem \ref{thm:psbmslrd-hardness}. For more background on the logic of average-case reductions, we refer to \citep{brennan_reducibility_2020}. 

The reduction from $\slrd$ to $\slr$ is nearly identical  to that in Lemma \ref{lemma:mslrd-to-mslr} with the midly less restrictive condition that $\snr \geq 1$, and hence the proof is omitted. 

Throughout the proofs, we use the following measure of recovery error for mixtures of linear regressions \citep{chen_convex_2014}:
\[
\rho((\hat{\bbeta_1}, \hat{\bbeta}_2), (\bbeta_1, \bbeta_2)):=\min\left\{ \left\Vert \hat{\bm{\bbeta}}_{1}-\bm{\bbeta}_{1}\right\Vert _{2}+\left\Vert \hat{\bm{\bbeta}}_{2}-\bm{\bbeta}_{2}\right\Vert _{2}, \, \left\Vert \hat{\bm{\bbeta}}_{1}-\bm{\bbeta}_{2}\right\Vert _{2}+\left\Vert \hat{\bm{\bbeta}}_{2}-\bm{\bbeta}_{1}\right\Vert _{2}\right\}.
\]
This error measure takes into account recovery of the two signals up to relabelling.  For vectors $a$, $b$, $\hat{a}$, $\hat{b} \in \reals^p$, define 
\begin{align}
\|(\hat{a}, \hat{b}) - (a, b)\|_{\infty} := \min\left\{ \|\hat{a} - a\|_{\infty} + \|\hat{b} - b\|_{\infty}, \|\hat{b} - a\|_{\infty} + \|\hat{a} - b\|_{\infty} \right\}. \nonumber
\end{align}
Notice that for $\epsilon \in [0, 1)$ and signals $(\bbeta_1, \bbeta_2) \distas{} \mathcal{P}_{\|\bbeta\|_2}(\{-1, 1\})$ we have that 
\begin{align*}
\P\left[ \rho((\hat{\bbeta_1}, \, \hat{\bbeta}_2), (\bbeta_1, \bbeta_2)) > \epsilon\right] \to 0 \iff \P\left[ \|(\hat{\bbeta_1}, \hat{\bbeta}_2) - (\bbeta_1, \bbeta_2)\|_{\infty} > \epsilon \right] \to 0,
\end{align*}
\paragraph{Main Statements}
We begin by defining the parameter regimes of interest: 
\begin{align}
& \mathcal{C}_1 = \left\{(p_i, n_i, k_i, \sigma_i)_{i=1}^{\infty} \subset \mathbb{N}^4: 
\, p_i = \omega_{i}(1), k_i = o(\sqrt{p_i}), n_i = \omega(\max\{k_i, \log{p_i}\}) \, ,  \right. \nonumber \\
& \hspace{4em} \left. n_i = o\left( (k_i + \sigma_i^2)^2 \cdot \frac{1}{\log{p_i}} \right) \right\}. \label{eq:C_3} \\
& \mathcal{C}_{2} = \left\{(p_i, n_i, k_i, \sigma_i)_{i=1}^{\infty} \subset \mathbb{N}^4: 
\, p_i = \omega_{i}(1), k_i = o(\sqrt{p_i}), n_i = \omega(\max\{k_i, \log{p_i}\}) \, ,  \right. \nonumber \\
& \hspace{4em} \left. n_i \gtrsim \frac{k_i \log{p_i}}{\log(1 + \frac{k_i}{\sigma_i^2})} \right\}. \label{eq:C_2IT}
\end{align}
Notice that $\mathcal{C}_1, \mathcal{C}_2$ are both contained within the parameter regime where $\sbmslrd$ encounters a computational barrier, as per Theorem \ref{thm:sbmslrd-hardness-cor-snr}. The following lemmas consist of two sub-reductions which together give the reduction argument from $\sbmslrd$ to exact recovery in $\psbmslr$ . 
\begin{lemma} \label{lemma:mslrd-to-mslr}
Let $(\bbeta_1, \bbeta_2) \distas{} \mathcal{P}_{\|\beta\|_2}(\{-1, 1\}), \snr = \omega(1)$. Given a sequence of parameters $\{(p_i, n_i, k_i, \sigma_i)\}_{i=1}^\infty$ in  $\mathcal{C}_2$ for $\sbmslrd$ and $\sbmslr$, if for any $\epsilon > 0$ there exists a randomized polynomial-time algorithm $\mathcal{A}$ for $\sbmslr$ producing $(\hat{\bbeta}_1, \hat{\bbeta}_2)$ with $\mathbb{P}\left[\| (\hat{\bbeta_1}, \hat{\bbeta}_2) - (\bbeta_1, \bbeta_2) \|_{\infty} < \epsilon \right] \underset{(i \to \infty)}{\to} 1$, then there exists a randomized polynomial-time detection algorithm $\mathcal{A'}$ for $\sbmslrd$ with vanishing Type I+II errors as $i \to \infty$.
\end{lemma}
The proof of Lemma \ref{lemma:mslrd-to-mslr} is given in Section \ref{subsec:sbmslrd-sbmslr-red}.
\begin{lemma} \label{lemma:mslr_c_red}
Fix signal priors to be $\mathcal{P}_{\|\beta\|_2}(\{-1, 1\})$. For any sequence of parameters $\{(p'_i, n'_i, k'_i, \sigma'_i) \}_{i = 1}^{\infty}$  in $\mathcal{C}_1$ for $\psbmslr$ with solution $\bbeta'_1, \bbeta'_2$ and problem instances $(\X', \y')$, there exists a sequence of parameters $\{(p_i, n_i, k_i, \sigma_i) \}_{i=1}^{\infty}$ in $\mathcal{C}_{1}$ for $\sbmslr$ with solution $\bbeta_1, \bbeta_2$ and problem instances $(\X, \y)$ such that, for any randomized polynomial time algorithm $\mathcal{A'}$ for $\psbmslr$  outputting $(\hat{\bbeta'}_1, \hat{\bbeta'}_2)$ with
\[ \P\left[ \|(\hat{\bbeta'}_1, \hat{\bbeta'}_2) - (\bbeta'_1, \bbeta'_2)\|_{\infty} > 0 \right] \to 0 ,\]
we can construct a second randomized polynomial time algorithm $\mathcal{A}$ for $\psbmslr$  outputting $(\hat{\bbeta_1}, \hat{\bbeta}_2)$ such that
\[ \P\left[ \|(\hat{\bbeta}_1, \hat{\bbeta}_2) - (\bbeta_1, \bbeta_2)\|_{\infty} > 0 \right] \to 0. \]
\end{lemma}
The proof of Lemma \ref{lemma:mslr_c_red} is given in Section \ref{subsec:sbmslr-psbmslr-red}.
\begin{theorem}[Reduction from $\sbmslrd$ to exact recovery in $\psbmslr$ ] \label{thm:psbmslrd-hardness}
Consider the setting of $\psbmslr$  \eqref{eq:psbmslr} with joint signal prior $\mathcal{P}_{\|\beta\|_2}(\{-1, 1\})$. Any randomized polynomial-time algorithm $\mathcal{A}$ solving $\psbmslr$ within parameter regimes $\mathcal{C}_1 \cap \mathcal{C}_{2}$ and with $\snr = \omega(1)$ would contradict Theorem \ref{thm:sbmslrd-hardness-cor-snr}.
\end{theorem}
\begin{proof} 
Suppose there exists a randomized polynomial-time algorithm $\mathcal{A}$ solving exact recovery in $\psbmslr$  with signals in $\{-1, 0, 1\}^p$ and parameter regime contained in $\mathcal{C}_1$ defined in \eqref{eq:C_3}. Then by Lemma \ref{lemma:mslr_c_red} we would have a randomized polynomial time algorithm $\mathcal{A}'$ solving exact recovery in $\sbmslr$ within this regime. By Lemma \ref{lemma:mslrd-to-mslr}, we would then consequently have a polynomial-time algorithm solving $\sbmslrd$ in the scaling regime $\mathcal{C}_1 \cap \mathcal{C}_{2}$, 
which is contained in the scaling regime of Theorem \ref{thm:sbmslrd-hardness-cor-snr} and hence contradicts Theorem \ref{thm:sbmslrd-hardness-cor-snr}.
\end{proof}
%
%
%
\begin{remark}
Note that the lower bounds $\frac{k\log{p}}{\log(1 + \snr)}$ in constraint $\mathcal{C}_1$ in \eqref{eq:C_3} used in Theorem \ref{thm:psbmslrd-hardness} are not restrictive as this is the information-theoretic minimal sample complexity for support recovery in $\slr$ \citep{reeves_all-or-nothing_2019, gamarnik_sparse_2022, wang_information-theoretic_2010}. 
\end{remark}
\subsection{Reduction from $\sbmslrd$ to $\sbmslr$} \label{subsec:sbmslrd-sbmslr-red}
We utilize a variant of a theorem in \cite{gamarnik_high-dimensional_2019} to construct our reduction, Lemma \ref{lem:psi2}. Consider the following optimization problem for $\sigma > 0$: 
%
\begin{align}
\psi :=  \min \quad & n^{-\frac{1}{2}} \lVert \sigma \w - \X \bbeta_1 \odot \z - \X \bbeta_2 \odot (1 - \z) \rVert_2 \nonumber\\
\textrm{s.t.} \quad & \bbeta_1, \bbeta_2 \in \{-1, 0, 1\}^p, \ \z \in \{0, 1\}^n \label{def:psi2}\\
& \|\bbeta_1\|_0 = \lVert \bbeta_2 \rVert_0 = k, \nonumber
\end{align}
where $\X \distas{\text{i.i.d.}} \mathcal{N}(0, 1)$, independent from $w_i \distas{\text{i.i.d.}} \mathcal{N}(0, 1)$.
\begin{lemma} \label{lem:psi2} Let $\psi$ be as defined in \eqref{def:psi2}. For $\delta > 0$ we have:
\[\mathbb{P}\left[\psi \geq e^{-(1 + \delta)/2} \exp\left(-\frac{2k(\log{p}+1)}{n} \right) \sqrt{k + \sigma^2} \right] \geq 1 - e^{-\frac{\delta}{2}n}. \]
\end{lemma}
The proof follows by nearly identical arguments as that of Theorem 3.1 in \citep{gamarnik_high-dimensional_2019} and is hence omitted.

\vspace{10pt}

\begin{proof}[Proof of Lemma \ref{lemma:mslrd-to-mslr}]
Throughout the proof, we drop the $i$ subscript in the parameters $(p_i, n_i, k_i, \sigma_i)$ for convenience. We refer to $\sbmslr$ and $\sbmslrd$ as $\mslr_{\xi, \tau}$ and $\mslrd_{\xi, \tau}$ respectively, with $\xi = 1, \tau = -1$. We take $\rP := \P(\X, \y)$ to represent the planted measure in the formulation of 
$\slrd$, and $\rQ := \P(\X) \otimes \P(\y)$ to represent the null measure. We emphasize that, since $\bbeta \distas{} \mathcal{P}_{\|\beta\|_2}(\{-1, 1\})$, $\|\bbeta\|^2_2$ and $k$ are interchangeable. As prescribed in the statement of the lemma, suppose that $\P\left[\rho((\hat{\bbeta_1}, \hat{\bbeta}_2), (\bbeta_1, \bbeta_2)) < \epsilon \right] \underset{(i \to \infty)}{\to} 1$ for any $\epsilon > 0$.

Define the two following events under the planted hypothesis $\rP$:
\[ \tilde{\Omega}_1 := \left\{ \{\hat{\bbeta}_1 = \bbeta_1, \hat{\bbeta}_2 = \bbeta_2 \} \cup \{ \hat{\bbeta}_2 = \bbeta_1, \hat{\bbeta}_1 = \bbeta_2\}\right\},\]
and
\[ \tilde{\Omega}_2 := \{|w_q| < | \sigma^{-1} \langle \X_q, \bbeta_2 - \bbeta_1 \rangle + w_q| ,\:\: \forall q \in [n] \}.\]
Note that by assumption, $\tilde{\Omega}_1$ occurs with probability $1 - o(1)$ under $\rP$. Indeed, we can choose $\epsilon < 1$ in the definition of our given algorithm $\mathcal{A}$ and since $\hat{\bbeta}_1, \hat{\bbeta}_2, \bbeta_1, \bbeta_2 \in \{-1, 0, 1\}^p$ we obtain that $\P\left[\tilde{\Omega}_1 \right] \to 1$. 
%
%
%
We first consider the planted hypothesis $\rP$. Let $\nu_q \distas{\text{i.i.d.}} \mathcal{N}(0, (1 - \xi \tau) \frac{k}{\sigma^2})$, and $g_q \distas{} \mathcal{N}(0, 1)$, independent from each other and from $w_q$, for $q \in [n]$. In this case we have by symmetry that
\begin{align*}
\rP\left[\tilde{\Omega}_2^\complement \middle| \tilde{\Omega}_1 \right] &= \rP\left[ \{|w_q| \geq | \sigma^{-1} \langle \X_q, \bbeta_2 - \bbeta_1 \rangle + w_q| ,\:\: \forall q \in [n] \}\right] \\
&= \int \rP\left[ \{|w_q| \geq |\nu_q + w_q| ,\:\: \forall q \in [n] \} \middle| w_q \right] \rP[dw_q] \\
&= 2 \int_{0}^{\infty} \rP\left[ \nu_q \in [-2w_q, 0] ,\:\: \forall q \in [n] \} \middle| w_q \right] \rP[dw_q] \\
&= 2 \int_{0}^{\infty} \rP\left[ \{g_q \in [-2w_q/((1 - \xi \tau) k/\sigma^2), 0] ,\:\: \forall q \in [n] \} \middle| w_q \right] \rP[dw_q],
\end{align*}
%
%
%
%
where we have
\[ \rP\left[  \{g_q \in [-2w_q/((1 - \xi \tau) k/\sigma^2), 0] ,\:\: \forall q \in [n] \} \middle| w_q \right] \to 0 \; \; \text{ as } {k/\sigma^2 \to \infty}, \]
and $\rP\left[ \{g \in [-2w_q/((1 - \xi \tau) k/\sigma^2), 0] ,\:\: \forall q \in [n] \} \middle| w_q \right] \leq 1$, so we can apply the Dominated Convergence Theorem to obtain that
\begin{align*}
\rP\left[\tilde{\Omega}_2^\complement \middle| \tilde{\Omega}_1 \right] &= 2 \int_{0}^{\infty} \rP\left[ \{g_q \in [-2w_q/((1 - \xi \tau) k/\sigma^2), 0] ,\:\: \forall q \in [n] \} \middle| w_q \right] \rP[dw_q] \to 0 \text{ as } k/\sigma^2 \to \infty.
\end{align*}
We therefore have that $\rP\left[\tilde{\Omega}_2 \middle| \tilde{\Omega}_1 \right] = 1 - o(1)$, and hence $\rP \left[ \left( \tilde{\Omega}_1, \tilde{\Omega}_2\right)\right] = 1 - o(1)$. 

Next, note that under the planted hypothesis $\rP$ and in the joint event $(\tilde{\Omega}_1, \tilde{\Omega}_2)$ we have for indices $q$ such that $z_q = 1$:
\begin{align*}
|y_q - \frac{1}{\sigma} \langle \X_q, \hat{\bbeta}_1 \rangle| &= |y_q - \frac{1}{\sigma} \langle \X_q, \bbeta_1 \rangle| \\
&= |w_q| \\
&< |\frac{1}{\sigma} \langle \X_q, \bbeta_1 - \bbeta_2 \rangle + w_q| \\
&= |y_q - \frac{1}{\sigma} \langle \X_q, \bbeta_2 \rangle| \\
&= |y_q - \frac{1}{\sigma} \langle \X_q, \hat{\bbeta}_2 \rangle|.
\end{align*}
An analogous statement with $\hat{\bbeta}_1$ and                    $\hat{\bbeta}_2$ swapped holds for indices $q$ such that $z_{q} = 0$. We therefore have that under $(\tilde{\Omega}_1, \tilde{\Omega}_2)$ we can exactly estimate $\z$ using the above thresholding procedure, and we call this exact estimate $\hat{\z}$.
We then define our detection algorithm in this case:
\[
\mathcal{A'}\left({\begin{bmatrix} \X \\ \y \end{bmatrix}}\right) = \begin{cases}
\mathtt{p}, & n^{-1/2} \|\y - \frac{1}{\sigma} \X \hat{\bbeta}_1 \odot \hat{\z} - \frac{1}{\sigma} \X \hat{\bbeta}_2 \odot (1 - \hat{\z})\| \leq \sqrt{5} \\
\mathtt{q}, & n^{-1/2} \|\y - \frac{1}{\sigma} \X \hat{\bbeta}_1 \odot \hat{\z} - \frac{1}{\sigma} \X \hat{\bbeta}_2 \odot (1 - \hat{\z})\| > \sqrt{5}
\end{cases}.
\]
We will proceed to prove that $\mathcal{A'}$ has vanishing Type II error. Indeed, under $\rP$ and under the high-probability event $(\tilde{\Omega}_1, \tilde{\Omega}_2)$ we have:
\begin{align*}
&\|\y - \frac{1}{\sigma} \X \hat{\bbeta}_1 \odot \hat{\z} - \frac{1}{\sigma} \X \hat{\bbeta}_2 \odot (1 - \hat{\z})\|_2 \\
&= \|\frac{1}{\sigma} \X \bbeta_1 \odot \z + \frac{1}{\sigma} \X \bbeta_2 \odot (1 - \z) + \w - \frac{1}{\sigma} \X \hat{\bbeta}_1 \odot \hat{\z} - \frac{1}{\sigma} \X \hat{\bbeta}_2 \odot (1 - \hat{\z})\|_2\\
&\leq \|\w\|_2 + \|\frac{1}{\sigma} \X \bbeta_1 \odot \z + \frac{1}{\sigma} \X \bbeta_2 \odot (1 - \z) - \frac{1}{\sigma} \X \hat{\bbeta}_1 \odot \hat{\z} - \frac{1}{\sigma} \X \hat{\bbeta}_2 \odot (1 - \hat{\z}) \|_2 \\
&= \|\w\|_2.
\end{align*}
%
%
%
%
%
We therefore have 
\begin{align*}
&\rP\left[\mathcal{A'}\left({\begin{bmatrix} \X \\ \y \end{bmatrix}}\right) = \mathtt{q} \right] \\
&= \rP\left[\mathcal{A'}\left({\begin{bmatrix} \X \\ \y \end{bmatrix}}\right) = \mathtt{q} \middle| (\tilde{\Omega}_1, \tilde{\Omega}_2)^{\complement}\right] \cdot \rP\left[(\tilde{\Omega}_1, \tilde{\Omega}_2)^\complement \right] + \rP\left[\left\{ \mathcal{A'}\left({\begin{bmatrix} \X \\ \y \end{bmatrix}}\right) = \mathtt{q} \right \} \cap (\tilde{\Omega}_1, \tilde{\Omega}_2)\right] \\
&\leq \rP\left[\mathcal{A'}\left({\begin{bmatrix} \X \\ \y \end{bmatrix}}\right) = \mathtt{q} \middle| (\tilde{\Omega}_1, \tilde{\Omega}_2)^{\complement}\right] \cdot \rP\left[(\tilde{\Omega}_1, \tilde{\Omega}_2)^\complement \right] + \rP\left[\left\{ \|\w\|_2 > \sqrt{5n} \right \} \cap (\tilde{\Omega}_1, \tilde{\Omega}_2)\right] \\
&\leq 1 \cdot o(1) + \rP\left[\left\{ \|\w\|_2 > \sqrt{5n} \right \}\right] \\
&\leq 1 \cdot o(1) + e^{-n} = o(1),
\end{align*}
where the last inequality is obtained using a standard  chi-square large deviation tail bounds (Example 2.11 of \cite{wainwright_high-dimensional_2019}):
\begin{align*}
 \rP\left[\|\w\|_2 < \sqrt{n + 2 \sqrt{nt} + 2t} \right] \geq 1 - e^{-t},  
\end{align*}
and taking $t= \sqrt{n}$.

We now turn to showing that $\mathcal{A'}$ has vanishing Type I error. Under the null hypothesis $\rQ$ we have by definition of $\psi$ in Definition \ref{def:psi2} that
\[
\|\y - \frac{1}{\sigma} \X \hat{\bbeta}_1 \odot \hat{\z} - \frac{1}{\sigma} \X \hat{\bbeta}_2 \odot (1 - \hat{\z})\|_2 \geq \psi,
\]
where we recall from Lemma \ref{lem:psi2} that
\[\rQ\left[\psi \geq e^{-3/2} \exp\left(-\frac{2k(\log{p}+1)}{n} \right) \sqrt{k + \sigma^2} \right] \geq 1 - e^{-n}, \]
hence it would suffice to show that
\[ e^{-3/2} \exp\left(-\frac{2k(\log{p}+1)}{n} \right) \sqrt{k + \sigma^2} > \sqrt{5}. \]
In order to do so, choose $n^* = \frac{4k(\log{p}+1)}{\log{(1 + \frac{k}{\sigma^2})} - \log{5} - 3} = \Theta(\frac{4k\log{p}}{\log(1 + \frac{k}{\sigma^2})})$ and notice that if the inequality holds for $n^*$, it must hold for all $n \geq n^*$ since the left hand side is increasing with $n$. We plug in $n^*$ to obtain
\[ e^{-3/2} \exp\left(-1/2 \log(1 + \frac{k}{\sigma^2}) + \log{\sqrt{5}} + 3/2\right) \sqrt{1 + 2 \frac{k}{\sigma^2}} = \frac{\sqrt{1 + 2\frac{k}{\sigma^2}}}{\sqrt{1 + \frac{k}{\sigma^2}}} \sqrt{5} > \sqrt{5}, \]
and therefore we have that
\begin{align*}
&\rQ\left[ \mathcal{A'}\left({\begin{bmatrix} \X \\ \y \end{bmatrix}}\right) = \mathtt{q} \right] \\
&= \rQ\left[ \|\y - \frac{1}{\sigma} \X \hat{\bbeta}_1 \odot \hat{\z} - \frac{1}{\sigma} \X \hat{\bbeta}_2 \odot (1 - \hat{\z})\|_2 > \sqrt{5} \right] \\
&\geq \rQ\left[\psi > \sqrt{5} \right] \\
&\geq \rQ\left[\psi \geq e^{-3/2} \exp\left(-\frac{2k(\log{p}+1)}{n} \right) \sqrt{k + \sigma^2} \right] \\
&\geq 1 - e^{-n}.
\end{align*}
Importantly, note that $n$ satisfies the constraints of $\mathcal{C}_{2}$.
\end{proof}
\subsection{Reduction from $\sbmslr$ to $\psbmslr$} \label{subsec:sbmslr-psbmslr-red}
\begin{proof}[Proof of Lemma  \ref{lemma:mslr_c_red}]
First, recall that the $\psbmslr$  regime implies, 
\[ \phi = 1/2 \text{ and } \bbeta_{1, j} = -\bbeta_{2, j} \text{ for } j \in J \subseteq \supp(\bbeta_1) \cap \supp(\bbeta_2) \text{ with } C_1 k \leq |J| \leq C_2 k, \]
for some constants $1 \geq C_1, C_2 > 0$. Without loss of generality, we can take $|J| = C k$ for some constant $0 < C \leq 1$, and all constants that follow can be lower bounded or upper bounded accordingly. In light of this, $\psbmslr$  corresponds to $\mslr_{\xi, \tau}$ for some $\tau \in (-1, 1)$ and some $\xi > 0$, where we recall that $\xi, \tau$ are of constant order, i.e., do not scale with respect to $n, p, k$. 

For brevity, we denote $\fP' := \psbmslr$ and $\fP := \sbmslr$. Let $c = 1 - \frac{\tau \cdot \xi + 1}{2} \in (0, 1)$ denote the proportion of matching non-zero entries with opposite sign between $\bbeta_1$ and $\bbeta_2$ (intuitively, this corresponds to the ``hard'' portion of the signal), and note that it is fixed. Note that this follows since $\tau \cdot \xi = \frac{\langle \bbeta_1, \bbeta_2 \rangle}{k}$ for $\xi > 0$.

Given a sequence of parameters $\{p'_i, n'_i, k'_i, \sigma'_i \}_{i=1}^{\infty} \subseteq \mathcal{C}_1$ for $\fP'$, consider the sequence of parameters $\{(p_i, n_i, k_i, \sigma_i) \}_{i=1}^{\infty} = \{(c p'_i, n'_i, c k'_i, \sigma'_i)\}_{i=1}^{\infty}$ for $\fP$. Notice that $\{(p_i, n_i, k_i, \sigma_i) \}_{i=1}^{\infty} \subseteq \mathcal{C}_1$ since (dropping the subscript $i$ notation for convenience):
\begin{itemize}
	\item $k = c k' = o(c C \sqrt{p'}) = o(c^{3/2} C \sqrt{p}) = o(\sqrt{p})$
	\item $n = n' = o\left( (k' + (\sigma')^2)^2 \cdot \frac{1}{\log{p'}} \right) = o\left( (k + \sigma^2)^2 \cdot \frac{1}{\log{p}} \right)$,
	\item $n = n' = \omega(\max\{k', \log{p'} \}) = \omega(\max\{k, \log{p} \})$,
\end{itemize}
and hence the parameter regimes of $\fP$ are also contained in $\mathcal{C}_1$.
For $i \in \mathbb{N}$, let $J = (\X, \y)$ denote an instance of $\fP$ with parameters $(p_i, n_i, k_i, \sigma_i)$, where we recall:
\begin{align*}
    &\begin{bmatrix}
        \X \\
        \y
    \end{bmatrix} = \begin{bmatrix}
           \X  \\
			     \frac{1}{\sigma} \X \bbeta_1 \odot \z + \frac{1}{\sigma} \X \bbeta_2 \odot (1-\z) + \w
         \end{bmatrix}.
\end{align*}
We now want to show that, given a sequence of parameters $\{p'_i, n'_i, k'_i \}_{i=1}^{\infty} \subseteq \mathcal{C}_1$ for $\fP'$ and a randomized polynomial-time algorithm $\mathcal{A'}$ solving it, we can construct a randomized polynomial-time algorithm $\mathcal{A}$ solving $\fP$ along the above parameter sequence $\{p_i, n_i, k_i \}_{i=1}^{\infty} \subseteq \mathcal{C}_1$.
We will construct our desired algorithm $\mathcal{A}$ by composing $\mathcal{A'}$ with a pre- and post-processing step. Indeed, we let $\mathcal{A} = \mathcal{B} \circ \mathcal{A'} \circ \mathcal{D}$, where we define $\mathcal{B}$ and $\mathcal{D}$ below.

First, let $\mathtt{RS}$ denote the random variable that reshuffles entries of a given size $p$ vector or the columns of a given $p$-column matrix according to a uniform shuffling of the index set $[p]$. Let $\mathtt{IRS}$ denote the random variable which inverts this reshuffling process on a vector or matrix, such that $\mathtt{IRS} \circ \mathtt{RS}$ is the identity operation. Let $\mathbbm{1}$ denote the all-ones vector of any size (to be inferred from context). Let $\bar{\mathbbm{1}}_1$ and $\bar{\mathbbm{1}}_2$ denote two independent copies of a $(1-c)k'$ sparse vector in $\{0, 1, -1\}^{(1-c)p'}$. We now define our pre- and post-processing procedures Algorithms \ref{alg:phi}, \ref{alg:B}, which can be seen to run in randomized polynomial time with respect to $p$. 
\begin{algorithm}
\caption{$\mathcal{D}(\X, \y)$, where $(\X, \y)$ is an instance of $\fP(p, n, k, \sigma)$}\label{alg:phi}
\; \KwData{$\V \in \mathbb{R}^{n \times (1-c)/c \: p} \text{ with columns} \distas{\text{i.i.d.}} \mathcal{N}(0, \I_n)$}
\; $\tilde{\y} \gets \y + \frac{1}{\sigma} \V \mathbbm{1}$ \; \\
$\tilde{\X} \gets \mathtt{RS}{\begin{bmatrix} \V & \X \end{bmatrix}}$ \; \\
\Return $(\tilde{\X}, \tilde{\y})$
\end{algorithm}
\begin{algorithm}
\caption{$\mathcal{B}(\tilde{\bbeta}_1, \tilde{\bbeta}_2)$, where $(\tilde{\bbeta}_1, \tilde{\bbeta}_2)$ are both in $\{0, 1, -1 \}^{p'}$}\label{alg:B}
${\begin{bmatrix} \bar{\mathbbm{1}}_1 \\ \bbeta_1 \end{bmatrix}} \gets \mathtt{IRS}(\tilde{\bbeta}_1)$ \; \\
${\begin{bmatrix} \bar{\mathbbm{1}}_2 \\ \bbeta_2 \end{bmatrix}} \gets \mathtt{IRS}(\tilde{\bbeta}_2)$ \; \\
\Return $(\bbeta_1, \bbeta_2)$
\end{algorithm}
\begin{proposition} \label{claim:D}
For $i \in \mathbb{N}$, $(\X, \y)$ an instance of $\fP(p_i, n_i, k_i, \sigma_i)$, and $(\X', \y')$ an instance of $\fP'(p'_i, n'_i, k'_i)$, we have that $\mathcal{D}(\X, \y) \overset{d}{=} (\X', \y')$.
\end{proposition}

\begin{proof}
Let $(\tilde{\X}, \tilde{\y}) := \mathcal{D}(\X, \y)$. First note that $\tilde{\X} \overset{d}{=} \X'$ since ${\begin{bmatrix} \V & \X \end{bmatrix}}$ has dimensions $n \times (p + \frac{1-c}{c} p) = n' \times (cp' + (1-c)p') = n' \times p'$.
Next, note that we can decompose $\tilde{\y}$ as follows:
\begin{align*}
\tilde{\y} &= \y + \frac{1}{\sigma} \V \mathbbm{1} \\
&= \frac{1}{\sigma} \X \bbeta_1 \odot \z + \frac{1}{\sigma} \X \bbeta_2 \odot (1-\z) + \frac{1}{\sigma} \V \mathbbm{1} + w  \\
&= \frac{1}{\sigma} \X \bbeta_1 \odot \z + \frac{1}{\sigma} \X \bbeta_2 \odot (1-\z) + \frac{1}{\sigma} \V \mathbbm{1} \odot \z + \frac{1}{\sigma} \V \mathbbm{1} \odot (1-\z) + w \\
&= \frac{1}{\sigma} (\mathtt{RS} {\begin{bmatrix} \V & \X \end{bmatrix}}) \left(\mathtt{RS} {\begin{bmatrix} \bar{\mathbbm{1}}_1 \\ \bbeta_1 \end{bmatrix}}\right) \odot \z + \frac{1}{\sigma} \left(\mathtt{RS} {\begin{bmatrix} \V & \X \end{bmatrix}}\right) \left(\mathtt{RS} {\begin{bmatrix} \bar{\mathbbm{1}}_2 \\ \bbeta_2 \end{bmatrix}}\right) \odot (1-\z) + w
\end{align*}
and hence $\tilde{\y}$ is the output of a $\mslr$ model with size $cp' + (1-c)p' = p'$ signals $\left(\mathtt{RS} {\begin{bmatrix} \bar{\mathbbm{1}}_1 \\ \bbeta_1 \end{bmatrix}}, \mathtt{RS} {\begin{bmatrix} \bar{\mathbbm{1}}_2 \\ \bbeta_2 \end{bmatrix}}\right)$ containing $c k'$ non-zero opposing sign entries from $(\bbeta_1, \bbeta_2)$ and $(1-c)k'$ remaining non-zero entries from appending $\bar{\mathbbm{1}}$, for a total support of size $k'$. Additionally, these signals are linearly transformed by a design matrix that is i.i.d Gaussian and $n' \times p'$ as mentioned above. The first and second point together imply that $\tilde{\y}$ is the output of a $\mslr_{\xi, \tau}$ model and  $\mathcal{D}(\X, \y) = (\tilde{\X}, \tilde{\y}) \overset{d}{=} (\X', \y')$.
\end{proof}
Following Proposition \ref{claim:D}, all that is left to show is that $\P [\| \mathcal{B} \circ \mathcal{A'} \circ \mathcal{D}(\X, \y) - (\bbeta_1, \bbeta_2)\|_{\infty} > 0] \to 0$. Indeed, the following steps hold due to Proposition \ref{claim:D} and the definition of $\mathcal{B}$:
\begin{align*}
&\|\mathcal{A'}(\X', \y') - (\bbeta'_1, \bbeta'_2) \|_{\infty} \\
\overset{d}{=}& \left\|\mathcal{A'}(\mathtt{RS}({\begin{bmatrix} \V & \X \end{bmatrix}}), \y + \frac{1}{\sigma} \V \mathbbm{1}) - \left( \mathtt{RS} {\begin{bmatrix} \bar{\mathbbm{1}}_1 \\ \bbeta_1 \end{bmatrix}}, \mathtt{RS} {\begin{bmatrix} \bar{\mathbbm{1}}_2 \\ \bbeta_2 \end{bmatrix}} \right) \right\|_{\infty} \\
=& \left\|\mathcal{A'} \circ \mathcal{D}(\X, \y) - \left( \mathtt{RS} {\begin{bmatrix} \bar{\mathbbm{1}}_1 \\ \bbeta_1 \end{bmatrix}}, \mathtt{RS} {\begin{bmatrix} \bar{\mathbbm{1}}_2 \\ \bbeta_2 \end{bmatrix}} \right) \right\|_{\infty} \\
\geq &\left\|\mathcal{B} \circ \mathcal{A'} \circ \mathcal{D}(\X, \y) - \mathcal{B}\left( \mathtt{RS} {\begin{bmatrix} \bar{\mathbbm{1}}_1 \\ \bbeta_1 \end{bmatrix}}, \mathtt{RS} {\begin{bmatrix} \bar{\mathbbm{1}}_2 \\ \bbeta_2 \end{bmatrix}} \right) \right\|_{\infty} \\
=&\left\|\mathcal{A}(\X, \y) - ( \bbeta_1 , \bbeta_2 ) \right\|_{\infty},
\end{align*}
and hence:
\begin{align*}
\P\left[ \|\mathcal{A'}(\X', \y') - (\bbeta'_1, \bbeta'_2) \|_{\infty} > 0\right] \to 0  \implies  \P\left[ \left\|\mathcal{A}(\X, \y) - ( \bbeta_1 , \bbeta_2 ) \right\|_{\infty}  > 0 \right] \to 0,
\end{align*}
where $\mathcal{A}$ runs in randomized polynomial time with respect to the input size $p$ since it is a composition of randomized polynomial time procedures with respect to $p$.
\end{proof}

\subsection{Reduction from $\sbmslrd$ to $\spr$ } \label{sec:spr-hardness}
In this section, we present a polynomial-time reduction from strong detection in noiseless $\sbmslrd$ to exact support recovery in $\spr$ , for signals with non-zero entries in $\{-1, 0, 1\}^p$.
More specifically, we consider the following \textit{symmetric} $\slr$ problem.
\begin{definition}[$\sslr$] \label{def:S-SLR}
For $\X \in \mathbb{R}^{n \times p}$, $\w \in \mathbb{R}^{n}$, and $\z \in \mathbb{R}^n$, consider the model:
\[ \y = g\left( \X \bbeta \right) + \w,  \]
where $\odot$ denotes element-wise product between vectors, $X_{i, j} \distas{\text{i.i.d.}} \mathcal{N}(0, 1)$, $w_i \distas{\text{i.i.d.}} \mathcal{N}(0, \sigma^2)$, $\bbeta \in \reals^{p}$ each $k$-sparse, and $g : \reals^{n} \to \reals^{n}$ a separable entry-wise even function ($g_i(x) = g_i(-x)$ for $i \in [n]$ and $x \in \reals$). Given $(\X, \y)$ the objective is to estimate $\bbeta$.
\end{definition}
We note that setting $g_i(x) = |x|$ and $g_i(x) = x^2$ yields two standard formulations of the phase retrieval problem with sparse signals ($\spr$) \citep{candes_phase_2015, liu_towards_2021}. We seek to show hardness within the parameter scaling regime of Theorem \ref{thm:sbmslrd-hardness-cor-snr} in the noiseless case ($ \frac{\snr + 1}{\snr } = 1$) , and hence we prove a reduction within the constraint set $\mathcal{C}_3$:
\begin{align}
& \mathcal{C}_3 = \left\{(p_i, n_i, k_i, \sigma_i)_{i=1}^{\infty} \subset \mathbb{N}^4: \exists C \in \mathbb{R}_{+} \, \text{ s.t. }
\, p_i = \omega_{i}(1), k_i = o(\sqrt{p_i}),  ,  \right. \nonumber \\
& \hspace{4em} \left. n_i = \omega(\max\{k_i, \log{p_i}\}) \, n_i = o\left( k_i^2 \cdot \frac{1}{\log{p_i}} \right) \right\}. \label{eq:C_4}
\end{align}
We note that $\mathcal{C}_3$ is just $\mathcal{C}_1$ in \eqref{eq:C_3}, but with the stricter constraint $n_i = o\left( k_i^2 \cdot \frac{1}{\log{p_i}} \right)$. 
We begin with a  lemma, which we use to initiate our reduction in Theorem \ref{thm:spr-hardness}. 
\begin{lemma} \label{lemma:gslr_red}
Fix signal priors to be $\mathcal{P}_{\|\beta\|_2}(\{-1, 1\})$. For any sequence of parameters $\{(p'_i, n'_i, k'_i, \sigma'_i) \}_{i = 1}^{\infty} $ in $\mathcal{C}_3$ for $\sslr$ with solution $\bbeta'$ and problem instances $(\X', \y')$, there exists a sequence of parameters $\{(p_i, n_i, k_i, \sigma_i = 0) \}_{i=1}^{\infty}$ in  $\mathcal{C}_{3}$ for $\sbmslr$ (noiseless) with solution $\bbeta_1, \bbeta_2$ and problem instances $(\X, \y)$ such that, for any randomized polynomial time algorithm $\mathcal{A'}$ for $\sslr$  producing $\hat{\bbeta'}$ with
\[ \P \left[ \|\hat{\bbeta'} - \bbeta \|_{\infty} > 0 \right] \to 0, \]
we can construct a second randomized polynomial time algorithm $\mathcal{A}$ for $\sbmslr$ outputting $(\hat{\bbeta_1}, \hat{\bbeta}_2)$ such that
\[ \P \left[ \|(\hat{\bbeta_1}, \hat{\bbeta}_2) - (\bbeta_1, \bbeta_2) \|_{\infty} > 0 \right] \to 0.\]
\end{lemma}
\begin{proof}
We drop the subscript $i$ notation for convenience. For brevity, let $\fP' := \sslr$, and $\fP := \sbmslr$.  Given a sequence of parameters $\{p'_i, n'_i, k'_i, \sigma'_i \}_{i=1}^{\infty}$ in $\mathcal{C}_3$ for $\fP'$, consider the sequence of parameters $\{(p_i, n_i, k_i, \sigma_i) \}_{i=1}^{\infty} = \{(p'_i, n'_i, k'_i, 0)\}_{i=1}^{\infty}$ for $\fP$. Notice that $\{(p_i, n_i, k_i, \sigma_i) \}_{i=1}^{\infty}$ in $\mathcal{C}_3$. 
Let $(\X, \y)$ be a problem instance of $\fP$ with parameters $\{(p_i, n_i, k_i, \sigma_i) \}_{i=1}^{\infty}$, and recall that for noiseless $\sbmslr$ the observation is of the form $(\X, \y = \X \bbeta \odot (2\z - 1))$. We apply a preprocessing step to construct $\tilde{\y} := g(\y) + \w \in \reals^n$ were $\w \distas{\text{i.i.d.}} \mathcal{N}(0, 1)$, done in randomized polynomial time. Notice that $(\X, \tilde{\y})$ is now an instance of $\fP'$, by virtue of $g$ being symmetric with respect to sign flips. 
We then run algorithm $\mathcal{A}'$ on $(\X, \tilde{\y})$ to obtain $\hat{\bbeta}'$ with $\P \left[ \|\hat{\bbeta'} - \bbeta \|_{\infty} > 0 \right] \to 0$ as per the problem statement. Without loss of generality setting $(\hat{\bbeta}_1, \hat{\bbeta}_2) = (\hat{\bbeta}, -\hat{\bbeta})$, we have an algorithm $\mathcal{A}$ which yields $\P \left[ \|(\hat{\bbeta_1}, \hat{\bbeta}_2) - (\bbeta_1, \bbeta_2) \|_{\infty} > 0 \right] \to 0$. 
\end{proof}
\begin{theorem}[Reduction from $\sbmslrd$ to exact recovery in $\spr$ ] \label{thm:spr-hardness}
Consider the setting of $\spr$  with joint signal prior $\mathcal{P}_{\|\beta\|_2}(\{-1, 1\})$. Any randomized polynomial-time algorithm $\mathcal{A}$ solving $\spr$  in parameter regime $\mathcal{C}_2 \cap \mathcal{C}_3$ and $\snr = \omega(1)$ would contradict Theorem \ref{thm:sbmslrd-hardness-cor-snr}.
\end{theorem}
\begin{proof} 
We first reduce $\sbmslrd$ to $\sbmslr$ within the constraint set $\mathcal{C}_{2}$ \eqref{eq:C_2IT}, using Lemma \ref{lemma:mslrd-to-mslr}. We then reduce noiseless $\sbmslr$ to $\spr$  using Lemma \ref{lemma:gslr_red} within the constraint set $\mathcal{C}_{3}$ \eqref{eq:C_4} by choosing $g(x) = |x|$ or $g(x) = x^2$ depending on the precise definition of $\spr$ . Throughout, we have let $\snr = \omega(1)$ to satisfy Lemma \ref{lemma:mslrd-to-mslr}. Suppose there exists a randomized polynomial-time algorithm $\mathcal{A}$ solving exact recovery in $\spr$  with signals with non-zero entries in $\{-1, 0, 1\}^p$, i.e. $\P \left[ \|\hat{\bbeta'} - \bbeta \|_{\infty} > 0 \right] \to 0$. Then by the aforementioned chain of reductions we would have a randomized polynomial time algorithm $\mathcal{A}'$ solving strong detection in $\sbmslrd$  with signals with non-zero entries in $\{-1, 0, 1\}^p$ in the scaling regime $\mathcal{C}_{2} \cap \mathcal{C}_{3}$, which is included in the parameter regime stated in Theorem \ref{thm:sbmslrd-hardness-cor-snr} and would hence contradict Theorem \ref{thm:sbmslrd-hardness-cor-snr}.
\end{proof}
For completeness, we also include a reduction from $\sbmslrd$ to a detection variant of $\spr$  in Theorem \ref{thm:sprd_red}. 
\begin{definition}[Detection Variant $\sprd$ ] \label{def:sprd}
For $\X \in \mathbb{R}^{n \times p}$, $\sigma > 0$, and $\w^{(1)}, \w^{(2)} \in \mathbb{R}^{n}$, consider the following hypothesis testing problem:
\begin{align*}
  &  {\P(\X) \otimes \P(\y)}: \begin{bmatrix}
        \X \\
        \y
    \end{bmatrix} = \begin{bmatrix}
          \X \\
          \sqrt{\frac{\|\bbeta\|^2_2}{\sigma^2}} |\w_1| + \w_2 \\
         \end{bmatrix} \\
  &  {\P(\X, \y)}: \begin{bmatrix}
        \X \\
        \y
    \end{bmatrix} = \begin{bmatrix}
          \X \\
          \frac{1}{\sigma} |\X \bbeta| + \w
         \end{bmatrix}
\end{align*}
where $(\bbeta_1, \bbeta_2) \sim \mathcal{P}_{\|\bbeta\|_2}(\mathcal{D})$, and $X_{i, j} \distas{\text{i.i.d.}} \mathcal{N}(0, 1)$, $w_i \distas{\text{i.i.d.}} \mathcal{N}(0, 1)$, $z_i \distas{\text{i.i.d.}} \text{Bernoulli}(\phi)$. The task is to construct a function $f$ which strongly distinguishes $\P(\X) \otimes \P(\y)$ from $\P(\X, \y)$.
\end{definition}
\begin{theorem} \label{thm:sprd_red}
Fix signal priors to be $\mathcal{P}_{\|\beta\|_2}(\{-1, 1\})$. For any sequence of parameters $\{(p'_i, n'_i, k'_i, \sigma'_i) \}_{i = 1}^{\infty} $ in $\mathcal{C}_3$ for $\sprd$ with signal $\bbeta'$ and problem instances $(\X', \y')$, there exists a sequence of parameters $\{(p_i, n_i, k_i, \sigma_i = 0) \}_{i=1}^{\infty} \subseteq \mathcal{C}_{3}$ for $\sbmslrd$ (noiseless) with signals $\bbeta_1, \bbeta_2$ and problem instances $(\X, \y)$ such that, for any randomized polynomial time algorithm $\mathcal{A'}$ solving strong detection in $\fP'$, we can construct a second randomized polynomial time algorithm $\mathcal{A}$ for solving strong detection in $\fP$. 
\end{theorem}
\begin{proof}
We drop the subscript $i$ notation for convenience. For brevity, let $\fP' := \sprd$, and $\fP := \sbmslrd$. Given a sequence of parameters $\{p'_i, n'_i, k'_i, \sigma'_i \}_{i=1}^{\infty}$ in $\mathcal{C}_3$ for $\fP'$, consider the sequence of parameters $\{(p_i, n_i, k_i, \sigma_i) \}_{i=1}^{\infty} = \{(p'_i, n'_i, k'_i, 0)\}_{i=1}^{\infty}$ for $\fP$. Notice that $\{(p_i, n_i, k_i, \sigma_i) \}_{i=1}^{\infty}$ is in $\mathcal{C}_3$. 
Let $(\X, \y)$ be a problem instance of $\fP$ with parameters $\{(p_i, n_i, k_i, \sigma_i) \}_{i=1}^{\infty}$, and recall that in noiseless $\sbmslrd$ the observation in the alternative hypothesis is of the form $(\X, \y = \X \bbeta \odot (2\z - 1) )$. We apply a preprocessing step to construct $\tilde{\y} := |\y| + \w \in \reals^n$ were $\w \distas{\text{i.i.d.}} \mathcal{N}(0, 1)$, done in randomized polynomial time. Notice that under both hypotheses, $(\X, \tilde{\y})$ is an instance of $\fP'$, by virtue of $g$ being symmetric with respect to sign flips (even). We can then run algorithm $\mathcal{A}'$ on $(\X, \tilde{\y})$ to solve the hypothesis testing problem of $\fP$ on $(\X, \y)$.
\end{proof}
\section{Proofs for efficient algorithms}
\label{sec:Proofs-algorithms}
\subsection{$\mathtt{CORR}$ for support recovery in $\mslr$} \label{subsec:corr-mslr-proofs}
\begin{theorem} [$\mathtt{CORR}$ achieves joint support recovery in $\mslr$] \label{thm:CORR-mslr-r-general}
Consider the general setting of $\mslr$, $(\bbeta_1, \bbeta_2) \distas{} \mathcal{P}_{\|\bbeta\|_2}(\mathcal{D})$. Let $\epsilon \in (0, 1)$ be the one used in $\corr$. Then provided 
\[n \geq \frac{32 (1 + \epsilon) }{\min\{\phi^2 \bbeta^2_{\texttt{min}}, (1-\phi)^2 \bbeta^2_{\texttt{min}}, \langle \bbeta \rangle^2_{\texttt{min}} \}} \frac{k (\snr + 1)}{\snr} \log{2p}, \] we have that $\mathtt{CORR}$ outputs the exact joint support of signals $\bbeta_1$ and $\bbeta_2$ with probability at least $1 - c_1(\frac{k}{p} + ke^{-c_2 n} + \frac{k}{n} + \frac{1}{p^{c_2}})$ for constants $c_1, c_2 > 0$.
\end{theorem}

\begin{proof} 
Let $\delta > 0$ to be chosen later, and for two sets $A$ and $B$ denote the symmetric difference $A \Delta B := (A \setminus B) \cup (B \setminus A)$. Let $\tau := \sqrt{2 (1 + \epsilon/2) \log{2p}}$ and define the error event
\[\mathcal{E} := \cup_{j \in \mathcal{S}_1 \cup \mathcal{S}_2} \left\{\left|\frac{\langle \X_j, \y \rangle}{\|\y\|_2}\right| < \tau \right\} \cup \left\{\max_{q \in (\mathcal{S}_1 \cup \mathcal{S}_2)^\complement} \left|\frac{\langle \X_q, \y \rangle}{\|\y \|_2}\right| \geq \tau \right\}. \]
From here we partition the set of indices $j \in \mathcal{S}_1 \cup \mathcal{S}_2$ into two sets,
\[J_a := \{j \in [p]: j \in \mathcal{S}_1 \cap \mathcal{S}_2\},  \qquad J_b := \{j \in [p]: j \in \mathcal{S}_1 \Delta \mathcal{S}_2\}, \]
with respect to which we perform a union bound:
\begin{align}
& \mathbb{P}\left[\mathcal{E} \right] \leq \mathbb{P}\left[\cup_{j \in \mathcal{S}_1 \cup \mathcal{S}_2} \left\{\left|\frac{\langle \X_j, \y \rangle}{\|\y\|_2}\right| < \tau \right\}\right] + \mathbb{P}\left[\max_{q \in (\mathcal{S}_1 \cup \mathcal{S}_2)^\complement} \left|\frac{\langle \X_q, \y \rangle}{\|\y \|_2}\right| \geq \tau \right] \nonumber \\
&\leq \mathbb{P}\left[\cup_{j_a \in J_a}  \left\{\left|\frac{\langle \X_{j_a}, \y \rangle}{\|\y\|_2}\right| < \tau \right\}\right] + \mathbb{P}\left[\cup_{j_b \in J_b} \left\{\left|\frac{\langle \X_{j_b}, \y \rangle}{\|\y\|_2}\right| < \tau \right\}\right] + \mathbb{P}\left[\max_{q \in (\mathcal{S}_1 \cup \mathcal{S}_2)^\complement} \left|\frac{\langle \X_q, \y \rangle}{\|\y \|_2}\right| \geq \tau \right] \nonumber \\
&\leq 2 k \mathbb{P}\left[\left|\frac{\langle \X_{j_a}, \y \rangle}{\|\y\|_2}\right| < \tau \right] + 2k\mathbb{P}\left[\left\{\left|\frac{\langle \X_{j_b}, \y \rangle}{\|\y\|_2}\right| < \tau \right\} \right] +  \mathbb{P}\left[\max_{q \in (\mathcal{S}_1 \cup \mathcal{S}_2)^\complement} \left|\frac{\langle \X_{q}, \y \rangle}{\|\y \|_2}\right| \geq \tau \right], \label{eq:firstthree}
\end{align}
where $j_a \in J_a$, $j_b \in J_b$, and $q \in \left(\mathcal{S}_1 \cup \mathcal{S}_2 \right)^\complement$. Analyzing the last term, we note that for $q \notin \mathcal{S}_1 \cap \mathcal{S}_2$ we have  $\frac{\langle \X_{q}, \y \rangle}{\|\y \|_2} \distas{\text{i.i.d.}} \mathcal{N}(0, 1)$.  We therefore have,
%
\begin{align*}
& \mathbb{P}\left[\max_{q \in (\mathcal{S}_1 \cup \mathcal{S}_2)^\complement} \left|\frac{\langle \X_{q}, \y \rangle}{\|\y \|_2}\right| \geq \sqrt{2 (1 + \epsilon/2) \log{2p}} \right] \\
&\stackrel{i)}{\leq} \mathbb{P}\left[\max_{q \in (\mathcal{S}_1 \cup \mathcal{S}_2)^\complement} \left|\frac{\langle \X_{q}, \y \rangle}{\|\y \|_2}\right| \geq \sqrt{2 \log{2p}} + \frac{\epsilon}{2 \sqrt{8}}\sqrt{\log{2p}}\right]
\stackrel{ii)}{\leq} (2p)^{- \frac{1}{16} \left( \frac{\epsilon}{2} \right)^2},
\end{align*}
where $i)$ follows from $\sqrt{2 (1 + \epsilon/2)} \geq \sqrt{2} + \frac{\epsilon}{2 \sqrt{8}}$, and $ii)$ from a tail bound on the maximum of standard Gaussians (see $\mathbb{P}\big[\Omega^{\complement}_2(t) \big]$ in Lemma \ref{lem:high-prob-events}).
Applying Lemmas \ref{lemma:A_bound_general} and \ref{lemma:B_bound_general} respectively (and choosing $\delta > 0$ small enough to satisfy these) to the first two terms in (\ref{eq:firstthree}) we obtain that, for some constants $c_1, c_2 > 0$,
\begin{align*}
\mathbb{P}\left[\mathcal{E} \right] &\leq c_1 \left(\frac{k}{p} + ke^{-c_2 n} + \frac{k}{n} + \frac{1}{p^{c_2}} \right).
\end{align*}
\end{proof}

\paragraph{The principal concentration lemmas.}
\begin{lemma} [General Bound for $j \in \mathcal{S}_1 \Delta \mathcal{S}_2$] \label{lemma:B_bound_general}
Consider the setting of $\mslr$. Let $j^* \in \mathcal{S}_1 \Delta \mathcal{S}_2$. Then if $n \geq \frac{32 (1 + \epsilon) }{\min\{\phi^2, (1-\phi)^2\} \bbeta^2_{\texttt{min}}} (\|\bbeta\|^2_2 + \sigma^2) \log{2p}$ for any $\epsilon \in (0,1)$ we obtain that
\[\mathbb{P}\left[\left|\frac{\langle \X_{j^*}, \y \rangle}{\|\y\|_2} \right| \leq \sqrt{2 (1 + \epsilon/2) \log{2p}} \right] \leq \frac{1}{p} + 4 e^{- \frac{\delta^2 n}{8}} + \frac{39}{\delta^2 n}. \]
\end{lemma}
\begin{proof} 
Let $g_i \distas{\text{i.i.d.}} \mathcal{N}(0, 1)$ (independent from $\y, \z, \bbeta_1, \bbeta_2$) for $i \in [n]$ and $\delta > 0$ to be chosen later. Without loss of generality, we can assume $j^* \in \mathcal{S}_1 \setminus \mathcal{S}_2$. We begin by considering fixed $\bbeta_1$, $\bbeta_2$, $\y$ and $\z$ vectors, and hence the initial randomness of interest lies in the design matrix $\X$. Define
\begin{itemize}{}
	\item $\sigma^{(1)}_i := \frac{y_i}{\|\y\|_2} \sqrt{1 -\frac{\bbeta^2_{1,j^*}}{\|\bbeta\|^2_2+\sigma^2}}$,
	\item $\sigma^{(2)}_i := \frac{y_i}{\|\y\|_2}$,
	\item $\mu^{(1)}_i := \frac{y^2_i}{\|\y\|_2} \frac{\bbeta_{1, j^*}}{\|\bbeta\|^2_2 + \sigma^2}$,
	\item $\mu^{(2)}_i := 0$,
\end{itemize}
and notice that for $\tau \in \mathbb{R}$,
\begin{align*}
&\mathbb{P}\left[X_{j^*, i} \leq \tau | \bbeta_1, \bbeta_2, \y, \z \right] 
= \mathbb{P}\left[z_i \left(\frac{y_i \cdot \bbeta_{1, j^*}}{\|\bbeta\|^2_2 + \sigma^2} + \sqrt{1 -\frac{\bbeta^2_{1,j^*}}{\|\bbeta\|^2_2+\sigma^2}} g_i \right) + (1 - z_i) g_i \leq \tau \middle| \bbeta_1, \bbeta_2, y_i, z_i \right],
\end{align*}
which follows from Lemma \ref{lemma:cond-mslr-general}. We then have that
\[ \mathbb{P}\left[ X_{j^*, i} \frac{y_i}{\|\y\|_2} \leq \tau \middle| \y, \z\right] = \mathbb{P}\left[ z_i (\mu^{(1)}_i + \sigma^{(1)}_i g_i) + (1 - z_i)(\mu^{(2)}_i + \sigma^{(2)}_i g_i) \leq \tau \middle| \y, \z\right], \]
from which it follows that for $\tau > 0$,
\begin{align}
&\mathbb{P}\left[\left| \frac{\langle \X_{j^*}, \y\rangle}{\|\y\|_2} \right| \leq \tau \middle| \bbeta_1, \bbeta_2, \y, \z\right] \nonumber \\
&= \mathbb{P}\left[\left| \sum_{i=1}^n z_i (\mu^{(1)}_i + \sigma^{(1)}_i g_i) + (1 - z_i)(\mu^{(2)}_i + \sigma^{(2)}_i g_i) \right| \leq \tau \middle| \bbeta_1, \bbeta_2, \y, \z\right] \nonumber \\
&= \mathbb{P}\left[\left| \sum_{i=1}^n (z_i \mu^{(1)} + (1 - z_i)\mu^{(2)}) + (z_i \sigma^{(1)} + (1 - z_i) \sigma^{(2)}) g_i \right| \leq \tau \middle| \bbeta_1, \bbeta_2, \y, \z\right] \nonumber \\
&\overset{i)}{=} \mathbb{P}\left[\left| \sum_{i=1}^n (z_i \mu^{(1)} + (1 - z_i)\mu^{(2)}) + g_1 \left(\sum_{i=1}^n (z_i \sigma^{(1)}_i + (1 - z_i) \sigma^{(2)}_i)^2 \right)^{\frac{1}{2}} \right| \leq \tau \middle| \bbeta_1, \bbeta_2, \y, \z\right] \nonumber  \\
&= \mathbb{P}\left[\left\{ \sum_{i=1}^n (z_i \mu^{(1)} + (1 - z_i)\mu^{(2)}) + g_1 \left(\sum_{i=1}^n (z_i \sigma^{(1)}_i + (1 - z_i) \sigma^{(2)}_i)^2 \right)^{\frac{1}{2}} \leq \tau \right\} \right. \nonumber \\
& \left. \quad \cap \left\{-\tau \leq \sum_{i=1}^n (z_i \mu^{(1)} + (1 - z_i)\mu^{(2)}) + g_1 \left(\sum_{i=1}^n (z_i \sigma^{(1)}_i + (1 - z_i) \sigma^{(2)}_i)^2 \right)^{\frac{1}{2}} \right\} \middle| \bbeta_1, \bbeta_2, \y, \z\right] \nonumber  \\
&\overset{ii)}{\leq} \mathbb{P}\left[ g_1 \left(\sum_{i=1}^n (z_i \sigma^{(1)}_i + (1 - z_i) \sigma^{(2)}_i)^2 \right)^{\frac{1}{2}} \leq \tau - \left|\sum_{i=1}^n (z_i \mu^{(1)} + (1 - z_i)\mu^{(2)})\right| \middle| \bbeta_1, \bbeta_2, \y, \z\right] \nonumber  \\
&= \mathbb{P}\left[ g_1 \left(\sum_{i=1}^n (z_i \sigma^{(1)}_i + (1 - z_i) \sigma^{(2)}_i)^2 \right)^{\frac{1}{2}} \leq \tau - \left|\sum_{i=1}^n \frac{y^2_i / \|\y\|_2}{\|\bbeta\|^2_2 + \sigma^2} \cdot z_i \bbeta_{1, j^*} \right| \middle| \bbeta_1, \bbeta_2, \y, \z\right] \nonumber  \\
&=: \mathbb{P}\left[ A \middle| \bbeta_1, \bbeta_2, \y, \z \right],
\label{eq:PA_cond}
\end{align}
where $i)$ holds by the closure of Gaussian random variables under finite sum, and $ii)$ holds by symmetry of the Gaussian  $g_1$ and since $\mathbb{P}\left[B_1 \cap B_2 \right] \leq \min_{i \in \{1, 2\}} \mathbb{P}\left[B_i \right]$. 

Using the high probability events $\Omega_1(\delta), \Omega_4(\delta)$  defined in Lemma \ref{lem:high-prob-events}, we  have
\begin{align}
&\mathbb{P}\left[\left|\frac{\langle \X_{j^*}, \y \rangle}{\|\y\|_2} \right| \leq \tau \right] =\int_{\bbeta_1, \bbeta_2} \int_{y, z} \mathbb{P}\left[\left| \frac{\langle X_{j^*, i}, \y\rangle}{\|\y\|_2} \right| \leq \tau \middle| \bbeta_1, \bbeta_2, \y, \z\right] d\mathbb{P}[\bbeta_1, \bbeta_2, \y, z] \nonumber \\
&\leq \int_{\bbeta_1, \bbeta_2}\int_{\Omega_1(\delta) \cap \Omega_4(\delta)} \mathbb{P}\left[\left| \frac{\langle X_{j^*, i}, \y\rangle}{\|\y\|_2} \right| \leq \tau \middle| \y, \z\right] d\mathbb{P}[\bbeta_1, \bbeta_2, \y, z]  \nonumber \\ 
& \quad + \int_{\bbeta_1, \bbeta_2} \mathbb{P}\left[\left(\Omega_1(\delta) \cap \Omega_4(\delta) \right)^{\complement} \middle| \bbeta_1, \bbeta_2 \right] d\mathbb{P}[\bbeta_1, \bbeta_2]\nonumber \\
&= \int_{\bbeta_1, \bbeta_2}\int_{\Omega_1(\delta) \cap \Omega_4(\delta)} \mathbb{P}\left[A \middle| \bbeta_1, \bbeta_2, \y, \z\right] d\mathbb{P}[\bbeta_1, \bbeta_2, \y, z] + \mathbb{P}\left[\left(\Omega_1(\delta) \cap \Omega_4(\delta) \right)^{\complement} \right] \nonumber\\
&\leq \int_{\bbeta_1, \bbeta_2}\int_{\Omega_1(\delta) \cap \Omega_4(\delta)} \mathbb{P}\left[A \middle| \bbeta_1, \bbeta_2, \y, \z\right] d\mathbb{P}[\bbeta_1, \bbeta_2, \y, z] + \mathbb{P}\left[\Omega^{\complement}_1(\delta)\right] + \mathbb{P}\left[\Omega_4^{\complement}(\delta) \right]. \label{eq:corr-mslr-last2-general-2}
\end{align}
Now setting $\tau = \sqrt{2(1 + \epsilon/2)\log{2p}}$ and $n \geq \frac{32 (1 + \epsilon)}{\phi^2 \bbeta^2_{\texttt{min}}} (\|\bbeta\|^2_2 + \sigma^2) \log{p}$ we obtain that for $(y, z) \in \Omega_1(\delta) \cap \Omega_4(\delta)$:
\begin{align}
&\mathbb{P}\left[ A \middle| \bbeta_1, \bbeta_2, \y, \z \right] \nonumber \\
&= \mathbb{P}\left[g_1 \left(\sum_{i=1}^n (z_i \sigma^{(1)}_i + (1 - z_i) \sigma^{(2)}_i)^2 \right)^{\frac{1}{2}} \leq \tau - \left|\sum_{i=1}^n \frac{y^2_i / \|\y\|_2}{\|\bbeta\|^2_2 + \sigma^2} \cdot z_i \bbeta_{1, j^*}\right| \middle| \bbeta_1, \bbeta_2, \y, \z\right] \nonumber\\
&\leq \mathbb{P}\left[g_1 \left(\sum_{i=1}^n (z_i \sigma^{(1)}_i + (1 - z_i) \sigma^{(2)}_i)^2 \right)^{\frac{1}{2}} \leq \sqrt{2(1 + \epsilon/2)\log{2p}} \right. \nonumber \\
&\left. \hspace{10em} - \left|\frac{\|\y\|_2}{\|\bbeta\|^2_2 + \sigma^2} \phi \bbeta_{1, j^*} - \sqrt{\frac{2(3 + \delta) \log{2p}}{ (1 - \delta) (1 + \sigma^2/\|\bbeta\|^2_2)}}  \right| \middle| \bbeta_1, \bbeta_2, \y, \z\right] \nonumber\\
&\leq \mathbb{P}\left[g_1 \left(\sum_{i=1}^n (z_i \sigma^{(1)}_i + (1 - z_i) \sigma^{(2)}_i)^2 \right)^{\frac{1}{2}} \leq \sqrt{2(1 + \epsilon/2)\log{2p}} \right. \nonumber\\
&\left. \hspace{10em} - \left|\sqrt{\frac{n (1-\delta)}{\|\bbeta\|^2_2 + \sigma^2}} \phi \bbeta_{1, j^*} - \sqrt{\frac{2(3 + \delta) \log{2p}}{ (1 - \delta) (1 + \sigma^2/\|\bbeta\|^2_2)}}  \right| \middle| \bbeta_1, \bbeta_2, \y, \z\right] \nonumber\\
&\leq \mathbb{P}\left[g_1 \left(\sum_{i=1}^n (z_i \sigma^{(1)}_i + (1 - z_i) \sigma^{(2)}_i)^2 \right)^{\frac{1}{2}} \leq \sqrt{2(1 + \epsilon/2)\log{2p}} \right. \nonumber\\
&\left. \hspace{10em} - \left|\sqrt{32(1+\epsilon)(1 - \delta)\log{2p}} - \sqrt{\frac{2(3 + \delta) \log{2p}}{ (1 - \delta) (1 + \sigma^2/\|\bbeta\|^2_2)}}  \right| \middle| \bbeta_1, \bbeta_2, \y, \z\right] \label{eq:rhs-subg2-general-2},
\end{align}
where for small enough $\delta > 0$, we have
\begin{align}
&  \sqrt{2(1 + \epsilon/2)} - \left|\sqrt{32 (1 + \epsilon) (1 - \delta)} - \sqrt{2(3+\delta)/(1-\delta)(1 + \sigma^2/\|\bbeta\|^2_2)}\right| \leq - \sqrt{2}. \label{ineq:delta2-2}
\end{align}
 Hence we obtain that the right hand side of  the inequality in (\ref{eq:rhs-subg2-general-2}) is negative, leading us to the following inequality for $(y, z) \in \Omega_1(\delta) \cap \Omega_4(\delta)$,
\begin{align}
\mathbb{P}\left[A \middle| \bbeta_1, \bbeta_2, \y, \z\right] &\leq \mathbb{P}\left[ g_1 \left(\sum_{i=1}^n (z_i \sigma^{(1)}_i + (1-z_i)\sigma^{(2)}_i)^2\right)^{1/2} \leq - \sqrt{2\log{2p}} \middle| \bbeta_1, \bbeta_2, \y, \z \right] \nonumber \\
&\leq \exp\left(- \frac{2 \log{2p}}{2 \sum_{i=1}^n (z_i \sigma^{(1)}_i + (1-z_i)\sigma^{(2)}_i)^2}\right) \nonumber \\
&\overset{i)}{\leq} \exp\left(- \frac{2 \log{2p}}{2 \sum_{i=1}^n \frac{y_i^2}{\|\y\|^2_2}}\right) \leq \frac{1}{2p}, \label{eq:PA_cond_bound}
\end{align}
where $i)$ follows almost surely from the fact that $z_i \in \{0, 1\}$ and $\left(1 - \frac{\bbeta^2_{1, j^*}}{\|\bbeta_1\|^2_2 + \sigma^2}\right) \leq 1$. We generalize to the case where $j \in \mathcal{S}_2 \setminus \mathcal{S}_1$, and hence consider the analogous result with $\phi$ and $\bbeta_1$ replaced with $(1 - \phi)$ and $\bbeta_2$. Putting it all together in \eqref{eq:corr-mslr-last2-general-2} using Lemma \ref{lem:high-prob-events} and choosing $\delta > 0$ small enough to satisfy \eqref{ineq:delta2-2}, we obtain that for $n \geq \frac{32}{\min\{\phi^2, (1-\phi)^2\} \bbeta^2_{\texttt{min}}} (1 + \epsilon) (\|\bbeta\|^2_2 + \sigma^2) \log{2p}$,
\begin{align}
\mathbb{P}\left[\left|\frac{\langle \X_{j^*}, \y \rangle}{\|\y\|_2} \right| \leq \sqrt{2 (1 + \epsilon/2) \log{2p}} \right] &\leq \frac{1}{p} + 4 e^{- \frac{\delta^2 n}{8}} + \frac{39}{\delta^2 n}.
\label{eq:Pcorr_dev_final}
\end{align}
\end{proof}

\begin{lemma} [General Bound for $j \in \mathcal{S}_1 \cap \mathcal{S}_2$] \label{lemma:A_bound_general}
Consider the setting of $\mslr$. Let $j^* \in \mathcal{S}_1 \cap \mathcal{S}_2$. Then if $n \geq \frac{32 (1 + \epsilon)}{(\phi \beta_{1, j^*} + (1-\phi) \beta_{2, j^*})^2}  (\|\bbeta\|^2_2 + \sigma^2) \log{2p}$ \, for any $\epsilon \in (0, 1)$ we obtain that
\[\mathbb{P}\left[\left|\frac{\langle \X_{j^*}, \y \rangle}{\|\y\|_2} \right| \leq \sqrt{2 (1 + \epsilon/2) \log{2p}} \right] \leq \frac{1}{p} + 4 e^{- \frac{\delta^2 n}{8}} + \frac{39}{\delta^2 n}. \]
\end{lemma}
\begin{proof} 
The proof is along the same lines as that of Lemma \ref{lemma:B_bound_general}, with the main difference being in the conditional  means and variances of $X_{j^*,i}$, for $i \in [n]$. As before, let $g_i \distas{\text{i.i.d.}} \mathcal{N}(0, 1)$ (independent from $y, z, \bbeta_1, \bbeta_2$) for $i \in [n]$ and $\delta > 0$ to be chosen later. 
Define
\begin{itemize}{}
	\item $\sigma^{(1)}_i := \frac{y_i}{\|\y\|_2} \sqrt{1 -\frac{\bbeta^2_{1,j^*}}{\|\bbeta\|^2_2+\sigma^2}}$,
	\item $\sigma^{(2)}_i := \frac{y_i}{\|\y\|_2} \sqrt{1 -\frac{\bbeta^2_{2,j^*}}{\|\bbeta\|^2_2+\sigma^2}}$,
	\item $\mu^{(1)}_i := \frac{y^2_i}{\|\y\|_2} \frac{\bbeta_{1, j^*}}{\|\bbeta\|^2_2 + \sigma^2}$,
	\item $\mu^{(2)}_i := \frac{y^2_i}{\|\y\|_2} \frac{\bbeta_{2, j^*}}{\|\bbeta\|^2_2 + \sigma^2}$,
\end{itemize}
and notice that for $\tau \in \mathbb{R}$,
\begin{align*}
&\mathbb{P}\left[X_{j^*, i} \leq \tau | \bbeta_1, \bbeta_2, \y, \z \right] \\
&= \mathbb{P}\left[z_i \left(\frac{y_i \cdot \bbeta_{1, j^*}}{\|\bbeta\|^2_2 + \sigma^2} + \sqrt{1 -\frac{\bbeta^2_{1,j^*}}{\|\bbeta\|^2_2+\sigma^2}} g_i \right) \right. \\
& \qquad \quad  \left.+ (1 - z_i) \left(\frac{y_i \cdot \bbeta_{2, j^*}}{\|\bbeta\|^2_2 + \sigma^2} + \sqrt{1 -\frac{\bbeta^2_{2,j^*}}{\|\bbeta\|^2_2+\sigma^2}} g_i \right) \leq \tau \middle| \bbeta_1, \bbeta_2, y_i, z_i \right],
\end{align*}
which follows from Lemma \ref{lemma:cond-mslr-general}. We then have that
\[ \mathbb{P}\left[ X_{j^*, i} \frac{y_i}{\|\y\|_2} \leq \tau \middle| \y, \z\right] = \mathbb{P}\left[ z_i (\mu^{(1)}_i + \sigma^{(1)}_i g_i) + (1 - z_i)(\mu^{(2)}_i + \sigma^{(2)}_i g_i) \leq \tau \middle| \y, \z\right]. \]
Then, using the same steps as in \eqref{eq:PA_cond}, we obtain that for  for $\tau > 0$:
\begin{align}
 &\mathbb{P}\left[\left| \frac{\langle \X_{j^*}, \y\rangle}{\|\y\|_2} \right| \leq \tau \middle| \bbeta_1, \bbeta_2, \y, \z\right] \nonumber \\
&= \mathbb{P}\left[ g_1 \left(\sum_{i=1}^n (z_i \sigma^{(1)}_i + (1 - z_i) \sigma^{(2)}_i)^2 \right)^{\frac{1}{2}} \leq \tau - \left|\sum_{i=1}^n \frac{y^2_i / \|\y\|_2}{\|\bbeta\|^2_2 + \sigma^2} \cdot (z_i \bbeta_{1, j^*} + (1 - z_i) \bbeta_{2, j^*})\right| \middle| \bbeta_1, \bbeta_2, \y, \z\right] \nonumber \\
&=: \mathbb{P}\left[ A \middle| \bbeta_1, \bbeta_2, \y, \z \right].
\label{eq:PA_cond_22}
\end{align}
%
Using the high probability events $\Omega_1(\delta), \Omega_3(\delta)$  defined in Lemma \ref{lem:high-prob-events}, by  the same arguments as in \eqref{eq:corr-mslr-last2-general-2} we  have
\begin{align}
&\mathbb{P}\left[\left|\frac{\langle \X_{j^*}, \y \rangle}{\|\y\|_2} \right| \leq \tau \right] \nonumber \\
&\leq \int_{\bbeta_1, \bbeta_2}\int_{\Omega_1(\delta) \cap \Omega_3(\delta)} \mathbb{P}\left[A \middle| \bbeta_1, \bbeta_2, \y, \z\right] d\mathbb{P}[\bbeta_1, \bbeta_2, \y, z] + \mathbb{P}\left[\Omega^{\complement}_1(\delta)\right] + \mathbb{P}\left[\Omega_3^{\complement}(\delta) \right], 
\label{eq:corr-mslr-last2-general}
\end{align}
Now setting $\tau = \sqrt{2(1 + \epsilon/2)\log{2p}}$ and $n \geq \frac{32}{\langle\bbeta\rangle^2_{\texttt{min}}} (1 + \epsilon) (\|\bbeta\|^2_2 + \sigma^2) \log{2p}$ we obtain that for $(y, z) \in \Omega_1(\delta) \cap \Omega_3(\delta)$,

\begin{align}
&\mathbb{P}\left[ A \middle| \bbeta_1, \bbeta_2, \y, \z \right] \nonumber \\
&= \mathbb{P}\left[g_1 \left(\sum_{i=1}^n (z_i \sigma^{(1)}_i + (1 - z_i) \sigma^{(2)}_i)^2 \right)^{\frac{1}{2}} \leq \tau - \left|\sum_{i=1}^n \frac{y^2_i / \|\y\|_2}{\|\bbeta\|^2_2 + \sigma^2} \cdot (z_i \bbeta_{1, j^*} + (1 - z_i) \bbeta_{2, j^*})\right| \middle| \bbeta_1, \bbeta_2, \y, \z\right] \nonumber\\
&\leq \mathbb{P}\left[g_1 \left(\sum_{i=1}^n (z_i \sigma^{(1)}_i + (1 - z_i) \sigma^{(2)}_i)^2 \right)^{\frac{1}{2}} \leq \sqrt{2(1 + \epsilon/2)\log{2p}} \right. \nonumber \\
&\left. \hspace{10em} - \left|\frac{\|\y\|_2}{\|\bbeta\|^2_2 + \sigma^2} (\phi \bbeta_{1, j^*} + (1 - \phi) \bbeta_{2, j^*}) - \sqrt{\frac{2(3 + \delta) \log{2p}}{ (1 - \delta) (1 + \sigma^2/\|\bbeta\|^2_2)}}  \right| \middle| \bbeta_1, \bbeta_2, \y, \z\right] \nonumber\\
&\leq \mathbb{P}\left[g_1 \left(\sum_{i=1}^n (z_i \sigma^{(1)}_i + (1 - z_i) \sigma^{(2)}_i)^2 \right)^{\frac{1}{2}} \leq \sqrt{2(1 + \epsilon/2)\log{2p}} \right. \nonumber\\
&\left. \hspace{10em} - \left|\sqrt{\frac{n (1-\delta)}{\|\bbeta\|^2_2 + \sigma^2}} (\phi \bbeta_{1, j^*} + (1 - \phi) \bbeta_{2, j^*}) - \sqrt{\frac{2(3 + \delta) \log{2p}}{ (1 - \delta) (1 + \sigma^2/\|\bbeta\|^2_2)}}  \right| \middle| \bbeta_1, \bbeta_2, \y, \z\right] \nonumber\\
&\leq \mathbb{P}\left[g_1 \left(\sum_{i=1}^n (z_i \sigma^{(1)}_i + (1 - z_i) \sigma^{(2)}_i)^2 \right)^{\frac{1}{2}} \leq \sqrt{2(1 + \epsilon/2)\log{2p}} \right. \nonumber\\
&\left. \hspace{10em} - \left|\sqrt{32(1+\epsilon)(1 - \delta)\log{2p}} - \sqrt{\frac{2(3 + \delta) \log{2p}}{ (1 - \delta) (1 + \sigma^2/\|\bbeta\|^2_2)}}  \right| \middle| \bbeta_1, \bbeta_2, \y, \z\right] \label{eq:rhs-subg2-general},
\end{align}
where
\begin{align}
 \sqrt{2(1 + \epsilon/2)} - \left|\sqrt{32 (1 + \epsilon) (1 - \delta)} - \sqrt{2(3+\delta)/(1-\delta)(1 + \sigma^2/\|\bbeta\|^2_2)}\right| \leq - \sqrt{2}, \label{ineq:delta2}
\end{align}
for small enough $\delta > 0$. Hence we obtain that the right hand side of (\ref{eq:rhs-subg2-general}) is negative, leading us to the following inequality for $(y, z) \in \Omega_1(\delta) \cap \Omega_3(\delta)$, obtained via the same steps as \eqref{eq:PA_cond_bound}:
\begin{align*}
\mathbb{P}\left[A \middle| \bbeta_1, \bbeta_2, \y, \z\right] 
\leq \frac{1}{2p}.
\end{align*}
%
Putting it all together in (\ref{eq:corr-mslr-last2-general}) using Lemma \ref{lem:high-prob-events} and choosing $\delta > 0$ small enough to satisfy (\ref{eq:rhs-subg2-general}) we obtain that for $n \geq \frac{32}{\langle\bbeta\rangle^2_{\texttt{min}}} (1 + \epsilon) (\|\bbeta\|^2_2 + \sigma^2) \log{2p}$,
\begin{align*}
\mathbb{P}\left[\left|\frac{\langle \X_{j^*}, \y \rangle}{\|\y\|_2} \right| \leq \sqrt{2 (1 + \epsilon/2) \log{2p}} \right] &\leq \frac{1}{p} + 4 e^{- \frac{\delta^2 n}{8}} + \frac{39}{\delta^2 n}.
\end{align*}
\end{proof}

\paragraph{The conditioning lemmas.}

\begin{lemma}[General conditioning lemma] \label{lemma:cond-mslr-general}
Consider the setting of $\mslr$.

For $j^* \in \mathcal{S}_1 \cap \mathcal{S}_2$ and $i \in [n]$ it holds that:

\begin{align}
\left. X_{j^*, i} \middle| (y_i, \bbeta_1, z_i = 1) \right. &\distas{} \mathcal{N}\left(\frac{y_i \cdot \bbeta_{1, j^*}}{\|\bbeta_1\|^2_2 + \sigma^2}, \left(1 - \frac{\bbeta^2_{1, j^*}}{\|\bbeta_1\|^2_2 + \sigma^2}\right) \right),  \label{mslr_cond_lemma1} \\
\left. X_{j^*, i} \middle| (y_i, \bbeta_2 , z_i = 0) \right. &\distas{} \mathcal{N}\left(\frac{y_i \cdot \bbeta_{2, j^*}}{\|\bbeta_1\|^2_2 + \sigma^2}, \left(1 - \frac{\bbeta^2_{2, j^*}}{\|\bbeta_1\|^2_2 + \sigma^2}\right) \right). \label{mslr_cond_lemma2}
\end{align}

For $j^* \in \mathcal{S}_1 \Delta \mathcal{S}_2$ (without loss of generality $j^* \in \mathcal{S}_1 \setminus \mathcal{S}_2$) and $i \in [n]$ it holds that
\begin{align}
\left. X_{j^*, i} \middle| (y_i, \bbeta_1, z_i = 1) \right. &\distas{} \mathcal{N}\left(\frac{y_i \cdot \bbeta_{1, j^*}}{\|\bbeta_1\|^2_2 + \sigma^2}, \left(1 - \frac{\bbeta^2_{1, j^*}}{\|\bbeta_1\|^2_2 + \sigma^2}\right) \right), \label{mslr_cond_lemma3}\\
\left. X_{j^*, i} \middle| (y_i, z_i = 0) \right. &\distas{} \mathcal{N}\left(0, 1\right) \label{mslr_cond_lemma4}
\end{align}
\end{lemma}

\begin{proof} 
To prove (\ref{mslr_cond_lemma1}) and (\ref{mslr_cond_lemma3}), recall that
\[ y_i = z_i (\X\bbeta_1)_i + (1 - z_i) (\X\bbeta_2)_i + w_i, \]
and hence given $z_i = 1$ we have
\[ y_i = \sum_{j\neq j^* \in \mathcal{S}_1} \bbeta_{1, j} X_{j,i} + \bbeta_{1, j^*} X_{j^*, i} + w_i, \]
implying that
\begin{align*}
\left. X_{j^*, i} \,  \middle| \,  (y_i, \bbeta_1, z_i = 1) \right. &=  \left. X_{j^*, i} \, \middle| \,  \left( \bbeta_1, y_i = \sum_{j\neq j^* \in \mathcal{S}_1} \bbeta_{1, j} X_{j,i} + \bbeta_{1, j^*} X_{j^*, i} + w_i\right) \right. .
\end{align*}
%
Therefore, applying Corollary \ref{cor:cond-general}, we obtain that $\left. X_{j^*, i} \middle| (y_i, \bbeta_1, z_i = 1) \right. \distas{} \mathcal{N}\left(\frac{y_i \bbeta_{1, j^*}}{\|\bbeta_1\|^2_2 + \sigma^2}, \left(1 - \frac{\bbeta^2_{1, j^*}}{\|\bbeta_1\|^2_2 + \sigma^2}\right) \right)$. The result in (\ref{mslr_cond_lemma2}) is proved in the same way, but replacing $\bbeta_1$ with $\bbeta_2$.

 For (\ref{mslr_cond_lemma4}), notice that given $z_i = 0$ we have
\[ y_i = \sum_{j \in \mathcal{S}_2} \bbeta_{2, j} X_{j,i} + w_i, \]
which is independent of $X_{j^*, i}$ by definition, and hence $X_{j^*, i} | (y_i, z_i = 0)$ has the same distribution as $X_{j^*, i} \distas{} \mathcal{N}(0, 1)$.
\end{proof}




\begin{lemma}[General conditioning lemma] \label{lemma:gsum-general}
Let $a \distas{} \mathcal{N}\left(0, \Sigma_{k \times k} \right)$, and $b \in \mathbb{R}^k$ a fixed vector. Then:
\[\left. a \middle| \right. \left( \sum_{j=1}^k b_j a_j = \eta \right) \distas{} \mathcal{N}\left(\eta v, B \Sigma_{k \times k} B^T \right) \]
where, letting $\mathbbm{1}$ denote the all-ones vector, 
\[v = \frac{1}{b^T \tilde{\Sigma} \mathbbm{1}} \tilde{\Sigma} \mathbbm{1}, \quad B = I_{k \times k} - v b^T, \quad \tilde{\Sigma} = \mathbb{E}\left[a (a \odot b)^T \right]. \]
\end{lemma}

\begin{proof}
Let $B$ be a deterministic matrix, and $\eta := \sum_{j=1}^k b_j a_j$. Then $(Ba, \eta)$ is jointly normal. We will construct a fixed matrix $B$ and fixed vector $v$ such that
\begin{itemize}
	\item $Ba$ is independent from $\eta$
	\item $a = Ba + \eta v$.
\end{itemize}
If the above holds, then by independence we have the required result that $a|\eta \distas{} \mathcal{N}\left(\eta v, B \Sigma B^T \right)$.
In order for the first point to hold, their covariances must be zero, implying
\[\mathbb{E}\left[Ba \eta \right] = \mathbb{E}\left[B a (a \odot b) ^T \mathbbm{1} \right] = B \tilde{\Sigma} \mathbbm{1} = 0.\]
Meanwhile, the second point is satisfied by choosing $B = I - v b^T$. Combining these two facts, we obtain the result.
\end{proof}

\begin{corollary} \label{cor:slr-conditioning}
Consider the setting of $\slr$, let $j^* \in \mathcal{S}$, and $i \in [n]$. Then it holds that

\[\left. X_{j^*, i} \middle| \left(\sum_{j \in \mathcal{S}} X_{j, i} + w_i = \eta\right) \right. \distas{} \mathcal{N}\left(\frac{\eta}{k + \sigma^2}, \left(1 - \frac{1}{k + \sigma^2}\right) \right)  \]
\end{corollary}

\begin{proof}
Apply Lemma \ref{lemma:gsum-general} conditioning on the sum $\sum_{j \in \mathcal{S}} X_{j ,i} + w_i =: \eta$, where we recall $\begin{bmatrix} \X_{\mathcal{S}} \; w_i \end{bmatrix} \distas{} \mathcal{N}\left(0, \Sigma \right)$ with
\[ \Sigma =
\begin{bmatrix}
1 & 0 & \cdots & 0 \\
0 & 1 & \cdots & 0 \\
\vdots & \vdots & \ddots & \vdots \\
0 & 0 & \cdots & \sigma^2
\end{bmatrix} \in \mathbb{R}^{(k+1) \times (k+1)}
\]
yielding
\[v= \frac{1}{k + \sigma^2}
\begin{bmatrix}
1 \\
1 \\
\vdots \\
\sigma^2
\end{bmatrix}
\]
and
\[B =
\begin{bmatrix}
(1 - \frac{1}{k + \sigma^2}) & -\frac{1}{k + \sigma^2} & \cdots & -\frac{1}{k + \sigma^2} \\
-\frac{1}{k + \sigma^2} & (1 - \frac{1}{k + \sigma^2}) &  \cdots & -\frac{1}{k + \sigma^2} \\
\vdots & \vdots & \ddots & \vdots \\
-\frac{1}{k + \sigma^2} & -\frac{1}{k + \sigma^2} & \cdots & (1 - \frac{1}{k + \sigma^2}) \\
\end{bmatrix}.
\]
Noticing that
$ (B \Sigma B^T)_{j^*, j^*} = 1 - \frac{1}{k + \sigma^2}$,
we obtain the result.
\end{proof}

\begin{corollary} \label{cor:cond-general}
Consider the setting of $\slr$, and let $b \in \mathbb{R}^p$ be fixed. Let $j^* \in \mathcal{S}$, and $i \in [n]$. Then it holds that

\[\left. X_{j^*, i} \middle| \left(\sum_{j \in \mathcal{S}} b_j X_{j, i} + w_i = \eta\right) \right. \distas{} \mathcal{N}\left(\frac{\eta \cdot b_{j^*}}{\|b\|^2_2 + \sigma^2}, \left(1 - \frac{b^2_{j^*}}{\|b\|^2_2 + \sigma^2}\right) \right)  \]
\end{corollary}
\begin{proof}
The result is obtained by applying  Lemma \ref{lemma:gsum-general} conditioning on the sum $\sum_{j \in \mathcal{S}} b_j X_{j ,i} + w_i = \eta$. We then have that
\[ \tilde{\Sigma} =
\mathbb{E}\left[{\begin{bmatrix} X_{1, i} \\ X_{2, i} \\ \vdots \\ w_i \end{bmatrix}} {\begin{bmatrix} b_1 X_{1, i} \\ b_2 X_{2, i} \\ \vdots \\ w_i \end{bmatrix}}^T \right]
=
\begin{bmatrix}
b_1 & 0 & \cdots & 0 \\
0 & b_2 & \cdots & 0 \\
\vdots & \vdots & \ddots & \vdots \\
0 & 0 & \cdots & \sigma^2
\end{bmatrix} \in \mathbb{R}^{(k+1) \times (k+1)}
\]
yielding
\[v= \frac{1}{\|b\|^2_2 + \sigma^2}
\begin{bmatrix}
b_1 \\
b_2 \\
\vdots \\
\sigma^2
\end{bmatrix}
\]
and
\[B =
\begin{bmatrix}
(1 - \frac{b^2_1}{\|b\|^2_2 + \sigma^2}) & -\frac{b_1 b_2}{\|b\|^2_2 + \sigma^2} & \cdots & -\frac{b_1}{\|b\|^2_2 + \sigma^2} \\
-\frac{b_2 b_1}{\|b\|^2_2 + \sigma^2} & (1 - \frac{b^2_2}{\|b\|^2_2 + \sigma^2}) &  \cdots & -\frac{b_2}{\|b\|^2_2 + \sigma^2} \\
\vdots & \vdots & \ddots & \vdots \\
-\frac{\sigma^2 b_1}{\|b\|^2_2 + \sigma^2} & -\frac{\sigma^2 b_2}{\|b\|^2_2 + \sigma^2} & \cdots & (1 - \frac{\sigma^2}{\|b\|^2_2 + \sigma^2}) \\
\end{bmatrix}.
\]
Noticing that
$ (B \Sigma B^T)_{j^*, j^*} = 1 - \frac{b^2_{j^*}}{\|b\|^2_2 + \sigma^2}$,
we obtain the result.
\end{proof}
\begin{lemma} \label{lem:y-z-indep}
Consider the setting of $\mslr$ as in Definition \ref{def:MSLR}. Then $y$ and $z$ are independent.
\end{lemma}
\begin{proof}
Let $g \distas{} \mathcal{N}(0, 1)$, and $A$ an event in the sigma algebra. By Bayes rule, for every $i \in [n]$ it holds that
\begin{align*}
\mathbb{P}\left[z_i = 1 \middle| y_i \in A\right] &= \frac{\mathbb{P}\left[y_i \in A \middle| z_i = 1\right] \mathbb{P}[z_i = 1]}{\mathbb{P}[y]} \\
&= \frac{\mathbb{P}\left[y_i \in A \middle| z_i = 1 \right] \mathbb{P}[z_i = 1 ]}{\mathbb{P}\left[y_i \in A \middle| z_i = 1 \right] \mathbb{P}[z_i = 1 ] + \mathbb{P}\left[y_i \in A \middle| z_i = 0\right] \mathbb{P}[z_i = 0]} \\
&= \frac{\mathbb{P}[\sqrt{\|\bbeta\|^2_2 + \sigma^2} g \in A] \phi}{\mathbb{P}[\sqrt{\|\bbeta\|^2_2 + \sigma^2} g \in A] \phi + \mathbb{P}[\sqrt{\|\bbeta\|^2_2 + \sigma^2} g \in A] (1 - \phi)} \\
&= \phi \\
&= \mathbb{P}[z_i = 1],
\end{align*}
and the analogous result holds for the case $z_i = 0$. The result then follows by recalling that $y_i$ and $z_i$ are i.i.d. across the $i$ indices.
\end{proof}

\paragraph{High-probability events and Concentration inequalities}
\begin{lemma} [High-probability events] \label{lem:high-prob-events}
Consider the setting of $\mslr$. Let $g_q \distas{\text{i.i.d.}} \mathcal{N}(0, 1)$ for $q \in [p]$ and $\delta \in (0, 1)$, $t > 0, j^* \in [p]$. Define the following events that will be necessary for the analysis of $\mathtt{CORR}$ on $\mslr$:
\[\Omega_1(\delta) := \left\{ (\|\bbeta\|^2_2 + \sigma^2) n (1 - \delta) \leq \|\y\|^2 \leq (\|\bbeta\|^2_2 + \sigma^2) n (1 + \delta)\right\}, \]
\[\Omega_2(t) := \left\{ \max_{q \in [p]} |g_q| \leq \sqrt{2 \log{2p}} + \sqrt{2 t \log{2p}} \right\}, \]
\begin{align*}
&\Omega_3(\delta) :=  \left\{ \left|\sum_{i=1}^n (z_i \bbeta_{1, j^*} + (1 - z_i) \bbeta_{2, j^*}) \frac{y^2_i/\|\y\|_2}{\|\bbeta\|^2_2 + \sigma^2} \right. \right. \\
& \qquad \qquad  \left. \left. - (\phi \bbeta_{1, j^*} + (1 - \phi) \bbeta_{2, j^*}) \frac{\|\y\|_2}{\|\bbeta\|^2_2 + \sigma^2} \right| < \sqrt{\frac{2(3 + \delta) \log{(2p)}}{(1 - \delta) (1 + \sigma^2/\|\bbeta\|^2_2)}} \right\},
\end{align*}
\[\Omega_4(\delta) := \left\{ \left|\sum_{i=1}^n z_i \bbeta_{1, j^*} \frac{y^2_i/\|\y\|_2}{\|\bbeta\|^2_2 + \sigma^2} - \phi \bbeta_{1, j^*} \frac{\|\y\|_2}{\|\bbeta\|^2_2 + \sigma^2}\right| < \sqrt{\frac{2(3 + \delta) \log{(2p)}}{(1 - \delta) (1 + \sigma^2/\|\bbeta\|^2_2)}} \right\}, \]
\[\Omega_5(\delta) := \left\{\left| \|\y\|^4_4 - 3n(\|\bbeta\|^2_2 + \sigma^2)^2 \right| < \delta n (\|\bbeta\|^2_2 + \sigma^2)^2 \right\}. \]
Then, the above events all occur with high probability. Specifically,
\begin{itemize}
	\item $\mathbb{P}\left[\Omega^{\complement}_1(\delta) \right] \leq 2 e^{-\frac{\delta^2 n}{8}}$ as per Example 2.11 in \cite{wainwright_high-dimensional_2019} noting that $y_i \distas{\text{i.i.d.}} \mathcal{N}(0, \|\bbeta\|^2_2 + \sigma^2)$ both marginally and conditionally on z (similarly for the setting of $\slr$).
	\item $\mathbb{P}\left[\Omega^{\complement}_2(t) \right] \leq \frac{1}{(2p)^t}$ as per Lemma 5.2 in \cite{van_handel_probability_2014}.
	\item $\mathbb{P}\left[\Omega^{\complement}_3(\delta) \middle| \Omega_1(\delta), \Omega_5(\delta) \right] \leq \frac{1}{2p}$ as per Lemma \ref{lemma:lln-bernoulli-general} and consequently $\mathbb{P}\left[\Omega^{\complement}_3(\delta) \right] \leq \frac{1}{2p} + 2e^{-\frac{\delta^2 n}{8}} + \frac{39}{\delta^2 n}$.
	\item $\mathbb{P}\left[\Omega^{\complement}_4(\delta) \middle| \Omega_1(\delta), \Omega_5(\delta) \right] \leq \frac{1}{2p}$ as per Lemma \ref{lemma:lln-bernoulli-general} with $\bbeta_{2, j^*} = 0$ and consequently $\mathbb{P}\left[\Omega^{\complement}_4(\delta) \right] \leq \frac{1}{2p} + 2e^{-\frac{\delta^2 n}{8}} + \frac{39}{\delta^2 n}$.
	\item $\mathbb{P}\left[\Omega^{\complement}_5(\delta) \right] \leq \frac{39}{\delta^2 n}$ as per Lemma \ref{lemma:l4-conc}.
\end{itemize}
\end{lemma}
\begin{lemma}[Fast $\ell_4$ norm concentration] \label{lemma:l4-conc}
Let $\y \in \reals^n$ with $y_i \distas{\text{i.i.d.}} \mathcal{N}(0, \sigma^2)$ for $i \in [n]$. Then for $\delta>0$ we have,
\[\mathbb{P}\left[\left| \|\y\|_4 - 3n(\sigma^2)^2 \right| \geq \delta n (\sigma^2)^2 \right] \leq \frac{39}{\delta^2 n}.\]
\end{lemma}

\begin{proof}
The proof follows from a standard application of Chebyshev's inequality. We first recall that the $p^{\text{th}}$ centered Gaussian moment is given by $(\text{variance})^{\frac{p}{2}} (p - 1)!!$, and consequently by independence we have that $\mathbb{E}\left[\|\y\|^4_4 \right] = 3 n (\sigma^2)^2$. We then proceed with Chebyshev's inequality (see \cite{boucheron_concentration_2013}) for $t > 0$:
\begin{align*}
\mathbb{P}\left[\left| \|\y\|^4_4 - \mathbb{E}\left[\|\y\|^4_4 \right] \right| \geq t \right] &\leq \frac{n \text{Var}(y^4_1)}{t^2}\\
&= \frac{n \left(\mathbb{E}y^8_1 - \left(\mathbb{E} y^4_1 \right)^2 \right)}{t^2} \\
&= \frac{n (\sigma^2)^4 ((7-1)!! - 9)}{t^2} = \frac{39 n (\sigma^2)^4}{t^2}.
\end{align*}
We set $t = n (\sigma^2)^2$ to obtain
\begin{align*}
\mathbb{P}\left[\left| \|\y\|^4_4 - \mathbb{E}\left[\|\y\|^4_4 \right] \right| \geq n (\sigma^2)^2 \right] &\leq \frac{39 n (k + \sigma^2)^4}{n^2 (\sigma^2)^4} = \frac{39}{n}.
\end{align*}
\end{proof}

\begin{lemma} \label{lemma:lln-bernoulli-general}
Consider the setting of $\mslr$. For $t > 0, \delta > 0, j^* \in [p]$ we have that

\begin{align*}
&\mathbb{P}\left[ \left|\sum_{i=1}^n (z_i \bbeta_{1, j^*} + (1 - z_i) \bbeta_{2, j^*}) \frac{y^2_i/\|\y\|_2}{\|\bbeta\|^2_2 + \sigma^2} - (\phi \bbeta_{1, j^*} + (1 - \phi) \bbeta_{2, j^*}) \frac{\|\y\|^2_2}{\|\bbeta\|^2_2 + \sigma^2} \right| \geq t \middle| \Omega_1(\delta), \Omega_5(\delta) \right] \\
&\leq \exp\left(-\frac{(1 - \delta)}{2(3 + \delta)} (1 + \sigma^2/\|\bbeta\|^2_2) t^2 \right)
\end{align*}
\end{lemma}

\begin{proof}
Consider $y \in \Omega_1(\delta) \cap \Omega_5(\delta)$ in Lemma \ref{lem:high-prob-events}. We first apply Hoeffding's inequality (see \cite{boucheron_concentration_2013}), then the definitions of $\Omega_1(\delta)$ and $\Omega_5(\delta)$ to obtain,

\begin{align*}
&\mathbb{P}\left[ \left|\sum_{i=1}^n (z_i \bbeta_{1, j^*} + (1 - z_i) \bbeta_{2, j^*}) \frac{y^2_i/\|\y\|_2}{\|\bbeta\|^2_2 + \sigma^2} - (\phi \bbeta_{1, j^*} + (1 - \phi) \bbeta_{2, j^*}) \frac{\|\y\|^2_2}{\|\bbeta\|^2_2 + \sigma^2} \right| \geq t \middle| \y \right] \\
&\leq \exp\left(- \frac{2t^2}{\sum_{i=1}^n (\bbeta_{1, j^*} - \bbeta_{2, j^*})^2 \frac{y^4_i/\|\y\|^2_2}{(\|\bbeta\|^2_2 + \sigma^2)^2}} \right) \\
&= \exp\left(-\frac{2t^2}{(\bbeta_{1, j^*} - \bbeta_{2, j^*})^2 \frac{\|\y\|^4_4}{\|\y\|^2_2} \frac{1}{(\|\bbeta\|^2_2 + \sigma^2)^2}} \right) \\
&\leq \exp\left(-\frac{2t^2}{(\bbeta_{1, j^*} - \bbeta_{2, j^*})^2 \frac{3n(\|\bbeta\|^2_2 + \sigma^2)^2 + \delta n (\|\bbeta\|^2_2 + \sigma^2)^2}{\|\y\|^2_2} \frac{1}{(\|\bbeta\|^2_2 + \sigma^2)^2}} \right) \\
&\leq \exp\left(-\frac{2t^2}{(\bbeta_{1, j^*} - \bbeta_{2, j^*})^2 \frac{3n(\|\bbeta\|^2_2 + \sigma^2)^2 + \delta n (\|\bbeta\|^2_2 + \sigma^2)^2}{n(\|\bbeta\|^2_2 + \sigma^2)(1 - \delta)} \frac{1}{(\|\bbeta\|^2_2 + \sigma^2)^2}} \right) \\
&= \exp\left(-\frac{2 (\|\bbeta\|^2_2 + \sigma^2) t^2}{(\bbeta_{1, j^*} - \bbeta_{2, j^*})^2 \frac{3 + \delta}{1 - \delta}} \right) \\
&\overset{i)}{\leq} \exp\left(-\frac{2 (\|\bbeta\|^2_2 + \sigma^2) t^2}{4 \|\bbeta\|^2_2 \frac{3 + \delta}{1 - \delta}} \right) \\
&= \exp\left(-\frac{(1 - \delta)}{2(3 + \delta)} (1 + \sigma^2/\|\bbeta\|^2_2) t^2 \right),
\end{align*}
where $i)$ follows from $(\bbeta_{1, j^*} - \bbeta_{2, j^*})^2 \leq \| \bbeta_1 - \bbeta_2 \|^2_2$ and the triangle inequality. Applying the law of total probability to the above, we obtain the result.
\end{proof}
\subsection{$\corr$  for $\mslrd$} \label{subsec:corr-for-mslrd}
\begin{proof} [Proof of Theorem \ref{thm:CORR-mslrd-general}]
We first note that, since $(\bbeta_1, \bbeta_2) \distas{} \mathcal{P}_{\|\bbeta\|_2}(\mathcal{D})$, the two signals have equal norm. Hence, the complement of the $\sbmslr$ regime can be equivalently expressed through the condition $\phi \bbeta_1 + (1 - \phi) \bbeta_2 \neq 0$. 

Let $\corr(\X, \y)$ denote the output of running $\corr$  on inputs $\X, \y$. Consider the test function

\[
g\left({\begin{bmatrix} \X \\ \y \end{bmatrix}}\right) := \begin{cases}
\mathtt{p} & \corr(\X, \y) \neq \emptyset \\
\mathtt{q} & \corr(\X, \y) = \emptyset
\end{cases}.
\]

Let ${\begin{bmatrix} \X \\ \y \end{bmatrix}} \distas{} {\P(\X) \otimes \P(\y)}$. Recall that $\corr$  outputs the following set

\[
\corr(\X, \y) = \left\{j \in [p]:  \left|\frac{\langle \X_{j}, \y \rangle}{\|\y \|_2}\right| \geq \sqrt{2 (1 + \epsilon/2) \log{2p}} \right\}.
\]

As in the proof of Theorem \ref{thm:CORR-mslr-r-general}, we note that
\begin{align}
\mathbb{P}\left[\max_{q \in [p]} \left|\frac{\langle \X_{q}, \y \rangle}{\|\y \|_2}\right| \geq \sqrt{2 (1 + \epsilon/2) \log{2p}} \right] &\leq \mathbb{P}\left[\max_{q \in [p]} \left|\frac{\langle \X_{q}, \y \rangle}{\|\y \|_2}\right| \geq \sqrt{2 \log{2p}} + \frac{\epsilon}{2 \sqrt{8}} \sqrt{\log{2p}}\right]. \label{eq:last2}
\end{align}
Noting that for ${\begin{bmatrix} \X \\ \y \end{bmatrix}} \distas{} {\P(\X) \otimes \P(\y)}$ we have that $\frac{\langle \X_{q}, \y \rangle}{\|\y \|_2} \distas{\text{i.i.d.}} \mathcal{N}(0, 1)$, we apply Lemma \ref{lem:high-prob-events} to (\ref{eq:last2}) and obtain that
\begin{align*}
\mathbb{P}\left[\max_{q \in [p]} \left|\frac{\langle \X_{j}, \y \rangle}{\|\y \|_2}\right| \geq \sqrt{2 (1 + \epsilon/2) \log{2p}} \right] \leq (2p)^{- \frac{1}{8} \left(\frac{\epsilon}{2} \right)^2},
\end{align*}
and hence we have that, under ${\P(\X) \otimes \P(\y)}$, $g\left({\begin{bmatrix} \X \\ \y \end{bmatrix}}\right) = \mathtt{q}$ with probability $1- o(1)$.

Conversely, let ${\begin{bmatrix} \X \\ \y \end{bmatrix}} \distas{} {\P(\X, \y)}$. Let $J \subseteq \supp(\bbeta_1) \cup \supp(\bbeta_2)$ such that $\phi \bbeta_{1, j} + (1-\phi) \bbeta_{2, j} \neq 0$ for $j \in J$. From Lemma \ref{lemma:A_bound_general}, we then have that, 
\begin{align*}
\mathbb{P}\left[| \corr(\X, \y) | = 0 \right] &= \mathbb{P}\left[\cap_{j \in [p]} \left\{ \left| \frac{\langle \X_j, \y \rangle}{\|\y\|_2} \right| \leq \sqrt{2(1+\epsilon/2)\log{2p}} \right\} \right] \\
&\leq \mathbb{P}\left[\cap_{j \in J} \left\{ \left| \frac{\langle \X_j, \y \rangle}{\|\y\|_2} \right| \leq \sqrt{2(1+\epsilon/2)\log{2p}} \right\} \right] \\
&\leq \mathbb{P}\left[\left| \frac{\langle \X_{J_1}, \y \rangle}{\|\y\|_2} \right| \leq \sqrt{2(1+\epsilon/2)\log{2p}} \right] \\
&\leq \left(\frac{1}{p} + 6 e^{\frac{\delta^2 n}{w}} + \frac{39}{\delta^2 n} \right) = o(1),
\end{align*}
and hence we have that, under ${\P(\X, \y)}$, $g\left({\begin{bmatrix} \X \\ \y \end{bmatrix}}\right) = \mathtt{p}$ with probability $1- o(1)$. 
\end{proof}

\subsection{Recovery algorithms for $\mslr$} \label{sec:recovery-algorithms-mslr}

\paragraph{General recovery algorithm for $\mslr$}
A recovery algorithm for $\mslr$ in the noiseless and balanced regimes is given in Theorem \ref{thm:CORR-recovery}. We measure the recovery error in $\mslr$ as in \cite{chen_convex_2014}, 
\[
\rho(\hat{\bm{\theta}},\bm{\theta}):=\min\left\{ \left\Vert \hat{\bm{\bbeta}}_{1}-\bm{\bbeta}_{1}\right\Vert _{2}+\left\Vert \hat{\bm{\bbeta}}_{2}-\bm{\bbeta}_{2}\right\Vert _{2},\left\Vert \hat{\bm{\bbeta}}_{1}-\bm{\bbeta}_{2}\right\Vert _{2}+\left\Vert \hat{\bm{\bbeta}}_{2}-\bm{\bbeta}_{1}\right\Vert _{2}\right\},
\]
where $\hat{\bm{\theta}} = (\hat{\bbeta}_1, \hat{\bbeta}_2)$ and $\bm{\theta} = (\bbeta_1, \bbeta_2)$. 

\begin{proof} [Proof of Theorem \ref{thm:CORR-recovery}]
In the case of $\sigma = 0, \phi \neq 1/2$ (noiseless), we first apply $\corr$  and then the Alternating Minimization (AM) algorithm of \cite{yi_alternating_2014}. Theorem \ref{thm:CORR+AM} shows that this succeeds in the regime of interest.

In the case of $\phi = 1/2$ (balanced), $\snr = \Omega(k)$, we first apply $\corr$  and then the algorithm of \cite{chen_convex_2014}. Theorem \ref{thm:corr+algo3} proves that in the high \snr regime $\snr = \Omega(1)$, we have that $\rho(\hat{\theta}, \theta) = \Theta\left(\sigma \sqrt{\frac{k}{n}}\right)$. In order for exact recovery to be achieved for all allowable finite $n$, we would require $\sqrt{\frac{\sigma^2 k}{(\|\bbeta\|^2_2 + \sigma^2) \log{p}}} \to 0$, which is satisfied for $\snr = \|\bbeta\|^2_2 / \sigma^2 = \Omega(k)$, implying a non-vanishing signal-to-noise ratio with respect to the support set. This is satisfied by hypothesis, and hence $\corr$  together with Algorithm \ref{algo:balanced} succeeds in solving asymptotically exact recovery (up to relabeling of $\bbeta_1$ and $\bbeta_2$) in the regime of interest.
\end{proof}

\paragraph{Recovery algorithm for noiseless, unbalanced $\mslr$}
In this section we will show that the $\corr$  algorithm can be used to reduce the $\mslr$ problem to a dense problem where $n, k = \Theta(p)$ and the signal is not assumed to be sublinearly sparse, where state-of-the-art algorithms for this dense mixed linear regression case can then infer $\bbeta_1$ from $\bbeta_2$. In the noiseless case with $\phi \neq 1/2$, we can apply $\mathtt{CORR}$ together with the existing polynomial-time Alternating Minimization (AM) algorithm from \cite{yi_alternating_2014} to fully solve $\mslr$ in what is a constant number of steps. We recall that Theorem \ref{thm:CORR-mslr-r-general} only provided guarantees on the support recovery of the joint signal, whereas now with the execution of the AM algorithm on the reduced joint support set one can fully infer $\bbeta_1$ from $\bbeta_2$. Define the mixture proportions, 
\begin{align*}
&\frac{n_1}{n} := \frac{\sum_{i=1}^n z_i}{n} \\
&\frac{n_2}{n} := \frac{\sum_{i=1}^n (1 - z_i)}{n}. 
\end{align*}
We begin by stating the theorem. 
\begin{theorem} [Success of $\mathtt{CORR} +$ AM on noiseless $\mslr$] \label{thm:CORR+AM}
Consider the general setting of $\mslr$ with parameters $p, n, k, \sigma = 0$, $\phi \neq 1/2$, $(\bbeta_1, \bbeta_2) \distas{} \mathcal{P}_{\|\bbeta\|_2}(\mathcal{D})$. Suppose 
$$n \geq \frac{32}{\min\{\phi^2 \bbeta^2_{\texttt{min}}, (1-\phi)^2 \bbeta^2_{\texttt{min}}, \langle \bbeta \rangle^2_{\texttt{min}} \}} (1 + \epsilon) \|\bbeta\|^2_2 \log{2p}$$ for $\epsilon \in (0,1) $ used in $\corr$. Then with probability at least $1 - c_1(\frac{k}{p} + ke^{-c_2 n} + \frac{k}{n} + \frac{1}{p^{c_2}})$ for constants $c_1, c_2 > 0$, the output $\hat{\theta} = (\hat{\bbeta}_1, \hat{\bbeta}_2)$ of $\mathtt{CORR} +$ Algorithm \ref{alg:init_withp} $+$ Algorithm \ref{alg:rspEM} satisfies
\[\rho(\hat{\theta}, \theta) = 0.\]
\end{theorem}
\begin{proof}[Proof of Theorem \ref{thm:CORR+AM}]
By Theorem \ref{thm:CORR-mslr-r-general}, we can recover the joint support set (at most of size $2k$) of signals $(\bbeta_1, \bbeta_2)$ with probability at least $1 - c_1(\frac{k}{p} + ke^{-c_2 n} + \frac{k}{n} + \frac{1}{p^{c_2}})$ by running $\mathtt{CORR}$.

After running $\mathtt{CORR}$ and identifying the $<2k$ joint support set indices, we restrict the regression problem to these indices by removing all other columns from the design matrix $\X$ (as these do not influence the output $y$ since they do not correspond to support indices of $\bbeta_1$ or $\bbeta_2$). We are then tasked with solving a two-component mixtures of regressions problem with $n$ samples and signals of dimension between $k$ and $2k$ (importantly, the dimension is no longer $p$).

From this point, the idea is to spectrally initialize $(\bbeta_1^{(0)}, \bbeta^{(0)}_2)$ using Algorithm \ref{alg:init_withp} for which Proposition \ref{prop:initwp} provides guarantees, pass $(\bbeta_1^{(0)}, \bbeta^{(0)}_2)$ into Algorithm \ref{alg:rspEM} for which Theorem \ref{thm:em} provides guarantees on geometric error decay given this initialization, and run Algorithm \ref{alg:rspEM} for a finite number of iterations guaranteed by Proposition \ref{prop:exactReco}.

The condition on sample size $n$ of Proposition \ref{prop:initwp} is met, since it is assumed that $n \gtrsim \|\bbeta\|^2_2 \log{2p} \geq \bmin^2 k \log{2p}$ and $k \log{2p} = \omega(k \log^2{k})$. The condition of Theorem \ref{thm:em} is met by the result of Proposition \ref{prop:initwp}. The condition of Proposition \ref{prop:exactReco} is met by running Algorithm \ref{alg:rspEM} (resampling) with $O(k \log^2{k})$ samples (see Remark \ref{remark:geom_decay_samples}). 

What remains to show is that $n_1 \neq n_2$, as this is required by Remark \ref{remark:fan}. Recall $n_1 = \sum_{i=1}^n z_i$ and $n_2 = \sum_{i=1}^n (1-z_i)$ and $z_i$ are independent $\text{Bernoulli}(\phi)$. Without loss of generality assuming $\phi < 1/2$, there exists a $\delta \in (0, 1)$ such that,
\begin{align*}
\P\left[\frac{1}{n} \sum_{i=1}^n z_i \geq 1/2 \right] \leq \P\left[ \frac{1}{n} \sum_{i=1}^n z_i \geq (1 + \delta)\phi \right] \leq e^{-\Theta{(n)}},
\end{align*}
after applying a standard Chernoff bound. This high probability statement can be absorbed into the $1 - c_1(\frac{k}{p} + ke^{-c_2 n} + \frac{k}{n} + \frac{1}{p^{c_2}})$ high probability statement provided by $\corr$ , choosing adjusted constants $c_1, c_2 > 0$. 

We hence conclude that running $\corr$  followed by Algorithm \ref{alg:init_withp} followed by Algorithm \ref{alg:rspEM} we obtain $\rho(\mathbf{\theta}^{(t)}, \mathbf{\theta}) = 0$ in finite $t$ with probability at least $1 - c_1(\frac{k}{p} + ke^{-c_2 n} + \frac{k}{n} + \frac{1}{p^{c_2}})$.
\end{proof}
Their initialization algorithm is based on the positive semidefinite matrix: 
\begin{align*}
M := \frac{1}{n} \sum_{i=1}^n y_i^2 \x_i \otimes \x_i \label{eq:M}
\end{align*}
which serves as an unbiased estimator of a matrix whose two largest eigenvectors span the space spanned by $\bbeta_1, \bbeta_2$.

\begin{remark} \label{remark:fan}
It is stated in \cite{fan_curse_2018} that, when the mixture frequencies are equal to each other ($n_1 = n_2$), the top two eigenvectors of $\E M$ will not be $\bbeta_1, \bbeta_2$. Hence, their algorithms only work for the case $\phi \neq 1/2$ and $\sigma = 0$ (noiseless).
\end{remark} 

Outside of the case $\phi = 1/2$, when the mixture proportions are known, an approximation of $\bbeta_1, \bbeta_2$ can be computed in closed form through Algorithm \ref{alg:init_withp}, where 
\begin{equation*}
sign(b) = \begin{cases}
1, \; b = 1\\
-1, \; b = 2.
\end{cases}
\end{equation*} 
\begin{algorithm}[ht]
\caption{Initialization with proportion information} \label{alg:init_withp}
\; \KwData{Input: $n_1, n_2$, samples $\{(y_i, \x_i), i =1,2,...,n\}$}
\; $M \gets \frac{1}{N}\sum_{i=1}^{N}y_i^2\x_i\otimes\x_i$ \; \\
Compute top 2 eigenvectors and eigenvalues $(v_b, \lambda_b), b = 1,2$ of $(M-I)/2$ \; \\
Compute $\bbeta_b^{(0)} = \sqrt{\frac{1 - \Delta_b}{2}}v_b + sign(b)\sqrt{\frac{1+\Delta_{b}}{2}}v_{-b}$, where $\Delta_b = \frac{(\lambda_{b} - \lambda_{-b})^2 + n_b^2 - n_{-b}^2}{2(\lambda_{-b} - \lambda_b)n_b}, \; b = 1,2$ \\
\Return $\bbeta_1^{(0)}, \bbeta_2^{(0)}$
\end{algorithm}
In what follows, we state the iterative algorithms proposed in \cite{fan_curse_2018} and their guarantees. 
\begin{algorithm}[ht!]
\caption{AM}\label{alg:altmin}
\KwData{Initial $\bbeta_1^{(0)}, \bbeta_2^{(0)}$, \# iterations $t_0$, samples $\{(y_i, \x_i), i =1,2,...,N\}$ }
\For{$t = 0,\cdots,t_0-1$}{
    $J_1, J_2 \gets \emptyset$\; \\
    \For{$i = 1, 2, \cdots, N$}{
        \eIf{$\left|y_i - \langle{\x_i}, { \bbeta}_1^{(t)}\rangle\right| < \left|y_i - \langle{\x_i},{ \bbeta_2^{(t)}}\rangle\right|$}{
			$J_1 \gets J_1 \cup \{i\}$\;
		    }
		    {
		    $J_2 \gets J_2 \cup \{i\}$\;
		    }
	}
\; $\bbeta_1^{(t+1)} \gets \argmin_{\bbeta \in \reals^k} \|{\y_{J_1} - \X_{J_1} \bbeta}\|_2$ \\
$\bbeta_2^{(t+1)} \gets \argmin_{\bbeta \in \reals^k} \|\y_{J_2} - \X_{J_2} \bbeta\|_2$ \;
}
\Return $\bbeta_1^{(t_0)}, \bbeta_2^{(t_0)}$
\end{algorithm}
\begin{proposition} \citep{yi_alternating_2014} \label{prop:initwp}
Consider the initialization method in Algorithm \ref{alg:init_withp}. Given any constant $\widehat{c} < 1/2$, with probablity at least $1 - \frac{1}{p^2}$, the approach produces an initialization $(\bbeta_1^{(0)},\bbeta_2^{(0)})$ satisfying
\begin{align*}
\rho({\mathbf{\theta}}^{(0)}, \mathbf{\theta}) ~ \leq ~ \widehat{c} \, \min\{n_1/n,n_2/n\} \, \|\bbeta_1 - \bbeta_2\|_2,
\end{align*}
if
\[ n \geq c_1 \left (\frac{1}{\widetilde{\delta}}\right )^2 \, p \, \log^2 p. \]
Here $c_1$ is a constant that depends on $\widehat{c}$. And
\[
\sqrt{\widetilde{\delta}} = \widehat{c}\sqrt{\min\{n_1/n,n_2/n\}}^3\|\bbeta_1 - \bbeta_2\|_2(\sqrt{1 - \kappa})\kappa,
\]
where $\kappa = \sqrt{1-4(1-\langle{\bbeta_1},{\bbeta_2}\rangle^2)\frac{n_1}{n} \frac{n_2}{n}}$.
\end{proposition}

\begin{algorithm}[ht]
\caption{AM with resampling} \label{alg:rspEM}
\KwData{Initial $\bbeta_1^{(0)}, \bbeta_2^{(0)}$, \# iterations $t_0$, samples $\{(y_i, \x_i), i =1,2,...,N\}$}
Partition the samples $\{(y_i, \x_i)\}$ into $t_0$ disjoint sets: $\mathcal{S}_1,...,\mathcal{S}_{t_0}$ \\
\For{$t = 1,\cdots,t_0$}{ 
    Use $\mathcal{S}_t$ to run Algorithm \ref{alg:altmin} initialized with $(\bbeta_1^{(t-1)}, \bbeta_2^{(t-1)})$ and returning $(\bbeta_1^{t}, \bbeta_2^{t})$
}
\Return $\bbeta_1^{(t_0)}, \bbeta_2^{(t_0)}$
\end{algorithm}

\begin{theorem} \cite{yi_alternating_2014} \label{thm:em}
Consider one iteration in Algorithm \ref{alg:rspEM}. For fixed $(\bbeta_1^{(t-1)},\bbeta_2^{(t-1)})$, there exist absolute constants $\widetilde{c},c_1,c_2$ such that if
\begin{equation*}
\label{cond_err}
  \rho({\mathbf{\theta}}^{(t-1)}, \mathbf{\theta}) ~ \leq  ~ \widetilde{c} \, \min \{n_1/n,n_2/n\} \,  \|\mathbf{\bbeta}_1^* - \mathbf{\bbeta}_2^*\|_2,
\end{equation*}
and if the number of samples in that iteration satisfies
\[
|\mathcal{S}_t| ~ \geq ~ \left( \frac{c_1}{ \min \{n_1/n,n_2/n\}} \right) \, p,
\]
then with probability greater than $1-\exp(-c_2 p)$ we have a geometric decrease in the error at the next stage, i.e.
\begin{equation*}
 \rho({\mathbf{\theta}}^{(t)}, \mathbf{\theta}) \leq \frac{1}{2} \rho({\mathbf{\theta}}^{(t-1)}, \mathbf{\theta})
\end{equation*}
\end{theorem}

\begin{proposition}(Exact Recovery) \cite{yi_alternating_2014} \label{prop:exactReco}
There exist absolute constants $c_1,c_2$ such that if
\[
\rho({\mathbf{\theta}}^{(t-1)}, \mathbf{\theta}) \leq \frac{c_1}{p^2}\|\bbeta_1 - \bbeta_2 \|_2
\]
and
\[
\frac{1}{\min\{n_1/n,n_2/n\}}p < |\mathcal{S}_t| < c_2 p,
\]
then with probability greater than $1 - \frac{1}{p}$,
\[
\rho({\mathbf{\theta}}^{(t)}, \mathbf{\theta}) = 0.
\]
\end{proposition}

\begin{remark} \label{remark:geom_decay_samples}
It is easy to see, and remarked in \cite{fan_curse_2018} (pp. 11-12, above and below Proposition 4) that, running Algorithm \ref{alg:rspEM} with guarantees given in Theorem \ref{thm:em}, one would require $O(p \log^2{p})$ samples to obtain $\rho({\mathbf{\theta}}^{(t-1)}, \mathbf{\theta}) \leq \frac{c_1}{p^2} \|\mathbf{\bbeta}_1^* - \mathbf{\bbeta}_2^*\|_2$ for the constant $c_1 > 0$ suitable for Proposition \ref{prop:exactReco}. \end{remark}
\paragraph{Algorithm for $\mslr$ in the balanced case}
Consider the case of Mixtures of Linear Regressions ($\mathtt{MLR}$) as in \cite{chen_convex_2014}, where $n, k = \Theta(p)$, and $\sigma^2$ is known. Further, consider the balanced regime where $\phi = 1/2$, but $\bbeta_{1, j} \neq - \bbeta_{2, j}$ for any $j \in \mathcal{S}_1 \cap \mathcal{S}_2$ ($\bavg \neq 0$). We claim that we can solve the aforementioned $\mslr$ problem in this regime by recovering the joint support of $\bbeta_1, \bbeta_2$ using $\corr$ , and then running the algorithm in \cite{chen_convex_2014} for general mixed linear regression. This latter Algorithm \ref{algo:balanced} is outlined below.
\begin{algorithm}
\caption{Estimate $\bbeta$'s \citep{chen_convex_2014}}
\label{algo:balanced}
\; \KwData{$(\X, \y)\in\mathbb{R}^{n\times p}\times\mathbb{R}^{n}$}
\; Let $(\hat{\bm{K}},\hat{\bm{g}}) := \argmin_{\bm{K},\bm{g}}\;  \sum_{i=1}^{n}\left(-\left\langle \bm{x}_{i}\bm{x}_{i}^{\top},\bm{K}\right\rangle +2y_{i}\left\langle \bm{x}_{i},\bm{g}\right\rangle -y_{i}^{2}+\sigma^{2}\right)^{2}+\lambda\left\Vert \bm{K} \right\Vert _{*}$ \; \\
Compute the matrix $\hat{\bm{J}}=\hat{\bm{g}}\hat{\bm{g}}^{\top}-\hat{\bm{K}}$, and its first eigenvalue-eigenvector pair $\hat{\lambda}$ and $\hat{\bm{v}}$ \; \\
Compute $\hat{\bm{\bbeta}}_{1},\hat{\bm{\bbeta}}_{2}=\hat{\bm{g}}\pm\sqrt{\hat{\lambda}}\hat{\bm{v}}$ \; \\
\Return $(\hat{\bm{\bbeta}}_{1},\hat{\bm{\bbeta}}_{2})$
\end{algorithm}
As motivated in \citep{chen_convex_2014}, the algorithm performs a convex penalized least squares optimization to determine matrix and vector $(\hat{\bm{K}},\hat{\bm{g}})$, $\hat{\bm{g}}$ being a naive estimate of $\bbeta_1, \bbeta_2$ and the leading eigenvector-eigenvalue of $\hat{\bm{J}}$ a necessary correction. We define the value $\alpha:=\frac{\left\Vert \bm{\bbeta}_{1}-\bm{\bbeta}_{2}\right\Vert _{2}^{2}}{\left\Vert \bm{\bbeta}_{1}\right\Vert _{2}^{2}+\left\Vert \bm{\bbeta}_{2}\right\Vert_2^{2}}$.

\begin{definition} \citep{chen_convex_2014} \label{def:conds}
Let $n_1 = \left\{i \in [n]: z_i = 1  \right\}$ denote the number of samples obtained from $\bbeta_1$, and $n_2$ the analogous for $\bbeta_2$. We define the following \textit{regularity} conditions, required for our further proofs:

\begin{enumerate}
 \item $\X$ is an i.i.d. standard Gaussian matrix,
	\item $\alpha \geq c_3$,
	\item $\min\{n_1, n_2 \} \geq c_4 p$,
	\item $\lambda = \Theta(\sigma (\|\bbeta_1\|_2 + \|\bbeta_2\|_2 + \sigma) \sqrt{np} \log^3{n})$,
	\item $n \geq c_3 p \log^8{n}$,
	\item $|n_1 - n_2| = O\left(\sqrt{n \log{n}} \right)$,
\end{enumerate}
for some constants $0<c_3<2$ and $c_4$.
\end{definition}

\begin{theorem} \cite{chen_convex_2014} \label{thm:chen_convex}
Suppose the conditions in Definition \ref{def:conds} hold. There exist constants $c_1, c_2, c_4 > 0$ such that with probability at least $1 - c_1 n^{-c_2}$, the output $\hat{\theta} = (\hat{\bbeta}_1, \hat{\bbeta}_2)$ of $\mathtt{Algorithm \; 3}$ satisfies

\[\rho(\hat{\theta}, \theta) \leq c_4 \sigma \sqrt{\frac{p}{n}}\log^4{n} + c_4 \min\left\{\frac{\sigma^2}{\|\bbeta_1\|_2 + \|\bbeta_2\|_2}, \sigma \left(\frac{p}{n} \right)^{1/4} \right\} \log^4{n}. \]
\end{theorem}
The result in Theorem \ref{thm:chen_convex} implies that, in the high-snr regime $\|\bbeta\|^2_2/\sigma^2 = \Omega(1)$, we have that $\rho(\hat{\theta}, \theta) = \Theta(\sigma \sqrt{\frac{p}{n}})$. This holds since $\|\bbeta_1\|_2 + \|\bbeta_2\|_2 = 2\sqrt{k}$.

\begin{theorem} [Success of $\mathtt{CORR} + \mathtt{Algorithm \; \ref{algo:balanced}}$ on $\mslr$] \label{thm:corr+algo3}
Consider the general setting of $\mslr$ with parameters $p, n, k, \phi = 1/2$, $(\bbeta_1, \bbeta_2) \distas{} \mathcal{P}_{\|\bbeta\|_2}(\mathcal{D})$. Suppose the conditions of Definition \ref{def:conds} hold. There exist constants $c_1, c_2, c_4 > 0$ such that, provided 
$$n \geq \frac{32}{\min\{\phi^2 \bbeta^2_{\texttt{min}}, (1-\phi)^2 \bbeta^2_{\texttt{min}}, \langle \bbeta \rangle^2_{\texttt{min}} \}} (1 + \epsilon) (\|\bbeta\|^2_2 + \sigma^2) \log{2p}$$ for  $\epsilon \in (0,1)$ used in $\corr$ , with probability at least $1 - c_1(\frac{k}{p} + ke^{-c_2 n} + \frac{k}{n} + \frac{1}{p^{c_2}})$, the output $\hat{\theta} = (\hat{\bbeta}_1, \hat{\bbeta}_2)$ of $\mathtt{CORR} + \mathtt{Algorithm \; \ref{algo:balanced}}$ satisfies

\[\rho(\hat{\theta}, \theta) \leq c_4 \sigma \sqrt{\frac{2k}{n}}\log^4{n} + c_4 \min\left\{\frac{\sigma^2}{\|\bbeta_1\|_2 + \|\bbeta_2\|_2}, \sigma \left(\frac{2k}{n} \right)^{1/4} \right\} \log^4{n}. \]
\end{theorem}

\begin{proof} [Proof of Theorem \ref{thm:corr+algo3}]
By Theorem \ref{thm:CORR-mslr-r-general}, we can recover the joint support set (at most of size $2k$) of signals $(\bbeta_1, \bbeta_2)$ with probability at least $1 - c_1(\frac{k}{p} + ke^{-c_2 n} + \frac{k}{n} + \frac{1}{p^{c_2}})$ by running $\mathtt{CORR}$.

After running $\mathtt{CORR}$ and identifying the $<2k$ joint support set indices, we restrict the regression problem to these indices by removing all other columns from the design matrix $\X$ (as these do not influence the output $y$ since they do not correspond to support indices of $\bbeta_1$ or $\bbeta_2$). We are then tasked with solving a two-component mixtures of regressions problem with $n$ samples and signals of dimension between $k$ and $2k$. We run $\mathtt{Algorithm \; 3}$ on this simplified regression problem, to obtain $(\hat{\bbeta}_1, \hat{\bbeta}_2)$.

What remains is to show that the assumptions of Theorem \ref{thm:chen_convex} hold with $p$ replaced by $2k$. Indeed, condition $1.$ holds since without loss of generality we can assume $\bbeta_1 \neq \bbeta_2$, otherwise the problem setting would be that of $\slr$. Condition $3., 4., 5.$ hold by the definition of the $\mslr$ problem, and by freedom with respect to $\lambda$ and $c_3$.

Condition $2.$ holds with probability at least $1 - \exp\left(- \Theta(n)\right)$, indeed by Hoeffding's inequality \citep{boucheron_concentration_2013} we have that

\begin{align*}
\mathbb{P}\left[ n_1 < 2 c_4 k\right] &= \mathbb{P}\left[ \sum_{i=1}^n z_i < 2c_4 k \right] \\
&= \mathbb{P}\left[ \sum_{i=1}^n (z_i - \phi) < 2c_4 k - \phi n\right] \\
&= \mathbb{P}\left[ \sum_{i=1}^n (-z_i + \phi) \geq \phi n -2c_4 k \right] \\
&\leq \exp\left(- \Theta(n)\right),
\end{align*}

where $\phi n -2c_4 k$ is positive for some $c_4 > 0$ since by assumption $n \geq \frac{32 (1 + \epsilon)}{\min\{(2 \phi - 1)^2, \phi^2, (1 - \phi)^2\}} (1 + \epsilon) (k + \sigma^2) \log{p}$ and $(k + \sigma^2) \log{p} = \omega(k)$, analogously for $n_2$.

In the case $\phi = 1/2$ (so that $2\phi - 1 = 0$), we have that condition $6.$ holds with probability $\exp(-\Theta(\log{n}))$ by Hoeffding's inequality,

\begin{align*}
&\mathbb{P}\left[\left|\sum_{i=1}^n z_i - \sum_{i=1}^n (1 - z_i)\right| \geq \Theta(\sqrt{n \log{n}})\right] \\
&\leq \mathbb{P}\left[\sum_{i=1}^n (2z_i - 1) \geq \Theta(\sqrt{n\log{n}}) \right] + \mathbb{P}\left[ \sum_{i=1}^n (-2z_i + 1) \geq \Theta(\sqrt{n \log{n}}) \right] \\
&\leq 2\exp\left(-\frac{2 \Theta(n \log{n})}{2 n}\right) = \exp\left(-\Theta(\log{n})\right).
\end{align*}

The above events occur with probability at least $1 - \exp(-\Theta(\log{n}))$, and the event that $\mathtt{CORR}$ succeeds occurs with probability $1 - c_1(\frac{k}{p} + ke^{-c_2 n} + \frac{k}{n} + \frac{1}{p^{c_2}})$. These two events occur together with probability at least $1 - c_1(\frac{k}{p} + ke^{-c_2 n} + \frac{k}{n} + \frac{1}{p^{c_2}})$ for adjusted constants $c_1, c_2 > 0$. Applying Theorem \ref{thm:chen_convex}, we obtain the result.
\end{proof}

\subsection{$\mathtt{CORR}$ for signed support recovery in $\slr$} \label{subsec:corr-slr}

\begin{proof} [Proof of Theorem \ref{thm:corr-slrd}]
The proof proceeds similarly as that of Theorem \ref{thm:CORR-mslrd-general}. Let $\corr(\X, \y)$ denote the output of running $\corr$  on inputs $\X, \y$. Consider the test function

\[
g\left({\begin{bmatrix} \X \\ \y \end{bmatrix}}\right) := \begin{cases}
\mathtt{p} & \corr(\X, \y) \neq \emptyset \\
\mathtt{q} & \corr(\X, \y) = \emptyset
\end{cases}.
\]

Let ${\begin{bmatrix} \X \\ \y \end{bmatrix}} \distas{} {\P(\X) \otimes \P(\y)}$. Recall that $\corr$  outputs the following set

\[
\corr(\X, \y) = \left\{j \in [p]:  \left|\frac{\langle \X_{j}, \y \rangle}{\|\y \|_2}\right| \geq \sqrt{2 (1 + \epsilon/2) \log{2p}} \right\}.
\]

As in the proof of Theorem \ref{thm:corr-slr}, we note that for ${\begin{bmatrix} \X \\ \y \end{bmatrix}} \distas{} {\P(\X) \otimes \P(\y)}$ we have $\frac{\langle \X_{q}, \y \rangle}{\|\y \|_2} \distas{\text{i.i.d.}} \mathcal{N}(0, 1)$. Applying Lemma \ref{lem:high-prob-events} to (\ref{eq:last3}) we obtain,
\begin{align}
\mathbb{P}\left[\max_{q \in [p]} \left|\frac{\langle \X_{q}, \y \rangle}{\|\y \|_2}\right| \geq \sqrt{2 (1 + \epsilon/2) \log{2p}} \right] &\leq \mathbb{P}\left[\max_{q \in [p]} \left|\frac{\langle \X_{q}, \y \rangle}{\|\y \|_2}\right| \geq \sqrt{2 \log{2p}} + \frac{\epsilon}{2 \sqrt{8}} \sqrt{\log{2p}}\right], \label{eq:last3} \\
&\leq {p^{- \frac{1}{16} (\frac{\epsilon}{2})^2}} = o(1), \nonumber
\end{align}
and hence we have that, under ${\P(\X) \otimes \P(\y)}$, $g\left({\begin{bmatrix} \X \\ \y \end{bmatrix}}\right) = \mathtt{q}$ with probability $1 - o(1)$.

Conversely, let ${\begin{bmatrix} \X \\ \y \end{bmatrix}} \distas{} {\P(\X, \y)}$. We then apply Theorem \ref{thm:corr-slr} to deduce that $\corr(\X, \y) = \supp(\bbeta) \neq \emptyset$ with probability at least $1 - \left(\frac{k}{p} + 2ke^{-c_2 n} + \frac{1}{p^{c_2}} \right)$ for some constant $c_2 > 0$. Hence, under $\P(\X, \y)$, $g\left({\begin{bmatrix} \X \\ \y \end{bmatrix}}\right) = \mathtt{p}$ with probability $1 - o(1)$.
\end{proof}

In what follows, we consider the modified  $\corr$  algorithm in \eqref{eq:signed_corr} for estimating the \textit{signed} support of $\bbeta$.
%

%
\begin{proof} [Proof of Theorem \ref{thm:corr-slr}]
Let $\mathcal{S}$ denote the support set of $\bbeta$. Define the error event
\[\mathcal{E} = \cup_{j \in \mathcal{S}} \left\{\left|\frac{\langle \X_j, \y \rangle}{\|\y\|_2}\right| < \sqrt{2(1 + \epsilon/2) \log{2p}} \right\} \cup \left\{\max_{q \in \mathcal{S}^\complement} \left|\frac{\langle \X_q, \y \rangle}{\|\y \|_2}\right| \geq \sqrt{2(1 + \epsilon/2) \log{2p}} \right\}. \]

The theorem claim follows by demonstrating that $\mathbb{P}\left[\mathcal{E} \right] = o(1)$. With this in mind, we perform a union bound

\begin{align*}
\mathbb{P}\left[\mathcal{E} \right] &\leq k \mathbb{P}\left[\left|\frac{\langle \X_{j^*}, \y \rangle}{\|\y\|_2}\right| < \sqrt{2(1 + \epsilon/2) \log{2p}} \right] + \mathbb{P}\left[ \max_{q \in \mathcal{S}^\complement} \left|\frac{\langle \X_q, \y \rangle}{\|\y \|_2}\right| \geq \sqrt{2(1 + \epsilon/2) \log{2p}}\right] \\
&:= k \nu_1 + \nu_2
\end{align*}
where $j^* \in \mathcal{S}$. We first focus on $\nu_2$, where we notice $\frac{\langle \X_q, \y \rangle}{\|\y \|_2} \distas{\text{i.i.d.}} \mathcal{N}(0, 1)$ for $q \in \mathcal{S}^\complement$. Applying Lemma \ref{lem:high-prob-events}, we deduce that
\begin{align*}
\nu_2 &= \mathbb{P}\left[\max_{q \in \mathcal{S}^\complement} \left|\frac{\langle \X_{q}, \y \rangle}{\|\y \|_2}\right| \geq \sqrt{2 (1 + \epsilon/2) \log{2p}} \right] \\
&\leq \mathbb{P}\left[\max_{q \in \mathcal{S}^\complement} \left|\frac{\langle \X_{q}, \y \rangle}{\|\y \|_2}\right| \geq \sqrt{2 \log{2p}} + \frac{\epsilon}{2 \sqrt{8}}\sqrt{\log{2p}}\right] \\
&\leq (2p)^{- \frac{1}{16} \left(\frac{\epsilon}{2}\right)^2},
\end{align*}
since $\sqrt{2(1+\epsilon/2)} \geq \sqrt{2} + \frac{\epsilon}{2 \sqrt{8}}$.  Setting $n \geq \frac{8 (1 + \epsilon)}{\bbeta^2_{\min}} (\|\bbeta\|^2_2 + \sigma^2) \log{2p}$ for some $\epsilon \in (0,1)$, the bound for $\nu_1$ follows from applying Lemma \ref{lemma:A-bound-slr} 
below:
\[ \nu_1 \leq \frac{1}{2p} + 2e^{-\frac{\delta^2 n}{8}}. \]
Putting it all together, we obtain that
$ \mathbb{P}\left[\mathcal{E} \right] \leq \frac{k}{2p} + 2ke^{-c_2 n} + \frac{1}{p^{c_2}}$,
for some constant $c_2 > 0$. 
\end{proof}
\begin{lemma} [Concentration bound for $\slr$] \label{lemma:A-bound-slr}
Consider the setting of $\slr$ for $j^* \in \mathcal{S}$. Then for $n \geq \frac{8 (1 + \epsilon)}{\bbeta^2_{\min}} (\|\bbeta\|^2_2 + \sigma^2) \log{2p}$ we have

\[ \mathbb{P}\left[\left| \frac{\langle \X_{j^*}, \y \rangle}{\|\y\|_2} \right| \leq \sqrt{2 (1 + \epsilon/2) \log{2p}} \right] \leq \frac{1}{p} + 2 e^{-\frac{\delta^2 n}{8}}\]
\end{lemma}

\begin{proof}
Let $\nu := \mathbb{P}\left[\left| \frac{\langle \X_{j^*}, \y \rangle}{\|\y\|_2} \right| \leq \sqrt{2 (1 + \epsilon/2) \log{2p}} \right]$ and $\delta > 0$ to be chosen later. We begin by conditioning on the event $\Omega_1 := \Omega_1(\delta)$ from Lemma \ref{lem:high-prob-events}, and apply Lemma \ref{lemma:gsum-general} (for the case $z_i = 1$) denoting $g_i \distas{\text{i.i.d.}} \mathcal{N}(0, 1)$ for $i \in [n]$ as independent (also from $y$, $z$) unit normal Gaussians, obtaining

\begin{align}
\nu &= \int \mathbb{P}\left[\left|\frac{\X_{j^*}^T \xi}{\|\xi \|_2}\right| < \sqrt{2(1 + \epsilon/2) \log{2p}} \; \middle| \; \y \right] d\mathbb{P}\left[y \right] \nonumber \\
&= \int \mathbb{P}\left[\left|\sum_{i=1}^n g_i \frac{y_i}{\|\y\|_2} \sqrt{1 - \frac{\bbeta^2_{j^*}}{\|\bbeta\|^2_2 + \sigma^2}} + \frac{\bbeta_{j^*} \cdot y^2_i/\|\y\|_2}{\|\bbeta\|^2_2 + \sigma^2}\right| < \sqrt{2(1 + \epsilon/2) \log{2p}} \; \middle| \; \y  \right] d\mathbb{P}\left[y \right] \nonumber\\
&= \int \mathbb{P}\left[\left|g_1 \sqrt{1 - \frac{\bbeta^2_{j^*}}{\|\bbeta\|^2_2 + \sigma^2}} + \frac{\bbeta_{j^*} \cdot \|\y\|_2}{\|\bbeta\|^2_2 + \sigma^2}\right| < \sqrt{2(1 + \epsilon/2) \log{2p}} \; \middle| \; \y  \right] d\mathbb{P}\left[y \right] \nonumber\\
&\leq \mathbb{P}\left[\left|g_1 \sqrt{1 - \frac{\bbeta^2_{j^*}}{\|\bbeta\|^2_2 + \sigma^2}} + \frac{\bbeta_{j^*} \cdot \|\y\|_2}{\|\bbeta\|^2_2 + \sigma^2}\right| < \sqrt{2(1 + \epsilon/2) \log{2p}} \; \middle| \; \Omega_1  \right] + \mathbb{P}\left[\Omega^{\complement}_1 \right]\nonumber\\
&= \mathbb{P}\left[\left\{g_1 \sqrt{1 - \frac{\bbeta^2_{j^*}}{\|\bbeta\|^2_2 + \sigma^2}} + \frac{\bbeta_{j^*} \cdot \|\y\|_2}{\|\bbeta\|^2_2 + \sigma^2} < \sqrt{2(1 + \epsilon/2) \log{2p}}\right\} \right. \nonumber \\
&\hspace{5em} \left. \cap \left\{-\sqrt{2(1 + \epsilon/2) \log{2p}} \leq g_1 \sqrt{1 - \frac{\bbeta^2_{j^*}}{\|\bbeta\|^2_2 + \sigma^2}} + \frac{\bbeta_{j^*} \cdot \|\y\|_2}{\|\bbeta\|^2_2 + \sigma^2} \right\} \; \middle| \; \Omega_1  \right] + \mathbb{P}\left[\Omega^{\complement}_1 \right]\nonumber\\
&\leq \mathbb{P}\left[g_1 \sqrt{1 - \frac{\bbeta^2_{j^*}}{\|\bbeta\|^2_2 + \sigma^2}} < \sqrt{2(1 + \epsilon/2) \log{2p}} - \left|\frac{\bbeta_{j^*} \cdot \|\y\|_2}{\|\bbeta\|^2_2 + \sigma^2}\right| \; \middle| \; \Omega_1  \right] + \mathbb{P}\left[\Omega^{\complement}_1 \right]\nonumber\\
&\leq \mathbb{P}\left[g_1 \sqrt{1 - \frac{\bbeta^2_{j^*}}{\|\bbeta\|^2_2 + \sigma^2}} < \sqrt{2(1 + \epsilon/2) \log{2p}} - \left|\frac{\bbeta_{j^*} \cdot \sqrt{n (\|\bbeta\|^2_2 + \sigma^2) (1-\delta)}}{\|\bbeta\|^2_2 + \sigma^2}\right| \; \middle| \; \Omega_1  \right] + \mathbb{P}\left[\Omega^{\complement}_1 \right]\nonumber\\
&= \mathbb{P}\left[g_1 \sqrt{1 - \frac{\bbeta^2_{j^*}}{\|\bbeta\|^2_2 + \sigma^2}} < \sqrt{2(1 + \epsilon/2) \log{2p}} - \left|\bbeta_{j^*} \sqrt{\frac{n (1 - \delta)}{\|\bbeta\|^2_2 + \sigma^2}}\right| \; \middle| \; \Omega_1  \right] + \mathbb{P}\left[\Omega^{\complement}_1 \right]\nonumber. 
\end{align}

Now setting $n \geq \frac{8 (1 + \epsilon)}{\bbeta^2_{\min}}(\|\bbeta\|^2_2 + \sigma^2)\log{2p}$ and applying standard sub-Gaussian bounds (see \cite{wainwright_high-dimensional_2019}) we obtain
\begin{align*}
\nu &\leq \mathbb{P}\left[g \sqrt{1 - \frac{\bbeta^2_{j^*}}{\|\bbeta\|^2_2 + \sigma^2}}  < (\sqrt{2(1 + \epsilon/2)} - \sqrt{8 (1 + \epsilon)(1 - \delta)}) \sqrt{\log{2p}}\right] + \mathbb{P}\left[\Omega_1^\complement \right] \\
&\leq \exp\left(-\frac{(\sqrt{2(1 + \epsilon/2)} - \sqrt{8 (1 + \epsilon)(1 - \delta)})^2 \log{2p}}{2 \left(1 - \frac{\bbeta^2_{j^*}}{\|\bbeta\|^2_2 + \sigma^2} \right)} \right) + 2 e^{-\frac{\delta^2 n}{8}},
\end{align*}
where
\begin{align}
\sqrt{2(1 + \epsilon/2)} - \sqrt{8 (1 + \epsilon)(1 - \delta)} < -\sqrt{2} \label{eq:slr-index-bound}
\end{align}
for $\delta > 0$ small enough. Hence, choosing $\delta$ to satisfy (\ref{eq:slr-index-bound}) above, we obtain that for $k$ large enough,
\begin{align*}
\nu &\leq \exp\left(-\frac{2 \log{2p}}{2 \left(1 - \frac{\bbeta^2_{j^*}}{\|\bbeta\|^2_2 + \sigma^2} \right)} \right) + 2 e^{-\frac{\delta^2 n}{8}} \\
&\leq \exp\left(-\frac{2 \log{2p}}{2 } \right) + 2 e^{-\frac{\delta^2 n}{8}} \\
&\leq \frac{1}{2p} + 2 e^{-\frac{\delta^2 n}{8}}.
\end{align*}
\end{proof}

\end{document}